%% file: main.tex
\documentclass{article}

\usepackage[preprint]{neurips_2026}
\usepackage{graphicx}
\usepackage{booktabs}
\usepackage{multirow}
\usepackage{amsmath} 
\usepackage{wrapfig}
\usepackage{enumitem}

\usepackage{subcaption}

\usepackage[utf8]{inputenc} 
\usepackage[T1]{fontenc}    
\usepackage{hyperref}       
\usepackage{url}            
\usepackage{booktabs}       
\usepackage{amsfonts}       
\usepackage{nicefrac}       
\usepackage{microtype}      
\usepackage{xcolor}         
\usepackage{booktabs}
\usepackage{multirow}
\usepackage{graphicx}
\usepackage{bm}

\newtheorem{theorem}{Theorem}

\newtheorem{definition}{Definition}
\newtheorem{proof}{Proof}

\title{\method{}: Polar Riemannian Flow Matching for Structure-Preserving Graph Domain Adaptation}

\author{
Yingxu Wang \\
The Chinese University of Hong Kong \\
\texttt{yingxv.wang@gmail.com}
\And
Xinwang Liu \\
National University of Defense Technology \\
\texttt{xinwangliu@nudt.edu.cn}
\And
Mengzhu Wang \\
Hebei University of Technology \\
\texttt{dreamkily@gmail.com}
\And
Siyang Gao, Nan Yin\\
City University of Hong Kong \\
\texttt{siyangao@cityu.edu.hk, yinnan8911@gmail.com}
}

\def\method{DisRFM} 

\begin{document}

\maketitle

\begin{abstract}
Graph Domain Adaptation (GDA) aims to transfer graph classifiers across domains with both semantic and topological shifts. Existing Euclidean adversarial methods face two challenges: {\textit{Structural Degeneration}}, where domain confusion entangles and suppresses label-relevant topology, and {\textit{Optimization Instability}}, where minimax training induces oscillatory gradients under large structural shifts. We propose \method{}, a geometry-aware GDA framework that addresses these challenges with Riemannian representation learning and flow-based transport. \method{} embeds graph representations on a constant-curvature manifold and expresses them in geodesic polar coordinates. Polar endpoint regularization calibrates topology-sensitive radial scales via univariate Wasserstein alignment and preserves scale-normalized class semantics through confidence-filtered angular alignment, with radial magnitude modulating pseudo-label reliability. \method{} introduces topology-conditioned polar flow matching, which couples class-compatible source and target samples by a normalized polar transport cost and learns a metric-corrected vector field along geodesic interpolants. Theoretical analysis characterizes the structural risk of unconditional domain confusion and relates polar discrepancies and flow error to target risk. Extensive experiments under diverse domain shifts demonstrate that \method{} consistently outperforms state-of-the-art methods.
\end{abstract}

\input{code/1_intro}
\input{code/2_related_work}
\input{code/3_method}

\input{code/4_experiments}
\input{code/5_conclusion}

\bibliographystyle{plain}
\bibliography{reference}

\newpage
\input{code/6_appendix}

\end{document}

%% file: code/1_intro.tex
\section{Introduction}

Graph Domain Adaptation (GDA) aims to transfer a graph classifier trained on a labeled source domain to an unlabeled target domain under distribution shift~\citep{wu2020unsupervise,wang2024degree,wang2026nested,wang2026usbd}. Unlike standard feature-domain adaptation, graph shifts are often simultaneously semantic and topological: the decision boundary should remain transferable, while graph size, density, degree profiles, structural roles, motifs, or community patterns may vary across domains~\citep{henderson2012rolx}. A reliable GDA model should therefore align task-relevant semantics without erasing structure-sensitive information that is
useful for target prediction.

Existing GDA methods align Euclidean latent spaces by making source
and target embeddings domain-indistinguishable through an adversarial
objective~\citep{sun2020adagcn,shen2020adversarial,ganin2016domain,feng2020graph}.
This paradigm is effective under moderate covariate shifts, but becomes fragile when domain variables are entangled with graph structure. First, Euclidean alignment can induce \textbf{\textit{Structural Degeneration}}. Since Euclidean coordinates conflate radial scale and direction, structural roles and class-discriminative cues may share coordinates; domain confusion can remove topology-sensitive variation alongside domain variation~\citep{bronstein2017geometric,nickel2017poincare,ganea2018hyperbolic,sala2018representation,chami2019hyperbolic}. Figure~\ref{challenge}(a) shows degraded structure-probe scores. Second, adversarial GDA optimizes a minimax game between the encoder and domain discriminator~\citep{goodfellow2014generative}, causing \textbf{\textit{Optimization Instability}}. Under
large structural shifts, it may yield oscillatory or poorly conditioned
gradients~\citep{arjovsky2017wasserstein,mescheder2018training,salimans2016improved},
as shown in Figure~\ref{challenge}(b). These limitations suggest that graph
domain adaptation should avoid unconditional domain confusion. Instead, it
should preserve label-relevant structure while transporting samples across
semantically and topologically compatible regions.

\begin{figure*}[t!]
  \centering
  \includegraphics[scale=0.4]{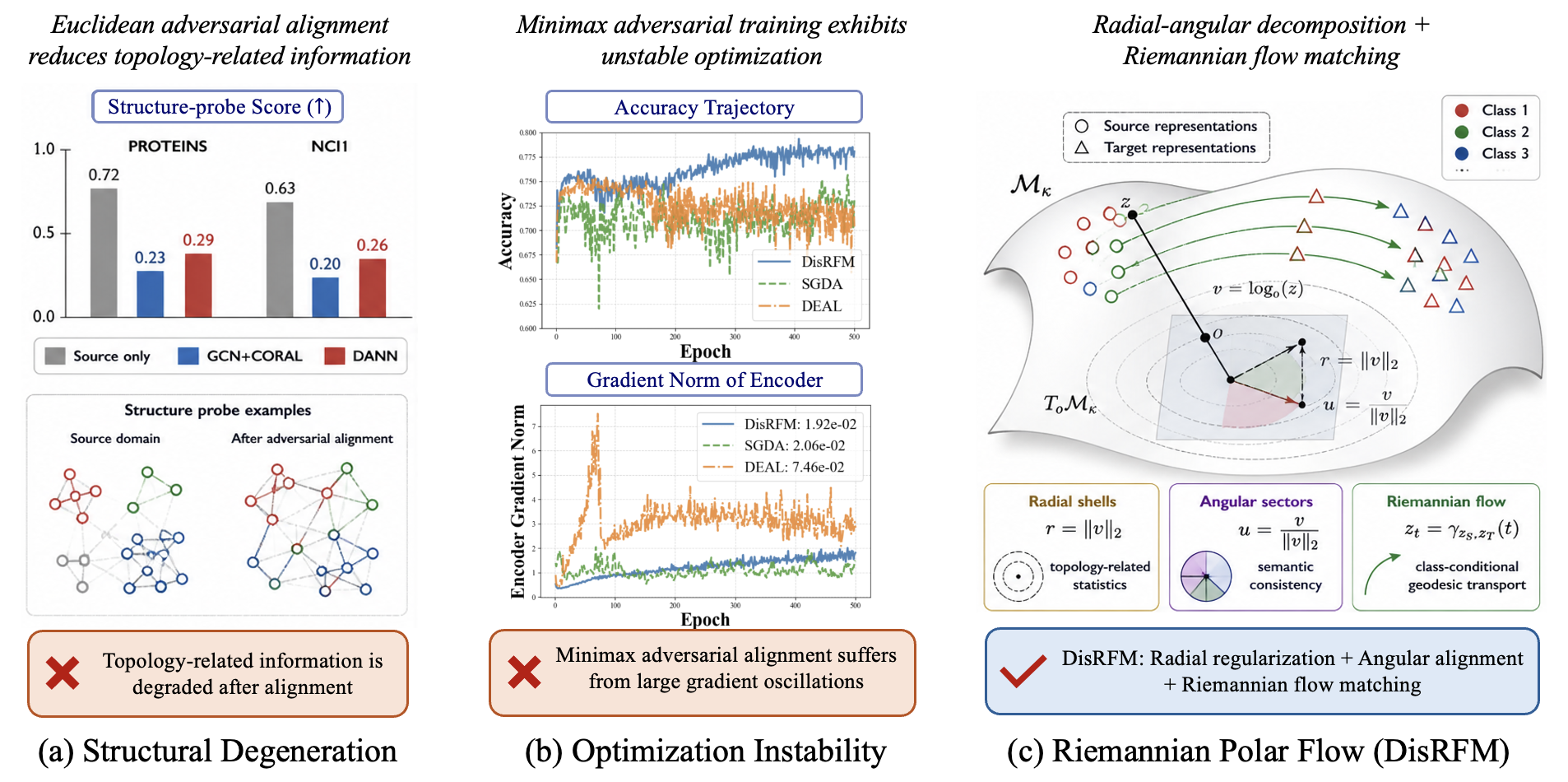} 
  \vspace{-5pt}
\caption{Motivation of DisRFM.
(a) Euclidean adversarial GDA erodes topology information, evidenced by lower
post-alignment structure-probe scores.
(b) Minimax alignment yields unstable training with gradient-norm oscillations.
(c) DisRFM imposes a Riemannian polar bias: radial shells act as compact
topology-related proxies, angular sectors preserve semantic consistency, and
flow matching transports source embeddings toward targets along class-conditioned
geodesics.}
  \label{challenge}
  \vspace{-1.0cm}
\end{figure*}

We address this problem from a representation-geometry and transport viewpoint.
Graph domain shifts often entangle label-relevant structural variation with
domain-specific nuisance factors. As a result, enforcing Euclidean domain
confusion may inadvertently suppress structural cues useful for target
prediction. Curved manifolds offer a suitable inductive bias for
hierarchy-rich and non-Euclidean graph structures beyond isotropic Euclidean
spaces~\citep{nickel2017poincare,tifrea2018poincar,chami2019hyperbolic}. In a
constant-curvature manifold, geodesic polar coordinates provide a natural view
of embeddings through a scale-sensitive radial component and a scale-normalized
angular component. This does not imply an identifiable separation between
topology and semantics; rather, it suggests a useful regularization bias for
graph transfer. The resulting challenge is to align domains without collapsing
all domain-correlated structure or treating radial, angular, and transport
effects as independent heuristics. This motivates topology-conditioned polar
transport: source and target graphs should be matched according to class
semantics and graph structure, and their alignment should follow the geometry of
the representation space.

To this end, we propose \method{}, a class-aware transport framework over constant-curvature manifolds. Polar geometry provides an operational coordinate bias, not an identifiable factorization: radius serves as a topology-related scale proxy, while angle offers a scale-normalized signal for semantic alignment. Radial structural calibration matches topology-sensitive metric scales, keeping comparable graphs at compatible radii without forcing structure into radius. Confidence-filtered angular alignment matches class prototypes on scale-normalized directions, preserving margins while limiting confirmation noise. Class-conditional Riemannian flow matching replaces discriminator-based alignment with velocity regression along pseudo-label coupled geodesics, transporting representations toward target support while preserving radius--angle coupling. Thus, \method{} replaces domain-level indistinguishability with geometry-constrained, class-conditional manifold alignment.

Our main contributions are summarized as follows: (1) We identify a structural failure mode of domain-confusion GDA and formalize, through a simple information and total-variation argument, why enforcing domain invariance can remove label-relevant graph structure when structure and domain are correlated. (2) We propose \method{}, which integrates coordinate-wise radial/angular regularization in Riemannian polar coordinates with geometry-aware 
flow matching for stable graph domain adaptation. 
(3) Extensive experiments on benchmark datasets demonstrate that \method{} outperforms state-of-the-art baselines and exhibits strong optimization stability.


%% file: code/2_related_work.tex
\section{Related Work}

\textbf{Domain Adaptation (DA).} DA transfers knowledge from a labeled source domain to an unlabeled target
domain by reducing distributional shifts while preserving task-relevant
semantics~\citep{singhal2023domain,xu2024video}. With the growing use of graph
data, this paradigm has extended to Graph Domain Adaptation (GDA), where both
node attributes and graph structures vary across domains~\citep{wu2024graph,
zhang2024survey}. Most methods learn graph representations with GNNs and apply
alignment techniques, such as adversarial learning, MMD minimization, or optimal
transport, to obtain domain-invariant features~\citep{yin2022deal,yin2023coco,
liu2024pairwise}. However, they typically align graphs in Euclidean space and
overlook non-Euclidean geometry, which often causes distorted representations
and negative transfer under significant structural shifts~\citep{yin2025dream,
fang2025benefits,chen2025smoothness}. To address this, we propose \method{} for geometry-aware alignment in non-Euclidean space.

\textbf{Manifold Learning on Graphs.}
Manifold learning assumes that high-dimensional data reside on low-dimensional
manifolds, capturing geometric properties that Euclidean methods often
distort~\citep{lin2008riemannian, zhang2011adaptive, meilua2024manifold}.
Recent advances extend this paradigm to non-Euclidean spaces, integrating graph
topology to better encode hierarchical and relational
structures~\citep{sun2025trace, sun2024motif, fei2025survey,wang2025dusego}. However, existing
manifold-based methods mainly focus on single-domain tasks, overlooking domain
adaptation~\citep{de2023latent, sun2023self, wang2024mixed,wang2026dsbd}. Crucially, naive
applications of general adaptation techniques to manifold spaces are suboptimal,
as they fail to model cross-domain distributional shifts while respecting the
underlying geometry~\citep{chen2025graffe, gharib2025geometric,collas2025riemannian}.
To bridge this gap, we propose \method{}, which formulates graph domain
adaptation as a continuous, geometry-aware transformation on a non-Euclidean
manifold via Riemannian flow matching.

%% file: code/3_method.tex
\section{Methodology}

\textbf{Overview.}
We propose \method{}, a two-component GDA framework: (1) \textbf{Polar Endpoint Regularization} mitigates structural degeneration using radius as a topology-related scale proxy and angle as a scale-normalized class-alignment signal. (2) \textbf{Topology-Conditioned Polar Flow Matching} mitigates indiscriminate alignment by transporting source and target graphs through class- and structure-compatible regions, replacing adversarial domain confusion with geometry-aware transport supervision. Together, they preserve label-relevant structure, reduce over-alignment, and improve transfer under structural shifts without assuming globally identifiable topology--semantics separation.

\subsection{Polar Endpoint Regularization}
\label{subsec:radial_alignment}

Existing GDA objectives usually align Euclidean embeddings with isotropic losses,
which do not distinguish scale-sensitive and direction-sensitive variations and
may therefore over-align graph representations under structural
shift~\citep{nickel2017poincare,bronstein2017geometric}. We instead use polar coordinates induced by a fixed origin on a constant-curvature manifold. On a geodesic normal ball, the metric has the block form
\begin{equation}
    ds^2 = dr^2 + S_c(r)^2 d\Omega^2,
    \qquad
    S_c(r)=
    \begin{cases}
    \sin(\sqrt{c}r)/\sqrt{c}, & c>0,\\
    r, & c=0,\\
    \sinh(\sqrt{-c}r)/\sqrt{-c}, & c<0,
    \end{cases}
    \label{eq:polar_metric}
\end{equation}
where $d\Omega^2$ is the standard spherical metric. This geometry motivates separate endpoint roles for radius and angle while preserving their metric coupling through the warping factor $S_c(r)^2$. Throughout the theoretical statements, $r$ denotes geodesic radius. In implementation, $r$ is computed from $\mathbf v=\operatorname{Log}_o^c(\mathbf z)$ in the origin tangent coordinates; the constant conformal factor at the origin only rescales all radii and is absorbed by the radial loss weight.

\textbf{Radial Structural Calibration.}
First, we obtain the graph-level representation $\mathbf{z} \!\in\! \mathcal{M}_c$ via~\citep{liu2019hyperbolic}
\begin{equation}
\label{HGCN}
    \mathbf{m}_j^{(l)}\! =\!\mathbf{W}^{(l)} \otimes_c \mathbf{h}_j^{(l)}\! =\! \text{Exp}_{\mathbf{0}}^c \left( \mathbf{W}^{(l)}\!\cdot \!\text{Log}_{\mathbf{0}}^c(\mathbf{h}_j^{(l)}) \right),\;
    \mathbf{h}_i^{(l+1)} \!=\! \text{Exp}_{\mathbf{0}}^c \Bigg( \sigma \Big( \!\!\sum_{j \in \mathcal{N}(i) \cup \{i\}} \!\!\alpha_{ij} \cdot \text{Log}_{\mathbf{0}}^c ( \mathbf{m}_j^{(l)} ) \Big) \Bigg),\nonumber
\end{equation}
where the manifold curvature $c$ controls the underlying geometry. 
We map a graph representation
$z\in\mathcal M_c$ to the tangent space at the origin by
$v=\mathrm{Log}_0^c(z)\in T_0\mathcal M_c$ and use the polar decomposition
\begin{equation}
    r = \|\mathbf{v}\|_{g_o}, \quad \mathbf{u} = {\mathbf{v}}\big/{\|\mathbf{v}\|_2},\nonumber
\end{equation}
where the radius $r$ acts as a structure-related scalar proxy and $\mathbf{u}$
records normalized tangent direction. Rather than assuming all structure lies in
$r$, we use radial magnitude as a compact topology-related summary and calibrate
the target radial profile to the source. We formulate this as marginal matching
and minimize the univariate Wasserstein
distance~\citep{courty2016optimal,peyre2019computational}, computed by sorting.
For batch size $B$, let $\mathcal{R}_S=\text{sort}(\{r_S^{(i)}\}_{i=1}^B)$ and
$\mathcal{R}_T=\text{sort}(\{r_T^{(i)}\}_{i=1}^B)$ denote radial order
statistics, and define:
\begin{equation}
    \mathcal{L}_{rad} = \frac{1}{B} \sum_{k=1}^{B} \left| \mathcal{R}_S[k] - \mathcal{R}_T[k] \right|.
    \label{eq:radial_loss}
\end{equation}
Minimizing $\mathcal{L}_{rad}$ aligns the radial statistics of the target with those of the source. This regularizer reduces the risk of structural collapse and preserves topology-related diversity during adaptation.

\textbf{Angular Semantic Alignment.}
\label{subsec:angular_alignment}
Let $\mathbf{W} \in \mathbb{R}^{K \times d}$ denote the source classifier weights, and let $\mathbf{w}_k$ be a prototype for class $k$. For a target graph with tangent representation $\mathbf{v}_T^{(i)}$, we define cosine similarities
\begin{equation}
s_{i,k}
    =
    \frac{\mathbf{v}_T^{(i)}}{\|\mathbf{v}_T^{(i)}\|_2}
    \cdot
    \left(
    \frac{\mathbf{w}_k}{\|\mathbf{w}_k\|_2}
    \right)^\top,\nonumber
    \end{equation}
so semantic matching depends on direction rather than scale.
Since target ground-truth labels are unavailable, we rely on the source classifier to generate pseudo-labels~\citep{lee2013pseudo}. Recognizing that indiscriminate usage of noisy pseudo-labels risks distorting the manifold geometry~\citep{zou2018unsupervised, chen2019progressive}, we use pseudo-labels $\hat y_i=\arg\max_k p(k|v_T^{(i)})$ but
filter uncertain targets with:
    \begin{equation}
    \begin{aligned}
        M_i = \mathbb{I}\left(\max_{k} p(k|\mathbf{v}_T^{(i)}) > \zeta\right), \nonumber
        \end{aligned}
    \end{equation}
where $p(\cdot\mid\mathbf v_T^{(i)}) = \operatorname{Softmax}(\mathbf{W}\mathbf{v}_T^{(i)}+\mathbf b)$ denotes class probability, and $\mathbb{I}(\cdot)$ is the indicator function. This mask $M_i$ retains samples with confidence exceeding $\zeta$, filtering out uncertain predictions. The confidence score is computed from the full tangent representation but only used to screen pseudo-labels; angular similarity remains scale-normalized. Because large-radius target samples often lie farther from the shared support under structural shift, we use radius only as a soft reliability prior for pseudo-label supervision, not as a reason to discard complex graphs. Specifically, we set $\alpha_i=\exp(-\tau r_T^{(i)})$, with $\tau=1$ after radial normalization. The angular loss is formulated as:
\begin{equation}
    \mathcal{L}_{ang} = \frac{\sum_{i=1}^{B} M_i \cdot \alpha_i \cdot \ell_{ce}(\gamma \cdot \mathbf{s}_i, \hat{y}_i)}{\sum_{i=1}^{B} M_i \cdot \alpha_i + \epsilon},
\end{equation}
where $\ell_{ce}$ is the Cross-Entropy loss, $\mathbf{s}_i$ is the vector of cosine similarities, and $\gamma$ is a temperature scaling factor to sharpen the predictions. By minimizing $\mathcal{L}_{ang}$, we encourage semantic consistency in the angular domain while allowing the radial magnitude to modulate the reliability of pseudo-label supervision. In our formulation, angle acts as the primary carrier of semantic alignment, whereas radius plays an auxiliary reliability-regulation role rather than a hard semantic exclusion role.

\subsection{Topology-Conditioned Polar Flow Matching}
\label{subsec:tpfm}

The transport component uses one objective, not separate coupling and flow losses. Polar coupling determines admissible source--target pairs, while flow matching supervises transport geometry. It is not an additional regularizer; it is a stop-gradient sampling measure for the flow matching loss.

\textbf{Stopped Polar Coupling.}
For each class $y$, define source and reliable target index sets
    $\mathcal I_y^S=\{i:y_i^S=y\}$,
    $\mathcal I_y^T=\{j:M_j=1,\hat y_j=y\}$,
and let $\rho_j=\max_k p(k|\mathbf v_T^{(j)})$ denote target confidence. For
$i\in\mathcal I_y^S$ and $j\in\mathcal I_y^T$, we define the normalized polar
transport cost
\begin{equation}
\begin{aligned}
    C_{ij}^{(y)}
    =
    \widetilde d_{\mathcal M_c}(z_i^S,z_j^T)^2
    + |\widetilde r_i^S-\widetilde r_j^T|^2
    + \widetilde d_{\mathbb S}(u_i^S,u_j^T)^2
    - \log(\rho_j+\epsilon),\nonumber
\end{aligned}
\label{eq:tpfm_cost}
\end{equation}
where $\widetilde d_{\mathcal M_c}$ and $\widetilde d_{\mathbb S}$ are z-score-normalized pairwise distances within the class minibatch, and $\widetilde r$ is the z-score-normalized radius over pooled source--target samples of that class. For raw component $A$, we use
    $\widetilde A_{ij}=\frac{A_{ij}-\operatorname{sg}(\mu_A)}
    {\operatorname{sg}(\sigma_A)+\epsilon}$,
where $\mu_A,\sigma_A$ are minibatch statistics and $\operatorname{sg}(\cdot)$ stops gradients through them. The geometric, radial, angular, and confidence terms respectively enforce locality, match structure-sensitive scale, preserve class-compatible directions, and downweight unreliable pseudo-labels. All normalized components use unit weights, avoiding cost-balancing hyperparameters.

Let $\mu_y$ be uniform over $\mathcal I_y^S$ and let
$\nu_y(j)=\rho_j/\sum_{l\in\mathcal I_y^T}\rho_l$. If either set is empty, class
$y$ is skipped for that minibatch. Otherwise, we compute an entropic transport
plan
\begin{equation}
    \pi_y^\star
    =
    \arg\min_{\pi\in\Pi(\mu_y,\nu_y)}
    \langle \pi,C^{(y)}\rangle
    +\varepsilon_{\rm ot}
    \sum_{i,j}\pi_{ij}\big(\log \pi_{ij}-1\big),\nonumber
    \label{eq:tpfm_coupling}
\end{equation}
where $\Pi(\mu_y,\nu_y)$ is the set of couplings with marginals $\mu_y$ and $\nu_y$. The minibatch coupling is $\pi_B=\operatorname{Normalize}(\sum_y\pi_y^\star)$. Here $\operatorname{sg}(\cdot)$ denotes stop-gradient: it is the identity in the forward pass and blocks backward gradients. We use $\operatorname{sg}(\pi_B)$ in the loss to prevent unstable Sinkhorn gradients.

\textbf{Metric-Corrected Polar Flow Regression.}
For a coupled pair $(\mathbf{z}_i^S,\mathbf{z}_j^T)\sim\operatorname{sg}(\pi_B)$, define the
constant-speed geodesic interpolant
\begin{equation}
    \mathbf{z}_t
    =
    \gamma_{ij}(t)
    =
    \operatorname{Exp}_{\mathbf{z}_i^S}^{c}\!\left(
    t\operatorname{Log}_{\mathbf{z}_i^S}^{c}(\mathbf{z}_j^T)
    \right),
    \qquad t\sim\mathcal U[0,1].\nonumber
    \label{eq:geodesic_interpolation}
\end{equation}
Let $\xi_t$ be its analytic velocity,
   $\xi_t=P_{\mathbf{z}_i^S\to \mathbf{z}_t}^{c}\!\left(\operatorname{Log}_{\mathbf{z}_i^S}^{c}(\mathbf{z}_j^T)\right)$,
where $P_{x\to y}^{c}$ denotes parallel transport. Writing $\mathbf{z}_t$ in polar coordinates as $(r_t,u_t)$ with $u_t=u(\mathbf{z}_t)$, the radial and angular velocities are:
\begin{equation}
    \dot r_t = d r_{\mathbf{z}_t}[\xi_t],
    \qquad
    \dot u_t = d u_{\mathbf{z}_t}[\xi_t]\in T_{u_t}\mathbb S^{d-1}.\nonumber
    \label{eq:target_polar_velocity}
\end{equation}
We parameterize the learned vector field in polar form,
\begin{equation}
    \mathbf v_\psi(z_t,t)
    =
    a_\psi(\mathbf{z}_t,t)\partial_r
    +
    b_\psi(\mathbf{z}_t,t),
    \qquad
    b_\psi(\mathbf{z}_t,t)\in T_{u_t}\mathbb S^{d-1},\nonumber
    \label{eq:polar_vector_field}
\end{equation}
where $a_\psi$ predicts radial speed and $b_\psi$ predicts angular speed. The
unified topology-conditioned polar flow matching loss is
\begin{equation}
\begin{aligned}
    \mathcal L_{\rm FM}
    =
    \mathbb E_{t\sim\mathcal U[0,1],\,(\mathbf{z}_i^S,\mathbf{z}_j^T)\sim\operatorname{sg}(\pi_B)}
    \Big[
    |a_\psi(\mathbf{z}_t,t)-\dot r_t|^2
    +S_c(r_t)^2
    \|b_\psi(\mathbf{z}_t,t)-\dot u_t\|_{\mathbb S}^{2}
    \Big].
\end{aligned}
\label{eq:tpfm_loss}
\end{equation}
The factor $S_c(r_t)^2$ is the angular block of the polar metric in
Eq.~\eqref{eq:polar_metric}. Therefore, \method{} learns one manifold flow whose
velocity error is decomposed and measured by the correct curvature-dependent
geometry. This distinguishes our objective from generic Euclidean or Riemannian
flow matching, while keeping the coupling and flow supervision in a single loss.

\subsection{Learning Framework}
\label{sec:framework}

\method{} separates endpoint shaping from trajectory learning. The radial and
angular losses regularize endpoint distributions in polar coordinates, while
flow matching learns a single coupled manifold transport. This distinction is
important: independent radial and angular flows would ignore the metric coupling
and can produce distorted, non-geodesic trajectories.

Formally, \method{} jointly optimizes the source task capability, geometric structural integrity, and the manifold transport trajectory. The objective function is formulated as:
\begin{equation}
    \mathcal{L}_{total} = \mathcal{L}_{task} + \lambda_{1}\mathcal{L}_{rad} + \lambda_{2}\mathcal{L}_{ang} +\lambda_{3}\mathcal{L}_{FM},
    \label{eq:total_loss}
\end{equation}
where $\lambda_{1}, \lambda_{2}$, and $\lambda_{3}$ are hyperparameters balancing the contributions of radial regularization, angular regularization, and manifold transport, respectively. We note that $\mathcal{L}_{task}$ is computed from the full tangent representation, so the radial-angular decomposition functions as a regularization bias rather than a hard restriction on all prediction signals. Specifically, $\mathcal{L}_{task}$ is the standard cross-entropy loss computed on the labeled source domain $\mathcal{D}_S$ to ensure discriminative capability:
\begin{equation}
    \mathcal{L}_{task} = \mathbb{E}_{(\mathcal G_S, y_S) \sim \mathcal{D}_S} \left[ -\log p(y_S | \mathbf{v}_S) \right],\quad
    \mathbf v_S=\operatorname{Log}_o^c(f_\theta(\mathcal G_S)),\nonumber
    \label{eq:task_loss}
\end{equation}
where $p(y|\mathbf{v}_S) = \operatorname{Softmax}(\mathbf{W}\mathbf{v}_S + \mathbf{b})$ denotes the predicted probability in the tangent space.

\subsection{Theoretical Analysis}
\label{sec:theory}

Our analysis is conditional: it identifies when topology-conditioned polar
transport is beneficial, rather than claiming universal topology-semantics
identifiability or global convergence of nonconvex training. Let
$\mathbf{z}=f_\theta(G)\in\mathcal M_c$ be the Riemannian representation of $G$, and
write $\mathbf{z}=(r,u)$ in polar coordinates around $o$. Let $a(G)\in\mathbb R^m$ denote
topology-related graph statistics and $y$ the class label.

\begin{theorem}[A Structural Failure Mode of Domain Confusion]
\label{thm:domain_confusion_failure}
Consider a binary structural variable $A\in\{0,1\}$ with equal priors
$\Pr(A=0)=\Pr(A=1)=1/2$, label $Y=A$, and domain variable $D\in\{0,1\}$ with
$\Pr(D=A)=1-\eta$ for some $0\le\eta<1/2$. Assume a representation $Z$ depends
on the graph only through $A$ in this toy model. If an adversarial objective
enforces approximate domain invariance
\[
    {\rm TV}\big(P_{Z|D=0},P_{Z|D=1}\big)\le \delta,
\]
then the total variation between the label-conditionals of $Z$ is bounded by
\[
    {\rm TV}\big(P_{Z|Y=0},P_{Z|Y=1}\big)
    \le
    \frac{\delta}{1-2\eta}.
\]
Consequently, any classifier using $Z$ has error at least
\[
    \inf_h\Pr(h(Z)\neq Y)
    \ge
    \frac12\left(1-\frac{\delta}{1-2\eta}\right).
\]
\end{theorem}
Theorem~\ref{thm:domain_confusion_failure} formalizes structural degeneration.
If label-relevant structure correlates with the domain, forcing source and
target representations to be indistinguishable can make them uninformative
about the label. \method{} avoids this failure mode by not imposing
unconditional domain confusion. Instead, it aligns class-conditioned source and
target distributions through topology-aware polar transport.

\begin{definition}[Operational Polar Representation]
\label{def:operational_polar_disentanglement}
A representation $\mathbf{z}=f_\theta(G)$ is
$(\epsilon_s,\epsilon_y,\epsilon_g)$ operationally polar-disentangled for a graph domain adaptation task if the following conditions hold:
\begin{itemize}
    \item \text{{(i) Structure sufficiency:}}
$\inf_{q_s:\mathbb R_+\to\mathbb R^m}
\mathbb E\big[\|a(G)-q_s(r(G))\|\big]\le \epsilon_s$,
\item {{(ii) Semantic sufficiency:}}
$\inf_{q_y:\mathbb S^{d-1}\to\Delta^{K-1}}
\mathbb E\big[\ell_{\rm ce}(q_y(u(G)),y)\big]\le \epsilon_y$,
\item
\text{{(iii) Low first-order interference:}}
${\rm GC}(\mathcal L_{\rm rad},\mathcal L_{\rm ang})
:=
\frac{
\left|
\left\langle
\nabla_g \mathcal L_{\rm rad},
\nabla_g \mathcal L_{\rm ang}
\right\rangle_g
\right|
}{
\|\nabla_g \mathcal L_{\rm rad}\|_g
\|\nabla_g \mathcal L_{\rm ang}\|_g+\delta
}
\le \epsilon_g$,
\end{itemize}
for a small numerical constant $\delta>0$.
\end{definition}
Definition~\ref{def:operational_polar_disentanglement} formalizes task-level
polar disentanglement. Conditions \textit{(i)} and \textit{(ii)} assign complementary coordinate
roles, while condition \textit{(iii)} bounds normalized Riemannian first-order interaction
between radial and angular endpoint objectives. The definition captures the
coordinate-wise bias induced by our regularizers without assuming a globally
identifiable topology--semantics factorization.

\begin{theorem}[First-Order Orthogonality of Polar Endpoint Losses]
\label{thm:polar_disentanglement}
Let $\mathcal M_c$ be a $d$-dimensional constant-curvature manifold and let
$o\in\mathcal M_c$ be the origin. Work on a geodesic normal ball $B_o(\rho)$,
with $\rho<\pi/\sqrt c$ when $c>0$, and exclude $o$. For
$\mathbf{z}=(r,u)\in B_o(\rho)\setminus\{o\}$, the polar metric is Eq.~\eqref{eq:polar_metric}.
For a batch $\mathbf z=(\mathbf{z}_i)_{i=1}^B$, equip $\mathcal M_c^B$ with the product
metric. Let $\mathcal L_{\rm rad}(\mathbf z)=\ell_r(r_{1:B})$ and
$\mathcal L_{\rm ang}(\mathbf z)=\ell_u(u_{1:B})$ be almost-everywhere
differentiable radial-only and angular-only endpoint losses. Then, wherever the
gradients exist,
\[
    \left\langle
    \nabla_g\mathcal L_{\rm rad},
    \nabla_g\mathcal L_{\rm ang}
    \right\rangle_g=0.
\]
If a differentiable radial reliability function weights the angular loss $\mathcal L_{\rm ang}^{w}=\sum_i\alpha(r_i)\ell_i(u_i)$, then
\[
\left|
\left\langle
\nabla_g\mathcal L_{\rm rad},
\nabla_g\mathcal L_{\rm ang}^{w}
\right\rangle_g
\right|
\le
\sum_{i=1}^{B}
|\partial_{r_i}\ell_r|\,|\alpha'(r_i)|\,|\ell_i(u_i)|.
\]
For the choice $\alpha(r)=e^{-r}$, the coupling term is bounded by
$\sum_i |\partial_{r_i}\ell_r|e^{-r_i}|\ell_i(u_i)|$.
\end{theorem}
Theorem~\ref{thm:polar_disentanglement} justifies endpoint regularization: pure radial and pure angular endpoint losses are first-order orthogonal under the polar metric. The practical radius-weighted angular loss is not exactly orthogonal, but the reliability-weight derivative explicitly bounds its additional first-order interaction.

\begin{theorem}[Radial Structural Proxy Bound]
\label{thm:radial_proxy}
Let $P_S$ and $P_T$ be source and target graph distributions. For a graph $G$,
let $r(G)=\|\operatorname{Log}_o^c(f_\theta(G))\|_{g_o}$ and let
$a(G)\in\mathbb R^m$ denote topology-related graph statistics. Denote by
$r_\#P$ and $a_\#P$ the pushforward distributions induced by $r(G)$ and $a(G)$.
Suppose there exists an $L_a$-Lipschitz map $h:\mathbb R_+\to\mathbb R^m$ such
that, for $D\in\{S,T\}$,
    $\mathbb E_{G\sim P_D}[\|a(G)-h(r(G))\|]\le\epsilon_D$.
Then
\[
    W_1(a_\#P_S,a_\#P_T)
    \le
    L_a W_1(r_\#P_S,r_\#P_T)+\epsilon_S+\epsilon_T.
\]
For equal-size empirical distributions,
    $W_1(\widehat P_S^r,\widehat P_T^r)
    =
    \frac1B\sum_{k=1}^B |\mathcal R_S[k]-\mathcal R_T[k]|$,
which is Eq.~\eqref{eq:radial_loss}.
\end{theorem}

Theorem~\ref{thm:radial_proxy} gives a conditional, testable guarantee: when
radius predicts the chosen structural statistics with small error, radial
calibration controls the corresponding structural discrepancy. It does not claim
that radius automatically captures every graph-structural factor.

\begin{theorem}[Angular Margin Preservation]
\label{thm:angular_margin}
Let $\{\hat w_c\}_{c=1}^K\subset\mathbb S^{d-1}$ be normalized class prototypes,
and define
    $h(u)=\arg\max_{c\in[K]} \hat w_c^\top u$.
For a labeled sample $(\mathbf{z},y)$, let $u=u(\mathbf{z})$ and let
$\tilde u$ be the angular coordinate after an endpoint transport map. Define
the angular margin
\[
    m_y(u)
    =
    \hat w_y^\top u
    -
    \max_{c\ne y}\hat w_c^\top u .
\]
If 
    $m_y(u)\ge \mu$
    and
    $d_{\mathbb S}(u,\tilde u)<\frac{\mu}{2}$ for some $\mu>0$, then
    $h(\tilde u)=y$.
For a labeled distribution $P$,
\[
    R_{\rm post}(h)
    \le
    R_{\rm pre}(h)
    +
    \rho_\mu
    +
    \frac{2}{\mu}\Delta_{\rm ang},
\]
where $R_{\rm pre}(h)=\Pr_{P}[h(u)\ne y]$,
    $R_{\rm post}(h)=\Pr_{P}[h(\tilde u)\ne y],$
    $\rho_\mu=\Pr_{P}\!\left[0\le m_y(u)<\mu\right]$, and
    $\Delta_{\rm ang}=\mathbb E_{P}\!\left[d_{\mathbb S}(u,\tilde u)\right]$.
\end{theorem}

Theorem~\ref{thm:angular_margin} explains why class-conditional angular alignment is preferable to unconditional domain alignment: labels are preserved when angular drift is small relative to the source prototype margin.

\begin{theorem}[Polar Flow Endpoint Stability]
\label{thm:fm_endpoint_stability}
Fix a coupling $\pi$ and assume all analytic geodesics and learned flow
trajectories stay inside a compact geodesic normal ball $B_o(\rho)$. Let
$\gamma_t(z_S,z_T)$ be the analytic geodesic and let $\varphi_t^\psi$ be the
trajectory generated by the learned polar field in Eq.~\eqref{eq:polar_vector_field}
with $\varphi_0^\psi=z_S$. Assume the reconstructed manifold vector field is
$L_v$-Lipschitz in $z$ on $B_o(\rho)$. Then there exists a constant
$C_{\rm geo}$ depending only on the ball and the metric such that
\[
    \mathbb E_{(z_S,z_T)\sim\pi}
    \left[d_{\mathcal M_c}(\varphi_1^\psi,z_T)\right]
    \le
    C_{\rm geo}e^{L_v}\sqrt{\mathcal L_{\rm FM}}.
\]
\end{theorem}

This result follows from a Riemannian Gronwall argument: endpoint deviation is controlled by time-integrated metric-corrected velocity-regression error under a fixed population coupling. It does not imply global optimal transport or convergence of nonconvex training; minibatch Sinkhorn couplings add pseudo-label, coupling-estimation, finite-sample, and numerical-integration errors.

\begin{theorem}[Target-Risk Bound with Polar Discrepancy]
\label{thm:target_risk_polar}
Let $P_S$ and $P_T$ be source and target distributions over graph-label pairs.
Assume $\mathbf{z}=f_\theta(G)\in B_o(\rho)\subset\mathcal M_c$ and the loss
$\ell_\phi(\mathbf{z},y)=\ell(g_\phi(\mathbf{z}),y)\in[0,1]$ is $L_\ell$-Lipschitz in $z$. Let
$\Phi_\psi$ be the endpoint map induced by the learned polar flow and let
$\Psi$ be an ideal class-conditioned transport map. For each class $y$, let
$\Gamma_y$ be a lifted coupling over source--target graph pairs $(G,G')$ such
that the induced representation pair
$\mathbf z=\Psi(f_\theta(G))$ and $\mathbf z'=f_\theta(G')$ has marginals
$\Psi_\#P_S(\mathbf z\mid y)$ and $P_T(\mathbf z\mid y)$, respectively. Define
\[
D_{\rm rad}^{\Gamma}=
\sum_y\omega_y\mathbb E_{(G,G')\sim\Gamma_y}|r(\Psi(f_\theta(G)))-r(f_\theta(G'))|,
\; D_{\rm top}^{\Gamma}=
\sum_y\omega_y\mathbb E_{(G,G')\sim\Gamma_y}
\|q(G)-q(G')\|_2,
\]
\[
D_{\rm ang}^{\Gamma}=
\sum_y\omega_y\mathbb E_{(G,G')\sim\Gamma_y}d_{\mathbb S}(u(\Psi(f_\theta(G))),u(f_\theta(G'))),
\]
where $\omega_y=P_T(Y=y)$. Let $C_\rho=\sup_{0\le t\le\rho}S_c(t)$ and let
$\lambda_{\rm lab}=\frac12\sum_y|P_S(Y=y)-P_T(Y=y)|$. If pseudo-labels used to
construct the target class-conditionals have error $\epsilon_{\rm pl}$, then
there exist constants $C_{\rm top}$, $C_{\rm FM}$, and $C_{\rm pl}$ such that
\[
\begin{aligned}
    R_T(g_\phi)
    \le
    R_S(g_\phi\circ\Phi_\psi)
    +L_\ell\Big(
        D_{\rm rad}^{\Gamma}
        +C_\rho D_{\rm ang}^{\Gamma}
        +C_{\rm top}D_{\rm top}^{\Gamma}
        +C_{\rm FM}\sqrt{\mathcal L_{\rm FM}}
    \Big)+C_{\rm pl}\epsilon_{\rm pl}
    +\lambda_{\rm lab}.
\end{aligned}
\]
\end{theorem}

Theorem~\ref{thm:target_risk_polar} shows that the target risk is controlled by the transported-source risk, radial discrepancy, angular class-conditional discrepancy, topology-conditioned transport discrepancy, flow endpoint error, pseudo-label noise, and label-prior mismatch. This is a conditional population bound: the empirical objective in Eq.~\eqref{eq:total_loss} is a tractable surrogate, while minibatch coupling, pseudo-labeling, and finite-sample effects remain approximation terms rather than being hidden in the notation.

%% file: code/4_experiments.tex
\section{Experiments}

\begin{table*}[t]
\small
\caption{Graph classification results (in \%) under node and edge density domain shifts on the Mutagenicity dataset, and feature domain shifts on DD, PROTEINS, BZR, BZR\_MD, COX2, and COX2\_MD. For convenience, PROTEINS, DD, COX2, COX2\_MD, BZR, and BZR\_MD are abbreviated as P, D, C, CM, B, and BM, respectively. \textbf{Bold} results indicate the best performance.}
\resizebox{1.0\textwidth}{!}{
\begin{tabular}{l|c|c|c|c|c|c|c|c|c|c|c|c}
\toprule
\textbf{Methods} & \multicolumn{3}{c|}{\textbf{Node Shift}} & \multicolumn{3}{c|}{\textbf{Edge Shift}} & \multicolumn{6}{c}{\textbf{Feature Shift}} \\
\cmidrule(lr){2-4} \cmidrule(lr){5-7} \cmidrule(lr){8-13}
& M0$\rightarrow$M1 & M0$\rightarrow$M2 & M0$\rightarrow$M3
& M0$\rightarrow$M1 & M0$\rightarrow$M2 & M0$\rightarrow$M3
& P$\rightarrow$D & D$\rightarrow$P & C$\rightarrow$CM & CM$\rightarrow$C & B$\rightarrow$BM & BM$\rightarrow$B \\
\midrule
WL subtree & 34.3 & 40.4 & 52.7 & 34.4 & 47.6 & 52.7 & 43.0 & 42.2 & 53.1 & 58.2 & 51.3 & 44.0 \\
GCN & 64.1\scriptsize{$\pm$1.4} & 65.5\scriptsize{$\pm$2.0} & 56.9\scriptsize{$\pm$2.1} & 66.3\scriptsize{$\pm$1.7} & 63.6\scriptsize{$\pm$1.4} & 56.0\scriptsize{$\pm$1.4} & 48.9\scriptsize{$\pm$2.0} & 60.9\scriptsize{$\pm$2.3} & 51.2\scriptsize{$\pm$1.8} & 66.9\scriptsize{$\pm$1.8} & 48.7\scriptsize{$\pm$2.0} & 78.8\scriptsize{$\pm$1.7} \\
GIN & 66.5\scriptsize{$\pm$2.1} & 52.0\scriptsize{$\pm$1.7} & 53.7\scriptsize{$\pm$1.7} & 67.1\scriptsize{$\pm$1.7} & 54.2\scriptsize{$\pm$2.6} & 55.4\scriptsize{$\pm$1.9} & 57.3\scriptsize{$\pm$2.2} & 61.9\scriptsize{$\pm$1.9} & 53.8\scriptsize{$\pm$2.5} & 55.6\scriptsize{$\pm$2.0} & 49.9\scriptsize{$\pm$2.4} & 79.2\scriptsize{$\pm$2.8} \\
GMT & 65.7\scriptsize{$\pm$1.8} & 62.1\scriptsize{$\pm$2.1} & 59.0\scriptsize{$\pm$2.0} & 67.9\scriptsize{$\pm$1.3} & 61.5\scriptsize{$\pm$1.8} & 58.2\scriptsize{$\pm$2.4} & 59.5\scriptsize{$\pm$2.5} & 50.7\scriptsize{$\pm$2.2} & 49.3\scriptsize{$\pm$1.8} & 58.2\scriptsize{$\pm$2.0} & 50.2\scriptsize{$\pm$2.3} & 74.4\scriptsize{$\pm$1.8} \\
CIN & 65.1\scriptsize{$\pm$1.7} & 66.0\scriptsize{$\pm$1.7} & 55.2\scriptsize{$\pm$1.5} & 66.3\scriptsize{$\pm$1.8} & 60.8\scriptsize{$\pm$1.7} & 55.8\scriptsize{$\pm$2.4} & 59.1\scriptsize{$\pm$2.6} & 58.0\scriptsize{$\pm$2.7} & 51.2\scriptsize{$\pm$2.0} & 55.6\scriptsize{$\pm$1.5} & 49.2\scriptsize{$\pm$1.4} & 74.2\scriptsize{$\pm$1.9} \\
PathNN & 70.2\scriptsize{$\pm$1.5} & 67.1\scriptsize{$\pm$2.0} & 58.0\scriptsize{$\pm$1.9} & 68.9\scriptsize{$\pm$1.9} & 62.9\scriptsize{$\pm$1.7} & 58.1\scriptsize{$\pm$1.6} & 57.9\scriptsize{$\pm$1.8} & 53.8\scriptsize{$\pm$3.3} & 49.8\scriptsize{$\pm$1.7} & 66.9\scriptsize{$\pm$2.5} & 50.3\scriptsize{$\pm$2.3} & 75.3\scriptsize{$\pm$2.2} \\
\midrule
dDGM & 79.1\scriptsize{$\pm$0.4} & 70.3\scriptsize{$\pm$0.5} & 63.9\scriptsize{$\pm$0.3} & 76.5\scriptsize{$\pm$1.3} & 69.0\scriptsize{$\pm$1.3} & 67.0\scriptsize{$\pm$0.4} & 59.0\scriptsize{$\pm$1.8} & 63.3\scriptsize{$\pm$1.5} & 58.4\scriptsize{$\pm$2.7} & 74.3\scriptsize{$\pm$1.2} & 53.6\scriptsize{$\pm$1.6} & 79.2\scriptsize{$\pm$1.5} \\
RieGrace & 78.5\scriptsize{$\pm$1.3} & 70.6\scriptsize{$\pm$1.4} & 62.8\scriptsize{$\pm$1.3} & 76.4\scriptsize{$\pm$0.5} & 68.9\scriptsize{$\pm$0.6} & 66.1\scriptsize{$\pm$1.2} & 59.8\scriptsize{$\pm$1.4} & 63.8\scriptsize{$\pm$1.3} & 56.8\scriptsize{$\pm$2.1} & 66.6\scriptsize{$\pm$1.4} & 52.2\scriptsize{$\pm$2.3} & 78.9\scriptsize{$\pm$2.0} \\
ProGDM & 75.2\scriptsize{$\pm$1.2} & 69.0\scriptsize{$\pm$1.1} & 55.7\scriptsize{$\pm$1.4} & 74.6\scriptsize{$\pm$1.2} & 67.5\scriptsize{$\pm$1.2} & 55.5\scriptsize{$\pm$0.7} & 58.7\scriptsize{$\pm$2.0} & 59.6\scriptsize{$\pm$2.3} & 54.2\scriptsize{$\pm$1.7} & 75.2\scriptsize{$\pm$2.0} & 52.6\scriptsize{$\pm$1.3} & 75.4\scriptsize{$\pm$2.9} \\
D-GCN & 75.2\scriptsize{$\pm$0.5} & 67.9\scriptsize{$\pm$0.4} & 56.3\scriptsize{$\pm$1.1} & 73.8\scriptsize{$\pm$1.2} & 66.8\scriptsize{$\pm$1.4} & 56.7\scriptsize{$\pm$1.0} & 62.4\scriptsize{$\pm$1.9} & 59.7\scriptsize{$\pm$2.1} & 53.7\scriptsize{$\pm$1.8} & 77.6\scriptsize{$\pm$1.9} & 51.7\scriptsize{$\pm$1.6} & 75.7\scriptsize{$\pm$2.0} \\
\midrule
DEAL & 77.1\scriptsize{$\pm$0.9} & 70.9\scriptsize{$\pm$0.9} & 60.3\scriptsize{$\pm$1.1} & 76.6\scriptsize{$\pm$1.6} & 70.6\scriptsize{$\pm$1.2} & 60.2\scriptsize{$\pm$2.1} & 61.7\scriptsize{$\pm$2.0} & 60.0\scriptsize{$\pm$1.5} & 52.7\scriptsize{$\pm$2.7} & 69.4\scriptsize{$\pm$2.9} & 52.4\scriptsize{$\pm$2.9} & 78.6\scriptsize{$\pm$1.4} \\
SGDA & 77.5\scriptsize{$\pm$0.6} & 69.7\scriptsize{$\pm$0.5} & 65.5\scriptsize{$\pm$0.8} & 75.9\scriptsize{$\pm$1.6} & 68.9\scriptsize{$\pm$0.8} & 64.4\scriptsize{$\pm$0.4} & 48.3\scriptsize{$\pm$2.0} & 55.8\scriptsize{$\pm$2.6} & 49.8\scriptsize{$\pm$1.8} & 66.9\scriptsize{$\pm$2.3} & 50.3\scriptsize{$\pm$2.1} & 78.8\scriptsize{$\pm$2.6} \\
A2GNN & 73.5\scriptsize{$\pm$1.9} & 66.1\scriptsize{$\pm$1.5} & 60.4\scriptsize{$\pm$1.1} & 69.5\scriptsize{$\pm$1.4} & 68.6\scriptsize{$\pm$1.4} & 58.8\scriptsize{$\pm$2.2} & 57.8\scriptsize{$\pm$2.1} & 60.3\scriptsize{$\pm$1.5} & 51.5\scriptsize{$\pm$1.8} & 67.7\scriptsize{$\pm$2.1} & 51.6\scriptsize{$\pm$2.3} & 77.5\scriptsize{$\pm$1.9} \\
StruRW & 78.3\scriptsize{$\pm$1.3} & 69.7\scriptsize{$\pm$1.3} & 62.6\scriptsize{$\pm$0.7} & 76.1\scriptsize{$\pm$1.5} & 69.0\scriptsize{$\pm$1.3} & 62.1\scriptsize{$\pm$1.0} & 59.1\scriptsize{$\pm$2.3} & 58.8\scriptsize{$\pm$2.8} & 51.2\scriptsize{$\pm$2.0} & 54.8\scriptsize{$\pm$2.9} & 49.2\scriptsize{$\pm$1.4} & 74.7\scriptsize{$\pm$2.1} \\
PA-BOTH & 69.8\scriptsize{$\pm$1.5} & 63.8\scriptsize{$\pm$1.9} & 55.3\scriptsize{$\pm$1.1} & 74.7\scriptsize{$\pm$1.1} & 65.3\scriptsize{$\pm$1.3} & 52.2\scriptsize{$\pm$1.5} & 54.2\scriptsize{$\pm$3.2} & 56.7\scriptsize{$\pm$2.6} & 52.9\scriptsize{$\pm$2.8} & 61.8\scriptsize{$\pm$2.0} & 47.5\scriptsize{$\pm$3.0} & 78.8\scriptsize{$\pm$1.9} \\
GAA & 79.3\scriptsize{$\pm$1.2} & 71.2\scriptsize{$\pm$0.7} & 65.6\scriptsize{$\pm$1.3} & 77.5\scriptsize{$\pm$1.2} & 70.0\scriptsize{$\pm$1.2} & 66.5\scriptsize{$\pm$1.3} & 62.4\scriptsize{$\pm$0.6} & 64.1\scriptsize{$\pm$0.8} & 59.4\scriptsize{$\pm$1.8} & 78.4\scriptsize{$\pm$1.2} & 57.2\scriptsize{$\pm$3.3} & 78.5\scriptsize{$\pm$3.0} \\
TDSS & 63.6\scriptsize{$\pm$1.3} & 56.7\scriptsize{$\pm$1.6} & 54.7\scriptsize{$\pm$1.0} & 71.6\scriptsize{$\pm$1.5} & 67.3\scriptsize{$\pm$1.0} & 55.4\scriptsize{$\pm$1.6} & 61.9\scriptsize{$\pm$1.1} & 63.6\scriptsize{$\pm$1.9} & 56.8\scriptsize{$\pm$1.3} & 77.0\scriptsize{$\pm$2.6} & 56.6\scriptsize{$\pm$1.1} & 79.1\scriptsize{$\pm$2.4} \\
\midrule
GOTDA & 77.3\scriptsize{$\pm$1.3} & 69.9\scriptsize{$\pm$1.7} & 65.3\scriptsize{$\pm$1.2} & 76.1\scriptsize{$\pm$1.2} & 68.5\scriptsize{$\pm$0.8} & 65.8\scriptsize{$\pm$0.6} & 62.1\scriptsize{$\pm$1.5} & 63.2\scriptsize{$\pm$1.7} & 58.5\scriptsize{$\pm$2.1} & 78.2\scriptsize{$\pm$1.1} & 57.0\scriptsize{$\pm$1.7} & 78.9\scriptsize{$\pm$2.6} \\
MASH & 76.1\scriptsize{$\pm$1.3} & 66.9\scriptsize{$\pm$1.3} & 54.3\scriptsize{$\pm$1.2} & 74.4\scriptsize{$\pm$1.2} & 66.1\scriptsize{$\pm$0.5} & 52.2\scriptsize{$\pm$1.2} & 61.8\scriptsize{$\pm$2.1} & 64.2\scriptsize{$\pm$2.4} & 56.2\scriptsize{$\pm$2.0} & 77.2\scriptsize{$\pm$2.1} & 55.2\scriptsize{$\pm$2.0} & 78.1\scriptsize{$\pm$2.3} \\
GeoAdapt & 76.9\scriptsize{$\pm$1.5} & 68.7\scriptsize{$\pm$1.4} & 60.1\scriptsize{$\pm$1.3} & 75.5\scriptsize{$\pm$1.2} & 67.5\scriptsize{$\pm$1.2} & 60.5\scriptsize{$\pm$0.6} & 63.1\scriptsize{$\pm$1.2} & 64.0\scriptsize{$\pm$1.4} & 56.7\scriptsize{$\pm$1.7} & 78.3\scriptsize{$\pm$2.1} & 57.5\scriptsize{$\pm$2.7} & 78.5\scriptsize{$\pm$2.2} \\
\midrule
\method{}-$\mathbb{E}^n$ & 79.4\scriptsize{$\pm$0.7} & 71.0\scriptsize{$\pm$1.0} & 65.2\scriptsize{$\pm$0.2} & 77.6\scriptsize{$\pm$0.4} & 70.2\scriptsize{$\pm$1.2} & 66.1\scriptsize{$\pm$0.3} & 61.9\scriptsize{$\pm$1.0} & 63.0\scriptsize{$\pm$2.0} & 56.5\scriptsize{$\pm$1.9} & 78.2\scriptsize{$\pm$1.0} & 56.8\scriptsize{$\pm$2.4} & 78.9\scriptsize{$\pm$1.3} \\
\method{}-$\mathbb{S}^n$ & 79.7\scriptsize{$\pm$0.7} & \textbf{72.0\scriptsize{$\pm$1.0}} & 65.9\scriptsize{$\pm$0.9} & 78.0\scriptsize{$\pm$0.5} & 71.0\scriptsize{$\pm$0.6} & \textbf{68.2\scriptsize{$\pm$1.0}} & \textbf{63.8\scriptsize{$\pm$1.2}} & 64.0\scriptsize{$\pm$2.7} & \textbf{59.7\scriptsize{$\pm$1.4}} & 78.4\scriptsize{$\pm$2.1} & 57.5\scriptsize{$\pm$1.8} & 79.1\scriptsize{$\pm$1.0} \\
\method{}-$\mathbb{H}^n$ & \textbf{79.8\scriptsize{$\pm$0.8}} & 71.7\scriptsize{$\pm$0.7} & \textbf{66.7\scriptsize{$\pm$0.8}} & \textbf{78.3\scriptsize{$\pm$0.7}} & \textbf{71.3\scriptsize{$\pm$0.4}} & 67.7\scriptsize{$\pm$0.9} & 63.3\scriptsize{$\pm$1.3} & \textbf{64.4\scriptsize{$\pm$1.6}} & 58.3\scriptsize{$\pm$2.2} & \textbf{79.0\scriptsize{$\pm$2.5}} & \textbf{58.0\scriptsize{$\pm$1.9}} & \textbf{79.4\scriptsize{$\pm$1.7}} \\
\bottomrule
\end{tabular}}
\vspace{-0.4cm}
\label{tab:node_edge_feature_shift_merged}
\end{table*}

\subsection{Experimental Settings}
\textbf{Dataset.}\label{dataset} We evaluate \method{} on two types of domain shift scenarios: (1) structure-based domain shifts: we partition the PROTEINS, NCI1, Mutagenicity~\citep{dobson2003distinguishing}, and ogbg-molhiv~\citep{hu2021ogblsc} datasets based on node and edge densities, following the protocol described in~\citep{yin2025dream}; (2) feature-based domain shifts: we evaluate \method{} on DD, PROTEINS, BZR, BZR\_MD, COX2, and COX2\_MD \citep{ sutherland2003spline} datasets. More details of experimental datasets are shown in Appendix~\ref{sec:dataset}.

\textbf{Baselines.}\label{baselines} We compare \method{} with competitive baselines on the above datasets, including two graph kernel methods: WL subtree \citep{shervashidze2011weisfeiler} and PathNN~\citep{michel2023path}; four general graph neural networks (GNNs): GCN \citep{kipf2017semi}, GIN \citep{xu2018how}, CIN \citep{bodnar2021weisfeiler} and GMT \citep{BaekKH21}; four manifold-based GNNs: dDGM~\citep{de2023latent}, RieGrace~\citep{sun2023self}, ProGDM~\citep{wang2024mixed}, and D-GCN~\citep{sun2024motif}; seven graph domain adaptation (GDA): DEAL~\citep{yin2022deal}, SGDA \citep{qiao2023semi}, StruRW \citep{liu2023structural}, A2GNN \citep{liu2024rethinking}, PA-BOTH \citep{liu2024pairwise}, GAA~\citep{fang2025benefits}, and TDSS~\citep{chen2025smoothness}; and three manifold-based GDA: GOTDA~\citep{long2022domain}, MASH~\citep{rustad2024graph}, and GeoAdapt~\citep{gharib2025geometric}. More settings are introduced in Appendix~\ref{sec:baselines}.

\subsection{Results Analysis}\label{sec:performance}

\textbf{Performance Comparison.} We evaluate \method{} against all baselines under diverse domain shifts. Results
are reported in Tables~\ref{tab:node_edge_feature_shift_merged} and
\ref{tab:proteins_node}--\ref{tab:hiv_idx}. We observe: First,
manifold-based GNNs outperform Euclidean GNNs, suggesting that non-Euclidean
geometry provides an inductive bias for graphs with hierarchical or relational
structure. However, manifold representation learning alone is insufficient, as
manifold-based adaptation baselines lag behind \method{} in many settings.
Second, \method{} achieves the strongest performance across node-density,
edge-density, and feature-shift scenarios. We attribute this to the two modules
of our framework: polar endpoint regularization reduces structural-scale
mismatch while preserving scale-normalized class semantics, and
topology-conditioned polar flow matching replaces adversarial minimax alignment
with stable geometry-aware transport supervision over compatible source--target
regions. Third, among the three geometries, the hyperbolic variant
\method{}-$\mathbb{H}^n$ obtains the best mean performance in 8 out of the 12
scenarios, indicating its advantage under hierarchy-rich structural shifts. The
spherical and Euclidean variants remain competitive, suggesting gains come not
only from curvature choice but also from the proposed polar endpoint and
transport objectives. More results and analysis are provided in
Appendix~\ref{sec:model performance}.

\begin{figure}[t]
    \centering
    \begin{minipage}[t]{0.56\linewidth}
        \centering
        \vspace{0pt}
        \includegraphics[width=\linewidth]{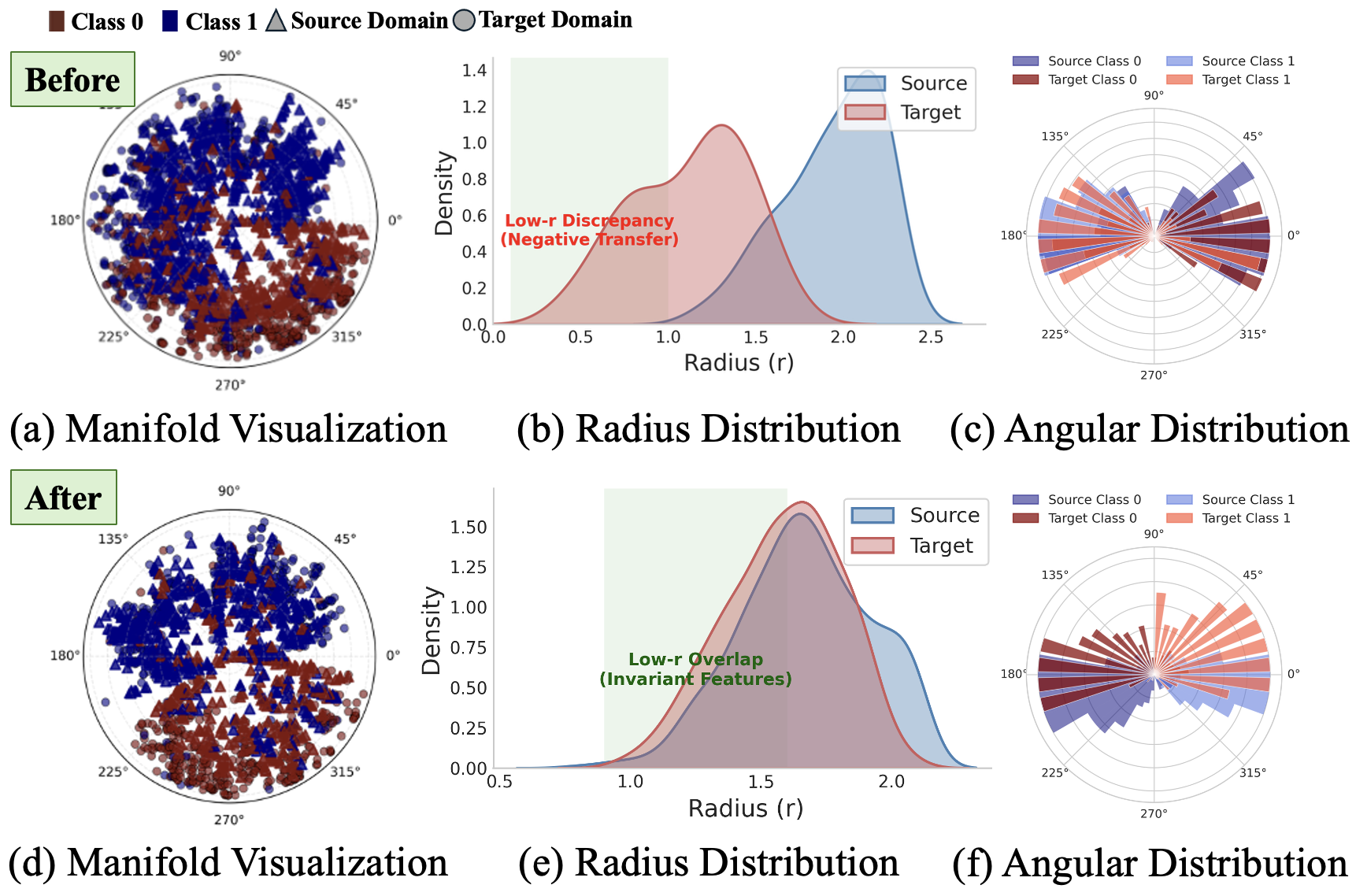}
        \captionof{figure}{Visualization of manifold geometry before (top) and after (bottom) alignment.}
        \label{fig:manifold_vis}
    \end{minipage}
    \hfill
    \begin{minipage}[t]{0.40\linewidth}
        \centering
        \vspace{0pt}
        \includegraphics[width=\linewidth]{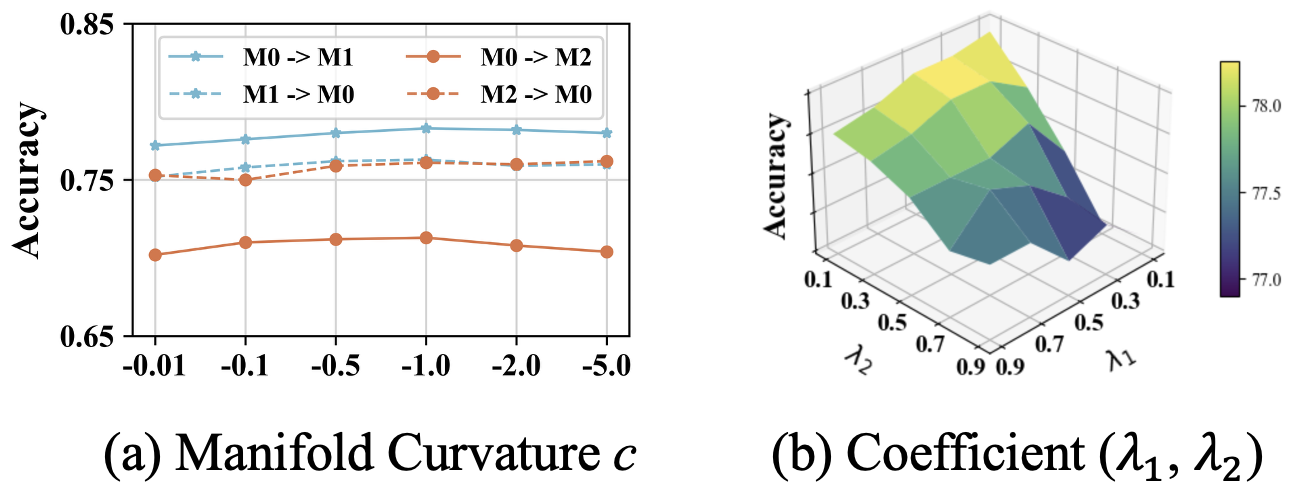}
        \vspace{-11pt}
        \captionof{figure}{Sensitivity to curvature $c$ and coefficients ($\lambda_1$, $\lambda_2$).}
        \label{fig:sensitivity}

        \vspace{4pt}

        \includegraphics[width=\linewidth]{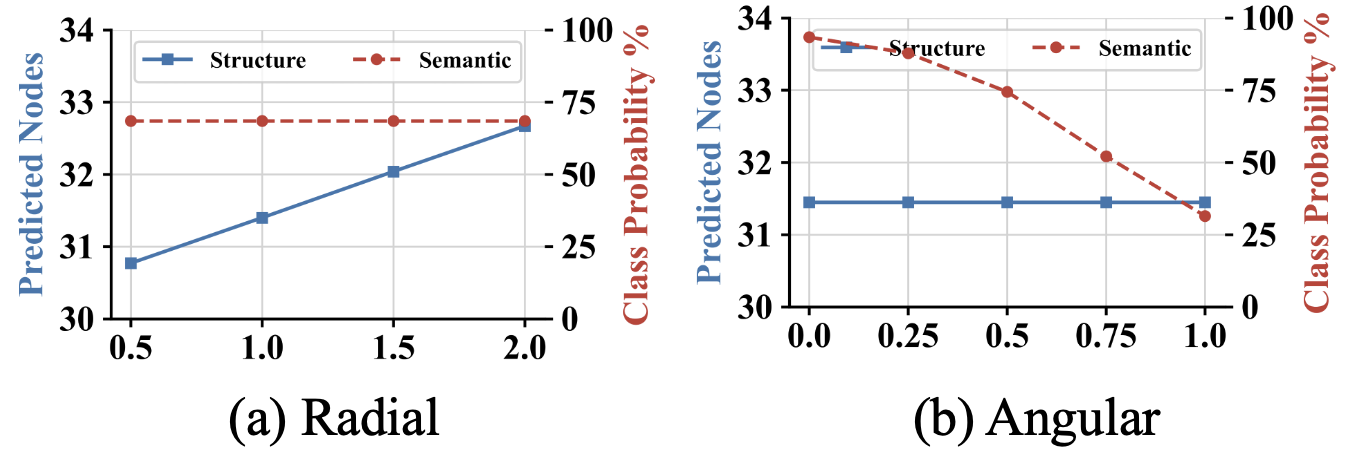}
        \vspace{-12pt}
        \captionof{figure}{Counterfactual disentanglement via radial and angular.}
        \label{fig:Intervention}
    \end{minipage}
    \vspace{-1.5cm}
\end{figure}

\textbf{Sensitivity Analysis.}
\label{sec:sensitivity} 
We conduct a sensitivity analysis of manifold curvature $c$ and balance
coefficients ($\lambda_1, \lambda_2$), shown in
Figure~\ref{fig:sensitivity}, to evaluate robustness of the chosen geometry and loss weighting under adaptation. As shown in Figure~\ref{fig:sensitivity}(a),
the model performs best at $c=-1.0$. This reflects a trade-off: curvatures
close to 0 fail to capture complex hierarchies due to limited Euclidean
capacity, while excessively large curvature causes geometric distortion,
hurting representation quality. We therefore choose $c=-1.0$ as the default
curvature. We set $\lambda_3=0.1$ and test ($\lambda_1$, $\lambda_2$) varying
between (0.1, 0.9). Figure~\ref{fig:sensitivity}(b) shows that \method{} reaches
its optimum when $\lambda_1=0.1$ and $\lambda_2=0.1$. Accuracy drops as these
weights increase, suggesting that excessive geometric alignment can overwhelm
the classification objective and hurt adaptation. More results appear in
Appendix~\ref{sec:sensitive analysis}.

\begin{wraptable}{r}{0.6\textwidth}
\vspace{-12pt}
\centering
\caption{The results of ablation studies on the Mutagenicity dataset. \textbf{Bold} results indicate the best performance.}
\vspace{-3pt}
\resizebox{0.6\textwidth}{!}{
\begin{tabular}{l|c|c|c|c|c|c}
\toprule
{\bf Methods} & M0$\rightarrow$M1 & M1$\rightarrow$M0 & M0$\rightarrow$M2 & M2$\rightarrow$M0 & M0$\rightarrow$M3 & M3$\rightarrow$M0 \\
\midrule
\method{} w/o FM & 76.2 & 73.6 & 69.6 & 74.0 & 65.3 & 66.2 \\
\method{} w/o RA & 76.9 & 74.4 & 70.7 & 73.9 & 66.1 & 67.1 \\
\method{} w/o AA & 76.4 & 74.0 & 70.2 & 73.5 & 65.7 & 66.0 \\
\method{} w/o PE & 75.9 & 72.9 & 69.0 & 72.5 & 64.6 & 64.8 \\
\midrule
\method{} & \textbf{78.3} & \textbf{76.3} & \textbf{71.3} & \textbf{76.1} & \textbf{67.7} & \textbf{69.2} \\
\bottomrule
\end{tabular}
}
\vspace{-0.4cm}
\label{tab:ablation_mutag_part}
\end{wraptable}
\textbf{Ablation Study.}
\label{sec:ablation} 
To examine the contribution of each component in \method{}, we conduct ablations
on four variants: 1) \method{} w/o FM, which removes the flow-based transport
mechanism; 2) \method{} w/o RA, which excludes the radial alignment loss; 3)
\method{} w/o AA, which excludes the angular alignment loss; and 4) \method{}
w/o PE, which disables the polar disentanglement strategy entirely.
Results are reported in Table~\ref{tab:ablation_mutag_part}, and we draw three observations: (1) \method{} outperforms \method{} w/o FM, confirming that
flow-based transport provides principled trajectory guidance to bridge domain
discrepancies. (2) The performance drops in \method{} w/o RA and \method{} w/o
AA validate the necessity of independent geometric constraints. (3) \method{}
w/o PE exhibits the sharpest degradation, underscoring that disentangling
structural and semantic components is critical for robust graph adaptation.
More results are reported in Appendix~\ref{sec:ablation study}.

\textbf{Visualization of Manifold Geometry.} 
We visualize polar geometry before and after alignment in
Figure~\ref{fig:manifold_vis}, using radius--angle plots and class-conditioned
histograms. Before alignment, source and target embeddings
exhibit radial mismatch: target samples concentrate at smaller radii, with
limited overlap. Same-class samples are poorly mixed across domains, with only
partial angular alignment. After alignment, radial profiles become closer,
source--target overlap increases near the origin, and angular histograms show
class-wise correspondence. Domains mix better within each class, without
collapsing inter-class angular margins, while different
classes remain directionally separated. These patterns support our
coordinate-wise design: radial regularization calibrates structural statistics,
whereas angular alignment preserves class-discriminative semantics during
transport.

\textbf{Disentanglement Analysis.}
To assess whether learned polar coordinates support an operational structure--semantics
separation, rather than global identifiability, we intervene on radial and angular factors. We scale $r$ while
fixing $u$, and interpolate $u$ across classes while fixing $r$, then evaluate
the perturbed representations with a structural probe and semantic classifier. As shown in
Figure~\ref{fig:Intervention}, radial scaling monotonically increases predicted
node count while class probability remains nearly unchanged. Conversely, angular
interpolation changes semantic probability but leaves the structural score
stable. These results suggest that $r$ mainly controls structure-related
variation, whereas $u$ carries class-discriminative semantics. More results
appear in Appendix~\ref{sec:disen_more}.

%% file: code/5_conclusion.tex
\section{Conclusion}

We proposed \method{}, a geometry-aware framework for graph domain adaptation under structural and semantic distribution shifts. \method{} uses polar endpoint regularization to calibrate structure-sensitive radial scales and align scale-normalized angular semantics with confidence-filtered pseudo-labels. It replaces adversarial domain confusion with topology-conditioned Riemannian flow matching, enabling stable transport between class- and structure-compatible source--target representations. Theoretical analysis explains the structural risk of unconditional domain confusion and supports the radial-angular design and flow-based transport objective. Experiments under node-density, edge-density, and feature shifts show strong and stable performance. Future work will explore adaptive curvature learning and extensions to heterogeneous and generative graph tasks.

%% file: code/6_appendix.tex
\appendix

\section{Riemannian Geometry and Stereographic Model}
\label{subsec:riemannian_geo}

\textbf{Problem Definition.}
We focus on the problem of GDA for graph classification. Let $\mathcal{D}_S\! =\! \{(\mathcal{G}_i^S, y_i^S)\}_{i=1}^{N_S}$ and $\mathcal{D}_T\! =\! \{\mathcal{G}_j^T\}_{j=1}^{N_T}$ denote the labeled source and unlabeled target domains. Both domains share a label space $\mathcal{Y}$ but under distribution shift $\mathbb{P}_S(\mathcal{D}_S) \neq \mathbb{P}_T(\mathcal{D}_T)$. Each graph is represented as $\mathcal{G} = (\mathcal{V}, \mathcal{E}, X)$, where $\mathcal{V}$ is the set of nodes, $\mathcal{E}$ is the set of edges, and $X \in \mathbb{R}^{|\mathcal{V}| \times F}$ denotes the node features. Our goal is to learn a geometry-aware encoder $f_\theta: \mathcal{G} \to \mathcal{M}_c$ and classifier $g_\phi: \mathcal{M}_c \to \mathcal{Y}$ to minimize classification error on $\mathcal{D}_T$, where $\mathcal{M}_c$ is a latent Riemannian manifold. 

Let $\mathcal{M}_c$ denote the $d$-dimensional Riemannian manifold with constant curvature $c$, defined as the interior of a ball $\mathcal{M}_c = \{ x \in \mathbb{R}^d \!:\! -c\|x\|^2 < 1 \}$. The Riemannian metric tensor $g_x^c$ is conformally equivalent to the Euclidean metric $g^E$, given by $g_x^c = (\lambda_x^c)^2 g^E$, where the conformal factor is $\lambda_x^c = 2 / (1 + c\|x\|^2)$. The curvature $c$ determines the geometry: $c < 0$ for the Hyperbolic space $\mathbb{H}^d$, $c > 0$ for the Spherical space $\mathbb{S}^d$, and $c \to 0$ for the Euclidean space $\mathbb{E}^d$.

\textbf{Exponential and Logarithmic Maps.}
Transformations between the manifold $\mathcal{M}_c$ and the tangent space $T_x\mathcal{M}_c$ (isomorphic to $\mathbb{R}^d$) are mediated by the Exponential map $\text{Exp}_x^c: T_x\mathcal{M}_c \to \mathcal{M}_c$ and the Logarithmic map $\text{Log}_x^c: \mathcal{M}_c \to T_x\mathcal{M}_c$. For the origin $0$, these maps are:
\begin{equation}
    \text{Exp}_0^c(v) = \tan_c(\|v\|) \frac{v}{\|v\|}, \quad \text{Log}_0^c(y) = \tan_c^{-1}(\|y\|) \frac{y}{\|y\|},\nonumber
\end{equation}
where $\tan_c(\cdot)$ and $\tan_c^{-1}(\cdot)$ are curvature-dependent trigonometric functions (\textit{e.g.}, $\tanh/\text{arctanh}$ for $c<0$).

\textbf{Geodesics and Distance.}
Geodesics generalize the concept of straight lines to Riemannian manifolds, representing the locally length-minimizing curves. The unique geodesic segment $\gamma: [0, 1] \to \mathcal{M}_c$ connecting a source point $x$ to a target point $y$ is parameterized as:
    $\gamma(t) = \text{Exp}_x^c(t \cdot \text{Log}_x^c(y)).\nonumber$
The Riemannian distance $d_{\mathcal{M}}(x, y)$ is the length of this geodesic, calculated as $d_{\mathcal{M}}(x, y) = \|\text{Log}_x^c(y)\|_x$. 

\textbf{Parallel Transport.}
Comparing tangent vectors at distinct points requires {Parallel Transport}. The transport of a vector $v \in T_0\mathcal{M}_c$ to $T_x\mathcal{M}_c$ along the geodesic connecting $0$ and $x$ is given by $P_{0 \to x}(v) = (\lambda_0^c / \lambda_x^c) v$. This operator ensures that the geometric properties of feature vectors (e.g., orientation) are preserved during transport.

\textbf{Riemannian Polar Coordinates.}
Formally, any point $x \in \mathcal{M}_c \setminus \{0\}$ can be uniquely represented via the polar decomposition of its tangent vector in $T_0\mathcal{M}_c$. We define the {radial coordinate} $r_x$ and {angular coordinate} $u_x$ as:
\begin{equation}
    r_x = \|\text{Log}_0^c(x)\|_2, \quad u_x = \frac{\text{Log}_0^c(x)}{\|\text{Log}_0^c(x)\|_2}.\nonumber
\end{equation}
Mathematically, $r_x$ corresponds to the Riemannian geodesic distance from the origin, \textit{i.e.}, $r_x = d_{\mathcal{M}}(0, x)$, while $u_x \in \mathbb{S}^{d-1}$ represents the direction in the tangent space. 

\textbf{Riemannian Graph Convolution.}
\label{subsec:rgcn}
Following the unified framework of Hyperbolic Graph Neural Networks (HGNN)~\citep{liu2019hyperbolic}, we generalize standard GCNs to a constant curvature Riemannian manifold $\mathcal{M}_c$. The sign of the curvature $c$ adaptively dictates the underlying geometry, unifying Hyperbolic space $\mathbb{H}$ ($c < 0$), Euclidean space $\mathbb{E}$ ($c = 0$), and Spherical Space $\mathbb{S}$ ($c > 0$) space within a single formalism. Fundamental operations, such as linear transformation and aggregation, are executed by mapping features to the tangent space $T_{\mathbf{0}}\mathcal{M}_c$, performing Euclidean computations, and retracting the results back to the manifold. Specifically, the linear transformation of a node feature $\mathbf{h}_j^{(l)} \in \mathcal{M}_c$ via a weight matrix $\mathbf{W}^{(l)}$ is defined by Möbius matrix-vector multiplication:
\begin{equation}
    \mathbf{m}_j^{(l)} = \mathbf{W}^{(l)} \otimes_c \mathbf{h}_j^{(l)} = \text{Exp}_{\mathbf{0}}^c \left( \mathbf{W}^{(l)} \cdot \text{Log}_{\mathbf{0}}^c(\mathbf{h}_j^{(l)}) \right).\nonumber
\end{equation}
Then, neighborhood aggregation is grounded in the Fréchet mean~\citep{karcher1977riemannian}, and approximated in the tangent space, resulting in the unified layer-wise update rule:
\begin{equation}
    \mathbf{h}_i^{(l+1)} \!=\! \text{Exp}_{\mathbf{0}}^c \Bigg( \sigma \Big( \sum_{j \in \mathcal{N}(i) \cup \{i\}} \alpha_{ij} \cdot \text{Log}_{\mathbf{0}}^c ( \mathbf{m}_j^{(l)} ) \Big) \Bigg).
\end{equation}
This architecture ensures that representation learning faithfully respects the underlying manifold geometry.

\setcounter{theorem}{0}
\section{Proof of Theorem 1}

\begin{theorem}[A Structural Failure Mode of Domain Confusion]
Consider a binary structural variable $A\in\{0,1\}$ with equal priors
$\Pr(A=0)=\Pr(A=1)=1/2$, label $Y=A$, and domain variable $D\in\{0,1\}$ with
$\Pr(D=A)=1-\eta$ for some $0\le\eta<1/2$. Assume a representation $Z$ depends
on the graph only through $A$ in this toy model. If an adversarial objective
enforces approximate domain invariance
\[
    {\rm TV}\big(P_{Z|D=0},P_{Z|D=1}\big)\le \delta,
\]
then the total variation between the label-conditionals of $Z$ is bounded by
\[
    {\rm TV}\big(P_{Z|Y=0},P_{Z|Y=1}\big)
    \le
    \frac{\delta}{1-2\eta}.
\]
Consequently, any classifier using $Z$ has error at least
\[
    \inf_h\Pr(h(Z)\neq Y)
    \ge
    \frac12\left(1-\frac{\delta}{1-2\eta}\right).
\]
\end{theorem}

\begin{proof}

Let $\mathcal Z$ denote the measurable representation space of $Z$, and let
$\mathcal B(\mathcal Z)$ denote the collection of measurable subsets of
$\mathcal Z$. For two probability measures $P$ and $Q$ on $\mathcal Z$, we use
the convention
\[
    {\rm TV}(P,Q)
    =
    \sup_{B\in\mathcal B(\mathcal Z)} |P(B)-Q(B)|.
\]
For $a,d\in\{0,1\}$, define
\[
    \mu_a := P_{Z\mid A=a},
    \qquad
    \nu_d := P_{Z\mid D=d}.
\]
Here, $\mu_a$ is the conditional distribution of the representation $Z$ given
the structural variable $A=a$, and $\nu_d$ is the conditional distribution of
$Z$ given the domain variable $D=d$. Since $Y=A$, we also have
\[
    P_{Z\mid Y=a} = P_{Z\mid A=a} = \mu_a .
\]

The assumption that $Z$ depends on the graph only through $A$ means that, in
this toy model, $Z$ is conditionally independent of $D$ given $A$. Equivalently,
\[
    P_{Z\mid A=a,D=d} = P_{Z\mid A=a}=\mu_a,
    \qquad a,d\in\{0,1\}.
\]
Therefore, each domain-conditional distribution $\nu_d$ can be written as a
mixture of the two structure-conditionals $\mu_0$ and $\mu_1$.

Assume equal class priors, i.e.,
\[
    \Pr(A=0)=\Pr(A=1)=\frac12 .
\]
Let
\[
    e_0 := \Pr(D=1\mid A=0),
    \qquad
    e_1 := \Pr(D=0\mid A=1),
\]
where $e_0$ and $e_1$ are the two possible mismatch rates between the structural
variable $A$ and the domain variable $D$. Since
\[
    \Pr(D=A)=1-\eta,
\]
we have
\[
    \Pr(D\neq A)=\eta.
\]
Under equal class priors,
\[
    \eta
    =
    \frac12 e_0+\frac12 e_1,
\]
and hence
\[
    e_0+e_1=2\eta .
\]

Now define the two mixture coefficients
\[
    \alpha_0 := \Pr(A=0\mid D=0),
    \qquad
    \alpha_1 := \Pr(A=0\mid D=1).
\]
By Bayes' rule,
\[
    \alpha_0
    =
    \frac{\Pr(D=0\mid A=0)\Pr(A=0)}{\Pr(D=0)}
    =
    \frac{1-e_0}{1-e_0+e_1},
\]
and
\[
    \alpha_1
    =
    \frac{\Pr(D=1\mid A=0)\Pr(A=0)}{\Pr(D=1)}
    =
    \frac{e_0}{e_0+1-e_1}.
\]
Therefore,
\[
    \nu_0
    =
    \alpha_0\mu_0+(1-\alpha_0)\mu_1,
\]
and
\[
    \nu_1
    =
    \alpha_1\mu_0+(1-\alpha_1)\mu_1 .
\]
Subtracting the two signed measures gives
\[
    \nu_0-\nu_1
    =
    (\alpha_0-\alpha_1)(\mu_0-\mu_1).
\]
Taking total variation on both sides yields
\[
    {\rm TV}(\nu_0,\nu_1)
    =
    |\alpha_0-\alpha_1|\,{\rm TV}(\mu_0,\mu_1).
\]

It remains to lower bound $|\alpha_0-\alpha_1|$. Since $\eta<1/2$, we have
$e_0+e_1<1$, so $\alpha_0>\alpha_1$. Thus
\[
\begin{aligned}
    \alpha_0-\alpha_1
    &=
    \frac{1-e_0}{1-e_0+e_1}
    -
    \frac{e_0}{1+e_0-e_1}  \\
    &=
    \frac{1-e_0-e_1}{(1-e_0+e_1)(1+e_0-e_1)} \\
    &=
    \frac{1-2\eta}{1-(e_0-e_1)^2}.
\end{aligned}
\]
Because $1-(e_0-e_1)^2\le 1$ and is positive, we obtain
\[
    \alpha_0-\alpha_1 \ge 1-2\eta .
\]
Combining this with the previous identity gives
\[
    {\rm TV}(\nu_0,\nu_1)
    \ge
    (1-2\eta)\,{\rm TV}(\mu_0,\mu_1).
\]
The adversarial domain-invariance condition assumes
\[
    {\rm TV}(P_{Z\mid D=0},P_{Z\mid D=1})
    =
    {\rm TV}(\nu_0,\nu_1)
    \le \delta .
\]
Therefore,
\[
    (1-2\eta)\,{\rm TV}(\mu_0,\mu_1)
    \le \delta .
\]
Since $\eta<1/2$, we have $1-2\eta>0$, and hence
\[
    {\rm TV}(\mu_0,\mu_1)
    \le
    \frac{\delta}{1-2\eta}.
\]
Using $Y=A$, this is equivalent to
\[
    {\rm TV}\big(P_{Z\mid Y=0},P_{Z\mid Y=1}\big)
    \le
    \frac{\delta}{1-2\eta}.
\]

Finally, consider any deterministic classifier $h:\mathcal Z\to\{0,1\}$.
Let
\[
    B_h := \{z\in\mathcal Z: h(z)=1\}
\]
be the measurable decision region where $h$ predicts class $1$. Under equal
class priors, the classification error of $h$ is
\[
\begin{aligned}
    \Pr(h(Z)\neq Y)
    &=
    \frac12 \Pr(h(Z)=1\mid Y=0)
    +
    \frac12 \Pr(h(Z)=0\mid Y=1) \\
    &=
    \frac12 \mu_0(B_h)
    +
    \frac12 \mu_1(B_h^c) \\
    &=
    \frac12\left[
        1-\big(\mu_1(B_h)-\mu_0(B_h)\big)
    \right].
\end{aligned}
\]
By the definition of total variation,
\[
    \mu_1(B_h)-\mu_0(B_h)
    \le
    {\rm TV}(\mu_0,\mu_1).
\]
Thus, for every classifier $h$,
\[
    \Pr(h(Z)\neq Y)
    \ge
    \frac12\left(1-{\rm TV}(\mu_0,\mu_1)\right).
\]
Taking the infimum over all classifiers $h$ and applying the bound above gives
\[
    \inf_h \Pr(h(Z)\neq Y)
    \ge
    \frac12
    \left(
        1-\frac{\delta}{1-2\eta}
    \right).
\]
This completes the proof.
\end{proof}

\section{Proof of Theorem 2}
\begin{theorem}[First-Order Orthogonality of Polar Endpoint Losses]
Let $\mathcal M_c$ be a $d$-dimensional constant-curvature manifold and let
$o\in\mathcal M_c$ be the origin. Work on a geodesic normal ball $B_o(\rho)$,
with $\rho<\pi/\sqrt c$ when $c>0$, and exclude $o$. For
$\mathbf{z}=(r,u)\in B_o(\rho)\setminus\{o\}$, the polar metric is Eq.~\eqref{eq:polar_metric}.
For a batch $\mathbf z=(\mathbf{z}_i)_{i=1}^B$, equip $\mathcal M_c^B$ with the product
metric. Let $\mathcal L_{\rm rad}(\mathbf z)=\ell_r(r_{1:B})$ and
$\mathcal L_{\rm ang}(\mathbf z)=\ell_u(u_{1:B})$ be almost-everywhere
differentiable radial-only and angular-only endpoint losses. Then, wherever the
gradients exist,
\[
    \left\langle
    \nabla_g\mathcal L_{\rm rad},
    \nabla_g\mathcal L_{\rm ang}
    \right\rangle_g=0.
\]
If a differentiable radial reliability function weights the angular loss $\mathcal L_{\rm wang}=\sum_i\alpha(r_i)\ell_i(u_i)$, then
\[
\left|
\left\langle
\nabla_g\mathcal L_{\rm rad},
\nabla_g\mathcal L_{\rm wang}
\right\rangle_g
\right|
\le
\sum_{i=1}^{B}
|\partial_{r_i}\ell_r|\,|\alpha'(r_i)|\,|\ell_i(u_i)|.
\]
For the choice $\alpha(r)=e^{-r}$, the coupling term is bounded by
$\sum_i |\partial_{r_i}\ell_r|e^{-r_i}|\ell_i(u_i)|$.
\end{theorem}

\begin{proof}
We prove the statement pointwise at any batch
$\mathbf z=(z_1,\ldots,z_B)\in \mathcal M_c^B$ where all involved derivatives
exist. For each sample $z_i\in B_o(\rho)\setminus\{o\}$, write its geodesic
polar coordinates as
\[
    z_i=(r_i,u_i),
\]
where $r_i=d_{\mathcal M_c}(o,z_i)>0$ is the geodesic radius and
$u_i\in \mathbb S^{d-1}$ is the unit angular direction. The exclusion of
$o$ ensures that $u_i$ is well-defined. When $c>0$, the condition
$\rho<\pi/\sqrt c$ keeps the normal ball away from the cut locus, so the
polar coordinates are smooth on the considered region.

By the polar metric in Eq.~\eqref{eq:polar_metric}, for a tangent vector
\[
    \xi_i=(a_i,b_i)\in T_{z_i}\mathcal M_c,
\]
where $a_i\in\mathbb R$ is the radial component and
$b_i\in T_{u_i}\mathbb S^{d-1}$ is the angular component, the metric takes the
form
\[
    g_{z_i}(\xi_i,\xi_i')
    =
    a_i a_i'
    +
    S_c(r_i)^2
    \langle b_i,b_i'\rangle_{\mathbb S},
\]
where $\xi_i'=(a_i',b_i')$, $\langle\cdot,\cdot\rangle_{\mathbb S}$ is the
canonical metric on the unit sphere, and
\[
S_c(r)=
\begin{cases}
\sin(\sqrt c\,r)/\sqrt c, & c>0,\\
r, & c=0,\\
\sinh(\sqrt{-c}\,r)/\sqrt{-c}, & c<0.
\end{cases}
\]
Since $r_i>0$ and the chart stays inside the normal ball, $S_c(r_i)>0$.

The batch space $\mathcal M_c^B$ is equipped with the product metric. Hence,
for two batch tangent vectors
\[
    \xi=(\xi_1,\ldots,\xi_B),
    \qquad
    \xi'=(\xi_1',\ldots,\xi_B'),
\]
with $\xi_i=(a_i,b_i)$ and $\xi_i'=(a_i',b_i')$, we have
\[
    g^{(B)}_{\mathbf z}(\xi,\xi')
    =
    \sum_{i=1}^B
    \left[
        a_i a_i'
        +
        S_c(r_i)^2
        \langle b_i,b_i'\rangle_{\mathbb S}
    \right].
\]

Now let $F:\mathcal M_c^B\to\mathbb R$ be any differentiable scalar function.
In polar coordinates, its differential along a batch tangent vector
$\xi=(a_i,b_i)_{i=1}^B$ is
\[
    dF[\xi]
    =
    \sum_{i=1}^B
    \left[
        \partial_{r_i}F \cdot a_i
        +
        \left\langle
        \nabla_{\mathbb S,i}F,
        b_i
        \right\rangle_{\mathbb S}
    \right],
\]
where $\partial_{r_i}F$ is the partial derivative of $F$ with respect to the
radius $r_i$, and $\nabla_{\mathbb S,i}F\in T_{u_i}\mathbb S^{d-1}$ is the
spherical gradient of $F$ with respect to the angular variable $u_i$.

By definition, the Riemannian gradient $\nabla_g F$ is the unique tangent
vector satisfying
\[
    g^{(B)}_{\mathbf z}(\nabla_g F,\xi)=dF[\xi]
    \qquad
    \text{for all } \xi\in T_{\mathbf z}\mathcal M_c^B .
\]
Comparing the radial and angular components gives
\[
    (\nabla_g F)_i
    =
    \left(
        \partial_{r_i}F,\,
        S_c(r_i)^{-2}\nabla_{\mathbb S,i}F
    \right).
\]
Therefore, for two differentiable scalar functions $F$ and $H$,
\[
\begin{aligned}
    \left\langle \nabla_g F,\nabla_g H\right\rangle_g
    &=
    \sum_{i=1}^B
    \left[
        \partial_{r_i}F\,\partial_{r_i}H
        +
        S_c(r_i)^{-2}
        \left\langle
        \nabla_{\mathbb S,i}F,
        \nabla_{\mathbb S,i}H
        \right\rangle_{\mathbb S}
    \right].
\end{aligned}
\]

We first apply this identity to
\[
    \mathcal L_{\rm rad}(\mathbf z)=\ell_r(r_{1:B}),
    \qquad
    \mathcal L_{\rm ang}(\mathbf z)=\ell_u(u_{1:B}),
\]
where $r_{1:B}=(r_1,\ldots,r_B)$ and
$u_{1:B}=(u_1,\ldots,u_B)$. Since $\mathcal L_{\rm rad}$ depends only on the
radial variables, it has no angular derivative:
\[
    \nabla_{\mathbb S,i}\mathcal L_{\rm rad}=0,
    \qquad i=1,\ldots,B.
\]
Since $\mathcal L_{\rm ang}$ depends only on the angular variables, it has no
radial derivative:
\[
    \partial_{r_i}\mathcal L_{\rm ang}=0,
    \qquad i=1,\ldots,B.
\]
Substituting these two identities into the inner-product formula yields
\[
\begin{aligned}
    \left\langle
    \nabla_g\mathcal L_{\rm rad},
    \nabla_g\mathcal L_{\rm ang}
    \right\rangle_g
    &=
    \sum_{i=1}^B
    \left[
        \partial_{r_i}\mathcal L_{\rm rad}\cdot 0
        +
        S_c(r_i)^{-2}
        \left\langle
        0,
        \nabla_{\mathbb S,i}\mathcal L_{\rm ang}
        \right\rangle_{\mathbb S}
    \right]  \\
    &=0.
\end{aligned}
\]
This proves the first-order orthogonality between purely radial and purely
angular endpoint losses.

We next consider the radially weighted angular loss
\[
    \mathcal L_{\rm wang}(\mathbf z)
    =
    \sum_{i=1}^B \alpha(r_i)\ell_i(u_i),
\]
where $\alpha:\mathbb R_+\to\mathbb R$ is differentiable and
$\ell_i:\mathbb S^{d-1}\to\mathbb R$ is the angular loss for sample $i$.
For this loss, the radial derivative is
\[
    \partial_{r_i}\mathcal L_{\rm wang}
    =
    \alpha'(r_i)\ell_i(u_i),
\]
where $\alpha'(r_i)$ denotes the derivative of $\alpha$ evaluated at $r_i$.
Its angular derivative is
\[
    \nabla_{\mathbb S,i}\mathcal L_{\rm wang}
    =
    \alpha(r_i)\nabla_{\mathbb S}\ell_i(u_i).
\]
Using the same inner-product formula and the fact that
$\nabla_{\mathbb S,i}\mathcal L_{\rm rad}=0$, we obtain
\[
\begin{aligned}
    \left\langle
    \nabla_g\mathcal L_{\rm rad},
    \nabla_g\mathcal L_{\rm wang}
    \right\rangle_g
    &=
    \sum_{i=1}^B
    \partial_{r_i}\mathcal L_{\rm rad}
    \cdot
    \partial_{r_i}\mathcal L_{\rm wang} \\
    &=
    \sum_{i=1}^B
    \partial_{r_i}\ell_r
    \cdot
    \alpha'(r_i)
    \ell_i(u_i),
\end{aligned}
\]
where $\partial_{r_i}\ell_r$ is shorthand for
$\partial \ell_r(r_{1:B})/\partial r_i$.

Taking absolute values and applying the triangle inequality gives
\[
\begin{aligned}
\left|
\left\langle
\nabla_g\mathcal L_{\rm rad},
\nabla_g\mathcal L_{\rm wang}
\right\rangle_g
\right|
&=
\left|
\sum_{i=1}^B
\partial_{r_i}\ell_r
\cdot
\alpha'(r_i)
\ell_i(u_i)
\right|  \\
&\le
\sum_{i=1}^B
|\partial_{r_i}\ell_r|\,
|\alpha'(r_i)|\,
|\ell_i(u_i)|.
\end{aligned}
\]
This proves the claimed bound.

Finally, for the reliability function used in the main text,
\[
    \alpha(r)=e^{-r},
\]
we have
\[
    \alpha'(r)=-e^{-r},
    \qquad
    |\alpha'(r)|=e^{-r}.
\]
Substituting this into the previous inequality gives
\[
\left|
\left\langle
\nabla_g\mathcal L_{\rm rad},
\nabla_g\mathcal L_{\rm wang}
\right\rangle_g
\right|
\le
\sum_{i=1}^{B}
|\partial_{r_i}\ell_r|\,
e^{-r_i}\,
|\ell_i(u_i)|.
\]
This completes the proof.
\end{proof}

\section{Proof of Theorem 3}
\begin{theorem}[Radial Structural Proxy Bound]
Let $P_S$ and $P_T$ be source and target graph distributions. For a graph $G$,
let $r(G)=\|\operatorname{Log}_o^c(f_\theta(G))\|_{g_o}$ and let
$a(G)\in\mathbb R^m$ denote topology-related graph statistics. Denote by
$r_\#P$ and $a_\#P$ the pushforward distributions induced by $r(G)$ and $a(G)$.
Suppose there exists an $L_a$-Lipschitz map $h:\mathbb R_+\to\mathbb R^m$ such
that, for $D\in\{S,T\}$,
    $\mathbb E_{G\sim P_D}[\|a(G)-h(r(G))\|]\le\epsilon_D$.
Then
\[
    W_1(a_\#P_S,a_\#P_T)
    \le
    L_a W_1(r_\#P_S,r_\#P_T)+\epsilon_S+\epsilon_T.
\]
For equal-size empirical distributions,
    $W_1(\widehat P_S^r,\widehat P_T^r)
    =
    \frac1B\sum_{k=1}^B |\mathcal R_S[k]-\mathcal R_T[k]|$,
which is Eq.~\eqref{eq:radial_loss}.
\end{theorem}

\begin{proof}
Let $\mathcal G$ denote the graph space. For a graph distribution $P$ on
$\mathcal G$, the pushforward distribution induced by the radial map
$r:\mathcal G\to\mathbb R_+$ is denoted by $r_\#P$. That is, for any measurable
set $B\subseteq\mathbb R_+$,
\[
    r_\#P(B)
    =
    P\big(\{G\in\mathcal G:r(G)\in B\}\big).
\]
Similarly, the pushforward distribution induced by the topology-statistic map
$a:\mathcal G\to\mathbb R^m$ is denoted by $a_\#P$, meaning that for any
measurable set $C\subseteq\mathbb R^m$,
\[
    a_\#P(C)
    =
    P\big(\{G\in\mathcal G:a(G)\in C\}\big).
\]

For $D\in\{S,T\}$, let
\[
    G_D\sim P_D,
    \qquad
    R_D := r(G_D),
    \qquad
    A_D := a(G_D),
    \qquad
    H_D := h(r(G_D))=h(R_D).
\]
Here, $G_D$ is a graph sampled from domain $D$, $R_D$ is its radial coordinate,
$A_D$ is its topology-statistic vector, and $H_D$ is the topology proxy predicted
from the radius by the map $h$. By definition,
\[
    R_D\sim r_\#P_D,
    \qquad
    A_D\sim a_\#P_D,
    \qquad
    H_D\sim h_\#(r_\#P_D),
\]
where $h_\#(r_\#P_D)$ denotes the pushforward of the radial distribution through
the map $h$.

Recall that for two probability measures $\mu$ and $\nu$ on a metric space
$(\mathcal X,d)$, the Wasserstein-1 distance is
\[
    W_1(\mu,\nu)
    =
    \inf_{\gamma\in\Pi(\mu,\nu)}
    \mathbb E_{(X,Y)\sim\gamma}\big[d(X,Y)\big],
\]
where $\Pi(\mu,\nu)$ denotes the set of all couplings whose first marginal is
$\mu$ and whose second marginal is $\nu$.

We first decompose the discrepancy between the source and target structural
statistics by inserting the radius-induced proxy distributions:
\[
\begin{aligned}
    W_1(a_\#P_S,a_\#P_T)
    &\le
    W_1(a_\#P_S,h_\#(r_\#P_S))  \\
    &\quad
    + W_1(h_\#(r_\#P_S),h_\#(r_\#P_T)) \\
    &\quad
    + W_1(h_\#(r_\#P_T),a_\#P_T).
\end{aligned}
\]
This follows from the triangle inequality of $W_1$.

We now bound the first term. Consider the coupling between
$a_\#P_S$ and $h_\#(r_\#P_S)$ induced by drawing the same source graph
$G_S\sim P_S$ and pairing
\[
    A_S=a(G_S),
    \qquad
    H_S=h(r(G_S)).
\]
This is a valid coupling because its two marginals are exactly
$a_\#P_S$ and $h_\#(r_\#P_S)$. Therefore, by the definition of $W_1$,
\[
\begin{aligned}
    W_1(a_\#P_S,h_\#(r_\#P_S))
    &\le
    \mathbb E_{G_S\sim P_S}
    \big[
        \|a(G_S)-h(r(G_S))\|
    \big]  \\
    &\le
    \epsilon_S .
\end{aligned}
\]
The last inequality is exactly the assumed source-domain proxy error.

The same argument applies to the target domain. Using the coupling induced by
the same target graph $G_T\sim P_T$, we obtain
\[
\begin{aligned}
    W_1(h_\#(r_\#P_T),a_\#P_T)
    &\le
    \mathbb E_{G_T\sim P_T}
    \big[
        \|h(r(G_T))-a(G_T)\|
    \big]  \\
    &\le
    \epsilon_T .
\end{aligned}
\]

It remains to bound the middle term
\[
    W_1(h_\#(r_\#P_S),h_\#(r_\#P_T)).
\]
Let $\gamma_r\in\Pi(r_\#P_S,r_\#P_T)$ be any coupling between the source and
target radial distributions. If $(R_S,R_T)\sim\gamma_r$, then
\[
    (h(R_S),h(R_T))
\]
is a valid coupling between $h_\#(r_\#P_S)$ and $h_\#(r_\#P_T)$. Since
$h:\mathbb R_+\to\mathbb R^m$ is $L_a$-Lipschitz, we have
\[
    \|h(R_S)-h(R_T)\|
    \le
    L_a |R_S-R_T|.
\]
Thus,
\[
\begin{aligned}
    W_1(h_\#(r_\#P_S),h_\#(r_\#P_T))
    &\le
    \mathbb E_{(R_S,R_T)\sim\gamma_r}
    \big[
        \|h(R_S)-h(R_T)\|
    \big]  \\
    &\le
    L_a
    \mathbb E_{(R_S,R_T)\sim\gamma_r}
    \big[
        |R_S-R_T|
    \big].
\end{aligned}
\]
Taking the infimum over all couplings
$\gamma_r\in\Pi(r_\#P_S,r_\#P_T)$ gives
\[
    W_1(h_\#(r_\#P_S),h_\#(r_\#P_T))
    \le
    L_a W_1(r_\#P_S,r_\#P_T).
\]

Combining the three bounds yields
\[
\begin{aligned}
    W_1(a_\#P_S,a_\#P_T)
    &\le
    \epsilon_S
    +
    L_a W_1(r_\#P_S,r_\#P_T)
    +
    \epsilon_T  \\
    &=
    L_a W_1(r_\#P_S,r_\#P_T)
    +
    \epsilon_S
    +
    \epsilon_T .
\end{aligned}
\]
This proves the population-level radial structural proxy bound.

We next prove the empirical formula. Suppose the source and target minibatches
have the same size $B$. Let their empirical radial distributions be
\[
    \widehat P_S^r
    =
    \frac1B\sum_{i=1}^B \delta_{r_i^S},
    \qquad
    \widehat P_T^r
    =
    \frac1B\sum_{i=1}^B \delta_{r_i^T},
\]
where $r_i^S=r(G_i^S)$ and $r_i^T=r(G_i^T)$ are source and target radial
coordinates, and $\delta_x$ denotes the Dirac measure at point $x$.

Let
\[
    \mathcal R_S[1]\le \cdots \le \mathcal R_S[B],
    \qquad
    \mathcal R_T[1]\le \cdots \le \mathcal R_T[B]
\]
be the sorted source and target radii. In one dimension with cost
$|r-r'|$, the optimal transport coupling is the monotone coupling, equivalently
the coupling induced by matching equal quantiles. Therefore,
\[
    W_1(\widehat P_S^r,\widehat P_T^r)
    =
    \int_0^1
    \left|
        F_S^{-1}(q)-F_T^{-1}(q)
    \right|
    dq,
\]
where $F_S^{-1}$ and $F_T^{-1}$ are the quantile functions of
$\widehat P_S^r$ and $\widehat P_T^r$. For equal-size empirical distributions,
these quantile functions are constant on each interval
$\big((k-1)/B,k/B\big]$:
\[
    F_S^{-1}(q)=\mathcal R_S[k],
    \qquad
    F_T^{-1}(q)=\mathcal R_T[k],
    \qquad
    q\in\left(\frac{k-1}{B},\frac{k}{B}\right].
\]
Hence,
\[
\begin{aligned}
    W_1(\widehat P_S^r,\widehat P_T^r)
    &=
    \sum_{k=1}^B
    \int_{(k-1)/B}^{k/B}
    \left|
        \mathcal R_S[k]-\mathcal R_T[k]
    \right|
    dq  \\
    &=
    \frac1B
    \sum_{k=1}^B
    \left|
        \mathcal R_S[k]-\mathcal R_T[k]
    \right|.
\end{aligned}
\]
This is exactly the radial loss in Eq.~\eqref{eq:radial_loss}. The proof is
complete.
\end{proof}

\section{Proof of Theorem 4}

\begin{theorem}[Angular Margin Preservation]
Let $\{\hat w_c\}_{c=1}^K\subset\mathbb S^{d-1}$ be normalized class prototypes,
and define
    $h(u)=\arg\max_{c\in[K]} \hat w_c^\top u$.
For a labeled sample $(\mathbf{z},y)$, let $u=u(\mathbf{z})$ and let
$\tilde u$ be the angular coordinate after an endpoint transport map. Define
the angular margin
\[
    m_y(u)
    =
    \hat w_y^\top u
    -
    \max_{c\ne y}\hat w_c^\top u .
\]
If 
    $m_y(u)\ge \mu$
    and
    $d_{\mathbb S}(u,\tilde u)<\frac{\mu}{2}$ for some $\mu>0$, then
    $h(\tilde u)=y$.
For a labeled distribution $P$,
\[
    R_{\rm post}(h)
    \le
    R_{\rm pre}(h)
    +
    \rho_\mu
    +
    \frac{2}{\mu}\Delta_{\rm ang},
\]
where $R_{\rm pre}(h)=\Pr_{P}[h(u)\ne y]$,
    $R_{\rm post}(h)=\Pr_{P}[h(\tilde u)\ne y],$
    $\rho_\mu=\Pr_{P}\!\left[0\le m_y(u)<\mu\right]$, and
    $\Delta_{\rm ang}=\mathbb E_{P}\!\left[d_{\mathbb S}(u,\tilde u)\right]$.
\end{theorem}

\begin{proof}
We first clarify the notation. Let $[K]=\{1,\ldots,K\}$ be the set of class
indices. For each class $c\in[K]$, $\hat w_c\in\mathbb S^{d-1}$ is a normalized
prototype, so
\[
    \|\hat w_c\|_2=1.
\]
For a representation $\mathbf z$, let $u=u(\mathbf z)\in\mathbb S^{d-1}$ denote
its angular coordinate. Let $\tilde u\in\mathbb S^{d-1}$ denote the angular
coordinate after the endpoint transport map. The classifier is
\[
    h(u)=\arg\max_{c\in[K]}\hat w_c^\top u .
\]
We assume a fixed deterministic tie-breaking rule for $\arg\max$, so that
$h$ is well-defined.

For a labeled sample with label $y\in[K]$, the angular margin is
\[
    m_y(u)
    =
    \hat w_y^\top u
    -
    \max_{c\ne y}\hat w_c^\top u .
\]
Equivalently, $m_y(u)$ measures how much larger the logit of the true class
$y$ is than the largest competing class logit. In particular, if
$m_y(u)>0$, then $y$ is the unique maximizer and hence $h(u)=y$.

We first prove the pointwise margin-preservation claim. For any competing class
$c\ne y$, define the pairwise margin
\[
    m_{y,c}(u)
    =
    \hat w_y^\top u-\hat w_c^\top u .
\]
By definition of $m_y(u)$,
\[
    m_{y,c}(u)\ge m_y(u),
    \qquad \forall c\ne y.
\]
Now compare the same pairwise margin after transport:
\[
\begin{aligned}
    m_{y,c}(\tilde u)
    &=
    \hat w_y^\top \tilde u-\hat w_c^\top \tilde u  \\
    &=
    \hat w_y^\top u-\hat w_c^\top u
    +
    \hat w_y^\top(\tilde u-u)
    -
    \hat w_c^\top(\tilde u-u)  \\
    &=
    m_{y,c}(u)
    +
    \hat w_y^\top(\tilde u-u)
    -
    \hat w_c^\top(\tilde u-u).
\end{aligned}
\]
Using Cauchy's inequality and $\|\hat w_y\|_2=\|\hat w_c\|_2=1$, we have
\[
    \left|\hat w_y^\top(\tilde u-u)\right|
    \le
    \|\tilde u-u\|_2,
\]
and
\[
    \left|\hat w_c^\top(\tilde u-u)\right|
    \le
    \|\tilde u-u\|_2.
\]
Therefore,
\[
\begin{aligned}
    m_{y,c}(\tilde u)
    &\ge
    m_{y,c}(u)-2\|\tilde u-u\|_2 .
\end{aligned}
\]
Since $u,\tilde u\in\mathbb S^{d-1}$, their Euclidean chordal distance is
bounded by their spherical geodesic distance:
\[
    \|\tilde u-u\|_2
    =
    2\sin\left(\frac{d_{\mathbb S}(u,\tilde u)}{2}\right)
    \le
    d_{\mathbb S}(u,\tilde u).
\]
Thus,
\[
    m_{y,c}(\tilde u)
    \ge
    m_{y,c}(u)-2d_{\mathbb S}(u,\tilde u).
\]
Using $m_{y,c}(u)\ge m_y(u)$ gives
\[
    m_{y,c}(\tilde u)
    \ge
    m_y(u)-2d_{\mathbb S}(u,\tilde u).
\]
If
\[
    m_y(u)\ge \mu
    \qquad\text{and}\qquad
    d_{\mathbb S}(u,\tilde u)<\frac{\mu}{2},
\]
then for every $c\ne y$,
\[
    m_{y,c}(\tilde u)
    >
    \mu-2\cdot\frac{\mu}{2}
    =
    0.
\]
Hence,
\[
    \hat w_y^\top\tilde u>\hat w_c^\top\tilde u,
    \qquad \forall c\ne y.
\]
Therefore, $y$ remains the unique maximizer after transport, and
\[
    h(\tilde u)=y.
\]
This proves the pointwise statement.

We now prove the distribution-level risk bound. Let $(\mathbf z,y)\sim P$ be a
random labeled sample. Define
\[
    U:=u(\mathbf z),
    \qquad
    \tilde U:=\tilde u,
\]
where $U$ is the pre-transport angular coordinate and $\tilde U$ is the
post-transport angular coordinate. Define the pre-transport error event
\[
    E_{\rm pre}
    :=
    \{h(U)\ne y\},
\]
and the post-transport error event
\[
    E_{\rm post}
    :=
    \{h(\tilde U)\ne y\}.
\]
By definition,
\[
    R_{\rm pre}(h)=\Pr_P(E_{\rm pre}),
    \qquad
    R_{\rm post}(h)=\Pr_P(E_{\rm post}).
\]

Next define the small-margin event
\[
    E_{\rm mar}
    :=
    \{0\le m_y(U)<\mu\},
\]
so that
\[
    \rho_\mu=\Pr_P(E_{\rm mar}).
\]
Also define the large-angular-drift event
\[
    E_{\rm drift}
    :=
    \left\{
        d_{\mathbb S}(U,\tilde U)\ge \frac{\mu}{2}
    \right\}.
\]

We claim that
\[
    E_{\rm post}
    \subseteq
    E_{\rm pre}\cup E_{\rm mar}\cup E_{\rm drift}.
\]
To prove this inclusion, consider any sample outside the right-hand side. Then
\[
    h(U)=y,
\]
so the true class is predicted correctly before transport. This implies
\[
    m_y(U)\ge 0,
\]
because the true class logit is at least as large as every competing class
logit. Since the sample is also outside $E_{\rm mar}$, it cannot satisfy
$0\le m_y(U)<\mu$. Therefore,
\[
    m_y(U)\ge \mu.
\]
Moreover, being outside $E_{\rm drift}$ gives
\[
    d_{\mathbb S}(U,\tilde U)<\frac{\mu}{2}.
\]
By the pointwise result proved above, these two inequalities imply
\[
    h(\tilde U)=y.
\]
Thus the sample is not in $E_{\rm post}$. This proves the event inclusion.

Taking probabilities and applying the union bound gives
\[
\begin{aligned}
    R_{\rm post}(h)
    &=
    \Pr_P(E_{\rm post}) \\
    &\le
    \Pr_P(E_{\rm pre})
    +
    \Pr_P(E_{\rm mar})
    +
    \Pr_P(E_{\rm drift}) \\
    &=
    R_{\rm pre}(h)
    +
    \rho_\mu
    +
    \Pr_P\left[
        d_{\mathbb S}(U,\tilde U)\ge \frac{\mu}{2}
    \right].
\end{aligned}
\]
The random variable
\[
    X:=d_{\mathbb S}(U,\tilde U)
\]
is nonnegative. By Markov's inequality,
\[
    \Pr_P\left[
        X\ge \frac{\mu}{2}
    \right]
    \le
    \frac{\mathbb E_P[X]}{\mu/2}
    =
    \frac{2}{\mu}\mathbb E_P[d_{\mathbb S}(U,\tilde U)].
\]
By definition,
\[
    \Delta_{\rm ang}
    =
    \mathbb E_P[d_{\mathbb S}(U,\tilde U)].
\]
Therefore,
\[
    \Pr_P\left[
        d_{\mathbb S}(U,\tilde U)\ge \frac{\mu}{2}
    \right]
    \le
    \frac{2}{\mu}\Delta_{\rm ang}.
\]
Substituting this into the previous inequality yields
\[
    R_{\rm post}(h)
    \le
    R_{\rm pre}(h)
    +
    \rho_\mu
    +
    \frac{2}{\mu}\Delta_{\rm ang}.
\]
This completes the proof.
\end{proof}

\section{Proof of Theorem 5}

\begin{theorem}[Polar Flow Endpoint Stability]
Fix a coupling $\pi$ and assume all analytic geodesics and learned flow
trajectories stay inside a compact geodesic normal ball $B_o(\rho)$. Let
$\gamma_t(z_S,z_T)$ be the analytic geodesic and let $\varphi_t^\psi$ be the
trajectory generated by the learned polar field in Eq.~\eqref{eq:polar_vector_field}
with $\varphi_0^\psi=z_S$. Assume the reconstructed manifold vector field is
$L_v$-Lipschitz in $z$ on $B_o(\rho)$. Then there exists a constant
$C_{\rm geo}$ depending only on the ball and the metric such that
\[
    \mathbb E_{(z_S,z_T)\sim\pi}
    \left[d_{\mathcal M_c}(\varphi_1^\psi,z_T)\right]
    \le
    C_{\rm geo}e^{L_v}\sqrt{\mathcal L_{\rm FM}}.
\]
\end{theorem}

\begin{proof}
Let
\[
    K:=\overline{B_o(\rho)}\subset \mathcal M_c
\]
denote the compact geodesic normal ball in which all analytic geodesics and
learned flow trajectories stay. We write
\[
    d(\cdot,\cdot):=d_{\mathcal M_c}(\cdot,\cdot)
\]
for the geodesic distance on $\mathcal M_c$. For a point
$z\in\mathcal M_c$, let $g_z$ be the Riemannian metric at $z$, and for a tangent
vector $\xi\in T_z\mathcal M_c$, define
\[
    \|\xi\|_{g_z}:=\sqrt{g_z(\xi,\xi)}.
\]

For a coupled source--target pair $(z_S,z_T)\sim\pi$, let
\[
    \gamma_t=\gamma_t(z_S,z_T),\qquad t\in[0,1],
\]
be the analytic constant-speed geodesic from $z_S$ to $z_T$. Thus,
\[
    \gamma_0=z_S,\qquad \gamma_1=z_T,
\]
and its velocity is denoted by
\[
    \dot\gamma_t\in T_{\gamma_t}\mathcal M_c.
\]

Let $V_\psi(z,t)\in T_z\mathcal M_c$ be the reconstructed manifold vector field
induced by the learned polar field in Eq.~\eqref{eq:polar_vector_field}. In
polar coordinates, if $z=(r,u)$, this field has the form
\[
    V_\psi(z,t)
    =
    a_\psi(z,t)\partial_r+b_\psi(z,t),
\]
where $a_\psi(z,t)$ is the radial velocity component and
$b_\psi(z,t)\in T_u\mathbb S^{d-1}$ is the angular velocity component. The
learned trajectory $\varphi_t^\psi$ is the solution of the initial value problem
\[
    \frac{d}{dt}\varphi_t^\psi
    =
    V_\psi(\varphi_t^\psi,t),
    \qquad
    \varphi_0^\psi=z_S.
\]

For each pair $(z_S,z_T)$ and time $t$, define the pointwise velocity-regression
error
\[
    \varepsilon(t;z_S,z_T)
    :=
    \left\|
        V_\psi(\gamma_t,t)-\dot\gamma_t
    \right\|_{g_{\gamma_t}} .
\]
This quantity measures, at the analytic geodesic point $\gamma_t$, the
Riemannian norm of the difference between the learned vector field and the true
geodesic velocity.

If $\gamma_t=(r_t,u_t)$ in polar coordinates and
\[
    \dot\gamma_t
    =
    \dot r_t\partial_r+\dot u_t,
    \qquad
    \dot u_t\in T_{u_t}\mathbb S^{d-1},
\]
then the polar metric
\[
    ds^2=dr^2+S_c(r)^2d\Omega^2
\]
gives
\[
\begin{aligned}
    \varepsilon(t;z_S,z_T)^2
    &=
    \left|
        a_\psi(\gamma_t,t)-\dot r_t
    \right|^2
    +
    S_c(r_t)^2
    \left\|
        b_\psi(\gamma_t,t)-\dot u_t
    \right\|_{\mathbb S}^2 .
\end{aligned}
\]
Therefore, the polar flow matching loss can be written as
\[
    \mathcal L_{\rm FM}
    =
    \mathbb E_{(z_S,z_T)\sim\pi,\;t\sim U[0,1]}
    \left[
        \varepsilon(t;z_S,z_T)^2
    \right],
\]
where $U[0,1]$ denotes the uniform distribution over $[0,1]$.

We now compare the learned trajectory $\varphi_t^\psi$ with the analytic
geodesic $\gamma_t$. Define
\[
    e(t):=d(\varphi_t^\psi,\gamma_t).
\]
Since both curves start from the same source point,
\[
    e(0)=d(z_S,z_S)=0.
\]
Also, since $\gamma_1=z_T$,
\[
    e(1)=d(\varphi_1^\psi,z_T).
\]

Let $\mathsf P_{\gamma_t\to\varphi_t^\psi}$ denote parallel transport from
$T_{\gamma_t}\mathcal M_c$ to $T_{\varphi_t^\psi}\mathcal M_c$ along the
minimizing geodesic inside the normal ball. Parallel transport is an isometry,
so for any $\xi\in T_{\gamma_t}\mathcal M_c$,
\[
    \left\|
        \mathsf P_{\gamma_t\to\varphi_t^\psi}\xi
    \right\|_{g_{\varphi_t^\psi}}
    =
    \|\xi\|_{g_{\gamma_t}}.
\]

By the first variation inequality for the Riemannian distance, the upper Dini
derivative of $e(t)$ satisfies
\[
\begin{aligned}
    D^+e(t)
    &\le
    \left\|
        V_\psi(\varphi_t^\psi,t)
        -
        \mathsf P_{\gamma_t\to\varphi_t^\psi}\dot\gamma_t
    \right\|_{g_{\varphi_t^\psi}} ,
\end{aligned}
\]
where
\[
    D^+e(t)
    :=
    \limsup_{h\downarrow 0}
    \frac{e(t+h)-e(t)}{h}.
\]
Using the triangle inequality, we obtain
\[
\begin{aligned}
    D^+e(t)
    &\le
    \left\|
        V_\psi(\varphi_t^\psi,t)
        -
        \mathsf P_{\gamma_t\to\varphi_t^\psi}
        V_\psi(\gamma_t,t)
    \right\|_{g_{\varphi_t^\psi}}  \\
    &\quad+
    \left\|
        \mathsf P_{\gamma_t\to\varphi_t^\psi}
        \left(
            V_\psi(\gamma_t,t)-\dot\gamma_t
        \right)
    \right\|_{g_{\varphi_t^\psi}} .
\end{aligned}
\]

The reconstructed vector field is assumed to be $L_v$-Lipschitz in $z$ on
$K$. In intrinsic Riemannian form, this means
\[
    \left\|
        V_\psi(x,t)
        -
        \mathsf P_{y\to x}V_\psi(y,t)
    \right\|_{g_x}
    \le
    L_v d(x,y),
    \qquad
    x,y\in K,\ t\in[0,1].
\]
Applying this with $x=\varphi_t^\psi$ and $y=\gamma_t$ gives
\[
    \left\|
        V_\psi(\varphi_t^\psi,t)
        -
        \mathsf P_{\gamma_t\to\varphi_t^\psi}
        V_\psi(\gamma_t,t)
    \right\|_{g_{\varphi_t^\psi}}
    \le
    L_v e(t).
\]
For the second term, the isometry of parallel transport gives
\[
\begin{aligned}
    \left\|
        \mathsf P_{\gamma_t\to\varphi_t^\psi}
        \left(
            V_\psi(\gamma_t,t)-\dot\gamma_t
        \right)
    \right\|_{g_{\varphi_t^\psi}}
    &=
    \left\|
        V_\psi(\gamma_t,t)-\dot\gamma_t
    \right\|_{g_{\gamma_t}} \\
    &=
    \varepsilon(t;z_S,z_T).
\end{aligned}
\]
Therefore,
\[
    D^+e(t)
    \le
    L_v e(t)+\varepsilon(t;z_S,z_T).
\]

By Gronwall's inequality and $e(0)=0$,
\[
\begin{aligned}
    e(1)
    &\le
    \int_0^1
    e^{L_v(1-t)}
    \varepsilon(t;z_S,z_T)
    \,dt  \\
    &\le
    e^{L_v}
    \int_0^1
    \varepsilon(t;z_S,z_T)
    \,dt .
\end{aligned}
\]
Since $e(1)=d(\varphi_1^\psi,z_T)$, we have the pointwise endpoint bound
\[
    d(\varphi_1^\psi,z_T)
    \le
    e^{L_v}
    \int_0^1
    \varepsilon(t;z_S,z_T)\,dt .
\]

Taking expectation over the coupling $\pi$ yields
\[
\begin{aligned}
    \mathbb E_{(z_S,z_T)\sim\pi}
    \left[
        d(\varphi_1^\psi,z_T)
    \right]
    &\le
    e^{L_v}
    \mathbb E_{(z_S,z_T)\sim\pi}
    \left[
        \int_0^1
        \varepsilon(t;z_S,z_T)\,dt
    \right]  \\
    &=
    e^{L_v}
    \mathbb E_{(z_S,z_T)\sim\pi,\;t\sim U[0,1]}
    \left[
        \varepsilon(t;z_S,z_T)
    \right].
\end{aligned}
\]
By Jensen's inequality, or equivalently Cauchy--Schwarz,
\[
\begin{aligned}
    \mathbb E_{(z_S,z_T)\sim\pi,\;t\sim U[0,1]}
    \left[
        \varepsilon(t;z_S,z_T)
    \right]
    &\le
    \left(
    \mathbb E_{(z_S,z_T)\sim\pi,\;t\sim U[0,1]}
    \left[
        \varepsilon(t;z_S,z_T)^2
    \right]
    \right)^{1/2}  \\
    &=
    \sqrt{\mathcal L_{\rm FM}} .
\end{aligned}
\]
Combining the previous two inequalities gives
\[
    \mathbb E_{(z_S,z_T)\sim\pi}
    \left[
        d_{\mathcal M_c}(\varphi_1^\psi,z_T)
    \right]
    \le
    e^{L_v}
    \sqrt{\mathcal L_{\rm FM}} .
\]

The above derivation is written in intrinsic Riemannian form, in which the
geometric constant can be taken as $1$. If the learned vector field is expressed
and analyzed through polar coordinates, the compactness of $K$ and smoothness of
the metric on the normal ball imply uniform equivalence between coordinate
norms, polar metric norms, and intrinsic Riemannian norms. These equivalence
constants depend only on $K=B_o(\rho)$ and the metric $g$, not on $\psi$,
$\pi$, or the sampled pairs. Collecting them into a single constant
$C_{\rm geo}>0$, we obtain
\[
    \mathbb E_{(z_S,z_T)\sim\pi}
    \left[
        d_{\mathcal M_c}(\varphi_1^\psi,z_T)
    \right]
    \le
    C_{\rm geo}e^{L_v}
    \sqrt{\mathcal L_{\rm FM}} .
\]
This proves the theorem.
\end{proof}

\section{Proof of Theorem 6}

\begin{theorem}[Target-Risk Bound with Polar Discrepancy]
Let $P_S$ and $P_T$ be source and target distributions over graph-label pairs.
Assume $\mathbf{z}=f_\theta(G)\in B_o(\rho)\subset\mathcal M_c$ and the loss
$\ell_\phi(\mathbf{z},y)=\ell(g_\phi(\mathbf{z}),y)\in[0,1]$ is $L_\ell$-Lipschitz in $z$. Let
$\Phi_\psi$ be the endpoint map induced by the learned polar flow and let
$\Psi$ be an ideal class-conditioned transport map. For each class $y$, let
$\Gamma_y$ be a lifted coupling over source--target graph pairs $(G,G')$ such
that the induced representation pair
$\mathbf z=\Psi(f_\theta(G))$ and $\mathbf z'=f_\theta(G')$ has marginals
$\Psi_\#P_S(\mathbf z\mid y)$ and $P_T(\mathbf z\mid y)$, respectively. Define
\[
D_{\rm rad}^{\Gamma}=
\sum_y\omega_y\mathbb E_{(G,G')\sim\Gamma_y}|r(\Psi(f_\theta(G)))-r(f_\theta(G'))|,
\; D_{\rm top}^{\Gamma}=
\sum_y\omega_y\mathbb E_{(G,G')\sim\Gamma_y}
\|q(G)-q(G')\|_2,
\]
\[
D_{\rm ang}^{\Gamma}=
\sum_y\omega_y\mathbb E_{(G,G')\sim\Gamma_y}d_{\mathbb S}(u(\Psi(f_\theta(G))),u(f_\theta(G'))),
\]
where $\omega_y=P_T(Y=y)$. Let $C_\rho=\sup_{0\le t\le\rho}S_c(t)$ and let
$\lambda_{\rm lab}=\frac12\sum_y|P_S(Y=y)-P_T(Y=y)|$. If pseudo-labels used to
construct the target class-conditionals have error $\epsilon_{\rm pl}$, then
there exist constants $C_{\rm top}$, $C_{\rm FM}$, and $C_{\rm pl}$ such that
\[
\begin{aligned}
    R_T(g_\phi)
    \le
    R_S(g_\phi\circ\Phi_\psi)
    +L_\ell\Big(
        D_{\rm rad}^{\Gamma}
        +C_\rho D_{\rm ang}^{\Gamma}
        +C_{\rm top}D_{\rm top}^{\Gamma}
        +C_{\rm FM}\sqrt{\mathcal L_{\rm FM}}
    \Big)+C_{\rm pl}\epsilon_{\rm pl}
    +\lambda_{\rm lab}.
\end{aligned}
\]
\end{theorem}

\begin{proof}
We interpret $\Gamma_y$ as a lifted coupling over graph pairs whose induced
representation marginals are
$\Psi_\#P_S(\mathbf z\mid y)$ and $P_T(\mathbf z\mid y)$.
Let $\mathcal Y=[K]$ be the label space. For a graph-label pair
$(G,Y)\sim P_D$, $D\in\{S,T\}$, define
\[
    Z=f_\theta(G)\in B_o(\rho)\subset\mathcal M_c .
\]
For each domain $D\in\{S,T\}$ and class $y\in\mathcal Y$, write
\[
    p_D(y):=P_D(Y=y),
    \qquad
    P_D^y:=P_D(Z\mid Y=y).
\]
Thus $p_D(y)$ is the class prior in domain $D$, and $P_D^y$ is the
class-conditional distribution of the representation $Z$.

The target risk of $g_\phi$ is
\[
    R_T(g_\phi)
    =
    \mathbb E_{(G,Y)\sim P_T}
    \big[
        \ell_\phi(f_\theta(G),Y)
    \big],
\]
where
\[
    \ell_\phi(\mathbf z,y)
    =
    \ell(g_\phi(\mathbf z),y)
    \in[0,1].
\]
The transported source risk induced by the learned endpoint map $\Phi_\psi$ is
\[
    R_S(g_\phi\circ\Phi_\psi)
    =
    \mathbb E_{(G,Y)\sim P_S}
    \big[
        \ell_\phi(\Phi_\psi(f_\theta(G)),Y)
    \big].
\]
By assumption, for every label $y$ and every
$\mathbf z,\mathbf z'\in B_o(\rho)$,
\[
    \left|
    \ell_\phi(\mathbf z,y)
    -
    \ell_\phi(\mathbf z',y)
    \right|
    \le
    L_\ell d_{\mathcal M_c}(\mathbf z,\mathbf z') .
\]

For each class $y$, let $\Gamma_y$ be a lifted coupling between source and
target graph-representation pairs such that the first representation marginal is
    $\Psi_\#P_S^y$
and the second representation marginal is
    $P_T^y$.
Equivalently, under $(G,G')\sim\Gamma_y$, the representation pair is
\[
    \mathbf z=\Psi(f_\theta(G)),
    \qquad
    \mathbf z'=f_\theta(G'),
\]
where $\mathbf z$ follows the ideal transported source class-conditional
distribution and $\mathbf z'$ follows the target class-conditional distribution.
The lifted form is used only so that topology statistics
$q(G)$ and $q(G')$ are well-defined.

We first decompose the target risk by class:
\[
\begin{aligned}
    R_T(g_\phi)
    &=
    \sum_{y\in\mathcal Y}
    P_T(Y=y)
    \mathbb E_{\mathbf z'\sim P_T^y}
    \big[
        \ell_\phi(\mathbf z',y)
    \big]  \\
    &=
    \sum_{y\in\mathcal Y}
    \omega_y
    \mathbb E_{(\mathbf z,\mathbf z')\sim\Gamma_y}
    \big[
        \ell_\phi(\mathbf z',y)
    \big],
\end{aligned}
\]
where
\[
    \omega_y=P_T(Y=y).
\]
Using the Lipschitz continuity of $\ell_\phi$ in its representation argument,
we obtain
\[
\begin{aligned}
    \ell_\phi(\mathbf z',y)
    &\le
    \ell_\phi(\mathbf z,y)
    +
    L_\ell d_{\mathcal M_c}(\mathbf z,\mathbf z').
\end{aligned}
\]
Therefore,
\[
\begin{aligned}
    R_T(g_\phi)
    &\le
    \sum_y
    \omega_y
    \mathbb E_{(\mathbf z,\mathbf z')\sim\Gamma_y}
    \big[
        \ell_\phi(\mathbf z,y)
    \big]  \\
    &\quad+
    L_\ell
    \sum_y
    \omega_y
    \mathbb E_{(\mathbf z,\mathbf z')\sim\Gamma_y}
    \big[
        d_{\mathcal M_c}(\mathbf z,\mathbf z')
    \big].
\end{aligned}
\]
Define the target-prior-weighted ideal transported source risk as
\[
    \bar R_S^\Psi(g_\phi)
    :=
    \sum_y
    \omega_y
    \mathbb E_{\mathbf z\sim \Psi_\#P_S^y}
    \big[
        \ell_\phi(\mathbf z,y)
    \big].
\]
Since the first marginal of $\Gamma_y$ is $\Psi_\#P_S^y$, the first term above
is exactly $\bar R_S^\Psi(g_\phi)$. Hence,
\[
    R_T(g_\phi)
    \le
    \bar R_S^\Psi(g_\phi)
    +
    L_\ell
    \sum_y
    \omega_y
    \mathbb E_{\Gamma_y}
    \big[
        d_{\mathcal M_c}(\mathbf z,\mathbf z')
    \big].
\]

We now upper bound the Riemannian distance by polar discrepancies. Write
\[
    \mathbf z=(r(\mathbf z),u(\mathbf z)),
    \qquad
    \mathbf z'=(r(\mathbf z'),u(\mathbf z'))
\]
in geodesic polar coordinates around the origin $o$. Since both points lie in
$B_o(\rho)$, we have
\[
    0\le r(\mathbf z),r(\mathbf z')\le \rho .
\]
Let
\[
    C_\rho=\sup_{0\le t\le\rho}S_c(t).
\]
Because $B_o(\rho)$ is compact and $S_c$ is continuous, $C_\rho<\infty$.

For any two non-origin points $\mathbf z,\mathbf z'\in B_o(\rho)$, construct a
piecewise smooth path from $\mathbf z$ to $\mathbf z'$ as follows. First move
radially from radius $r(\mathbf z)$ to radius $r(\mathbf z')$ while keeping the
angular direction fixed at $u(\mathbf z)$. This segment has length
\[
    |r(\mathbf z)-r(\mathbf z')|.
\]
Then move angularly on the geodesic sphere of radius $r(\mathbf z')$ from
$u(\mathbf z)$ to $u(\mathbf z')$. Under the polar metric
\[
    ds^2=dr^2+S_c(r)^2d\Omega^2,
\]
this angular segment has length
\[
    S_c(r(\mathbf z'))
    d_{\mathbb S}(u(\mathbf z),u(\mathbf z'))
    \le
    C_\rho
    d_{\mathbb S}(u(\mathbf z),u(\mathbf z')).
\]
The geodesic distance is no larger than the length of any admissible path, so
\[
\begin{aligned}
    d_{\mathcal M_c}(\mathbf z,\mathbf z')
    &\le
    |r(\mathbf z)-r(\mathbf z')|
    +
    C_\rho d_{\mathbb S}(u(\mathbf z),u(\mathbf z')).
\end{aligned}
\]
If one of the two points is the origin, the same inequality holds by continuity;
the angular coordinate at the origin can be chosen arbitrarily because
$S_c(0)=0$.

Since
\[
    \|q(G)-q(G')\|_2\ge 0,
\]
for any finite constant $C_{\rm top}\ge 0$ we also have the looser
topology-augmented bound
\[
\begin{aligned}
    d_{\mathcal M_c}(\mathbf z,\mathbf z')
    &\le
    |r(\mathbf z)-r(\mathbf z')|
    +
    C_\rho d_{\mathbb S}(u(\mathbf z),u(\mathbf z')) \\
    &\quad+
    C_{\rm top}\|q(G)-q(G')\|_2 .
\end{aligned}
\]
Here $q(G)$ denotes the chosen topology descriptor of graph $G$, and
$C_{\rm top}$ is a finite constant attached to this topology discrepancy. In
the purely polar metric bound one may take $C_{\rm top}=0$; retaining a positive
$C_{\rm top}$ gives the stated topology-conditioned form.

Taking expectation under each $\Gamma_y$ and summing with weights
$\omega_y$ gives
\[
\begin{aligned}
    \sum_y\omega_y
    \mathbb E_{\Gamma_y}
    \big[
        d_{\mathcal M_c}(\mathbf z,\mathbf z')
    \big]
    &\le
    \sum_y\omega_y
    \mathbb E_{\Gamma_y}
    \big[
        |r(\mathbf z)-r(\mathbf z')|
    \big] \\
    &\quad+
    C_\rho
    \sum_y\omega_y
    \mathbb E_{\Gamma_y}
    \big[
        d_{\mathbb S}(u(\mathbf z),u(\mathbf z'))
    \big] \\
    &\quad+
    C_{\rm top}
    \sum_y\omega_y
    \mathbb E_{\Gamma_y}
    \big[
        \|q(G)-q(G')\|_2
    \big].
\end{aligned}
\]
By the definitions in the theorem, this is
\[
    D_{\rm rad}^{\Gamma}
    +
    C_\rho D_{\rm ang}^{\Gamma}
    +
    C_{\rm top}D_{\rm top}^{\Gamma}.
\]
Thus,
\[
    R_T(g_\phi)
    \le
    \bar R_S^\Psi(g_\phi)
    +
    L_\ell
    \left(
        D_{\rm rad}^{\Gamma}
        +
        C_\rho D_{\rm ang}^{\Gamma}
        +
        C_{\rm top}D_{\rm top}^{\Gamma}
    \right).
\]

It remains to compare the target-prior-weighted ideal transported source risk
$\bar R_S^\Psi(g_\phi)$ with the actual source risk after the learned endpoint
map $\Phi_\psi$.

First, compare $\bar R_S^\Psi(g_\phi)$ with the source-prior-weighted ideal
transported source risk
\[
    R_S(g_\phi\circ\Psi)
    =
    \mathbb E_{(G,Y)\sim P_S}
    \big[
        \ell_\phi(\Psi(f_\theta(G)),Y)
    \big].
\]
For each class $y$, define
\[
    a_y
    :=
    \mathbb E_{\mathbf z\sim P_S^y}
    \big[
        \ell_\phi(\Psi(\mathbf z),y)
    \big].
\]
Since $\ell_\phi\in[0,1]$, we have
\[
    0\le a_y\le 1.
\]
Then
\[
    \bar R_S^\Psi(g_\phi)
    =
    \sum_y \omega_y a_y,
\]
whereas
\[
    R_S(g_\phi\circ\Psi)
    =
    \sum_y P_S(Y=y)a_y.
\]
Therefore,
\[
\begin{aligned}
    \bar R_S^\Psi(g_\phi)-R_S(g_\phi\circ\Psi)
    &=
    \sum_y
    \left(
        P_T(Y=y)-P_S(Y=y)
    \right)a_y .
\end{aligned}
\]
Let
\[
    \Delta_y=P_T(Y=y)-P_S(Y=y).
\]
Since both $P_T(Y=\cdot)$ and $P_S(Y=\cdot)$ are probability distributions,
\[
    \sum_y\Delta_y=0.
\]
Using $0\le a_y\le 1$, we have
\[
\begin{aligned}
    \sum_y\Delta_y a_y\le
    \sum_{y:\Delta_y>0}\Delta_y =
    \frac12\sum_y|\Delta_y| =
    \lambda_{\rm lab}.
\end{aligned}
\]
Hence,
\[
    \bar R_S^\Psi(g_\phi)
    \le
    R_S(g_\phi\circ\Psi)
    +
    \lambda_{\rm lab}.
\]

Next, compare the ideal endpoint map $\Psi$ with the learned endpoint map
$\Phi_\psi$. By Lipschitz continuity of the loss,
\[
\begin{aligned}
    R_S(g_\phi\circ\Psi)
    &=
    \mathbb E_{(G,Y)\sim P_S}
    \big[
        \ell_\phi(\Psi(f_\theta(G)),Y)
    \big] \\
    &\le
    \mathbb E_{(G,Y)\sim P_S}
    \big[
        \ell_\phi(\Phi_\psi(f_\theta(G)),Y)
    \big] \\
    &\quad+
    L_\ell
    \mathbb E_{(G,Y)\sim P_S}
    \big[
        d_{\mathcal M_c}
        (
            \Psi(f_\theta(G)),
            \Phi_\psi(f_\theta(G))
        )
    \big] \\
    &=
    R_S(g_\phi\circ\Phi_\psi)
    +
    L_\ell
    \mathbb E_{(G,Y)\sim P_S}
    \big[
        d_{\mathcal M_c}
        (
            \Psi(Z),
            \Phi_\psi(Z)
        )
    \big].
\end{aligned}
\]

Let $\pi_\Psi$ denote the population coupling that pairs each source
representation $Z$ with its ideal transported endpoint $\Psi(Z)$. Applying the
polar flow endpoint stability result to this coupling gives
\[
    \mathbb E_{Z\sim P_S}
    \big[
        d_{\mathcal M_c}
        (
            \Phi_\psi(Z),
            \Psi(Z)
        )
    \big]
    \le
    C_{\rm FM}\sqrt{\mathcal L_{\rm FM}},
\]
where $C_{\rm FM}$ is a finite constant depending on the compact normal ball,
the metric, and the Lipschitz constant of the reconstructed vector field. In
particular, under Theorem~\ref{thm:fm_endpoint_stability}, one may take
\[
    C_{\rm FM}=C_{\rm geo}e^{L_v}.
\]
Therefore,
\[
    R_S(g_\phi\circ\Psi)
    \le
    R_S(g_\phi\circ\Phi_\psi)
    +
    L_\ell C_{\rm FM}\sqrt{\mathcal L_{\rm FM}}.
\]

Combining the previous inequalities yields the bound under true target
class-conditionals:
\[
\begin{aligned}
    R_T(g_\phi)
    &\le
    R_S(g_\phi\circ\Phi_\psi)
    +
    L_\ell
    \left(
        D_{\rm rad}^{\Gamma}
        +
        C_\rho D_{\rm ang}^{\Gamma}
        +
        C_{\rm top}D_{\rm top}^{\Gamma}
        +
        C_{\rm FM}\sqrt{\mathcal L_{\rm FM}}
    \right)
    +
    \lambda_{\rm lab}.
\end{aligned}
\]

Finally, we account for pseudo-label noise. Let $\widehat Y$ denote the
pseudo-label used to construct target class-conditionals, and let
$\widehat P_T$ denote the pseudo-labeled target distribution over
$(G,\widehat Y)$. Suppose the pseudo-label error is
\[
    \epsilon_{\rm pl}
    =
    \Pr_{(G,Y)\sim P_T}(\widehat Y\ne Y).
\]
Under the canonical coupling that pairs $(G,Y)$ with $(G,\widehat Y)$, the two
joint distributions differ only when $\widehat Y\ne Y$. Hence their total
variation distance is bounded by
\[
    {\rm TV}\big(P_T(G,Y),\widehat P_T(G,\widehat Y)\big)
    \le
    \epsilon_{\rm pl}.
\]
Equivalently, for the class-conditional construction, assume the aggregate
pseudo-label perturbation satisfies
\[
    \sum_y
    \omega_y
    {\rm TV}
    \big(
        P_T(Z\mid Y=y),
        \widehat P_T(Z\mid \widehat Y=y)
    \big)
    \le
    C_{\rm cond}\epsilon_{\rm pl},
\]
where $C_{\rm cond}$ is finite when retained target classes have non-vanishing
class mass. This is the standard conditioning cost caused by replacing true
target labels with pseudo-labels.

Because $B_o(\rho)$ is compact, the manifold distance is bounded on the
representation support. Moreover, on the considered graph support, the topology
descriptor discrepancy is assumed to have finite bound. Therefore, every
loss/discrepancy term appearing in the preceding coupling argument is bounded
by a finite constant. Thus, by the variational characterization of total
variation, replacing true target class-conditionals by pseudo-labeled
class-conditionals changes the bound by at most
\[
    C_{\rm pl}\epsilon_{\rm pl},
\]
for some finite constant $C_{\rm pl}$ depending on the conditioning constant,
the compact ball, the loss bound, and the bounded topology statistics.

Adding this perturbation term gives
\[
\begin{aligned}
    R_T(g_\phi)
    \le
    R_S(g_\phi\circ\Phi_\psi)
    &+
    L_\ell
    \Big(
        D_{\rm rad}^{\Gamma}
        +
        C_\rho D_{\rm ang}^{\Gamma}
        +
        C_{\rm top}D_{\rm top}^{\Gamma}
        +
        C_{\rm FM}\sqrt{\mathcal L_{\rm FM}}
    \Big) \\
    &+
    C_{\rm pl}\epsilon_{\rm pl}
    +
    \lambda_{\rm lab}.
\end{aligned}
\]
This is the desired result.
\end{proof}

\section{Dataset}\label{sec:dataset}

\subsection{Dataset Description}

\input{table/dataset}
\input{figure/dataset_vis}

We conduct extensive experiments on a variety of datasets. The statistics of the datasets are summarized in Table \ref{tab:dataset}. More detailed descriptions of each dataset are provided as follows:

\begin{itemize}

\item For structure-based domain shifts:

\begin{itemize}

\item \textbf{PROTEINS.} The PROTEINS dataset~\cite{dobson2003distinguishing} consists of 1,113 protein graphs, each annotated with a binary label indicating whether the protein is an enzyme. In each graph, nodes correspond to amino acids, and edges connect amino acids that are within 6~\AA{} of each other along the sequence. We further partition the dataset into four subsets, denoted as P0, P1, P2, and P3, according to edge density and node density.

\item \textbf{NCI1.} The NCI1 dataset~\cite{wale2008comparison} consists of 4,110 molecular graphs, where nodes represent atoms and edges correspond to chemical bonds. Each graph is annotated with a binary label indicating whether the molecule inhibits cancer cell growth. We further partition the dataset into four subsets, denoted as N0, N1, N2, and N3, according to edge density and node density.

\item \textbf{Mutagenicity.} 
The Mutagenicity dataset~\cite{kazius2005derivation} consists of 4,337 molecular graphs, where nodes represent atoms and edges correspond to chemical bonds. Each graph is annotated with a binary label indicating whether the compound is mutagenic. Following the PROTEINS dataset, we further partition the dataset into four subsets, denoted as M0, M1, M2, and M3, according to edge density and node density.

\item \textbf{ogbg-molhiv.} The ogbg-molhiv dataset~\cite{hu2021ogblsc} consists of 41,127 molecular graphs, where nodes represent atoms and edges correspond to chemical bonds. Each graph is annotated with a binary label indicating whether the molecule exhibits HIV inhibitory activity. Following the PROTEINS dataset, we further partition the dataset into four subsets, denoted as H0, H1, H2, and H3, according to edge density and node density.

\end{itemize}

\item For feature-based domain shifts:

\begin{itemize}
\item \textbf{DD.} The DD dataset~\cite{dobson2003distinguishing} consists of 1,178 protein structure graphs, where nodes correspond to amino acids and edges encode spatial or chemical proximity between them. Compared to the PROTEINS dataset, DD graphs are substantially larger and denser, thereby introducing pronounced structural variability while preserving similar semantic labels.

\item \textbf{COX2.} The COX2 dataset~\cite{sutherland2003spline} consists of 467 molecular graphs, while COX2\_MD contains 303 modified molecular graphs. In both datasets, nodes correspond to atoms and edges represent chemical bonds. COX2\_MD introduces controlled structural variations to COX2 while preserving semantic labels.

\item \textbf{BZR.} The BZR dataset~\cite{sutherland2003spline} consists of 405 molecular graphs, while BZR\_MD contains 306 structurally modified graphs derived from BZR. In both datasets, nodes correspond to atoms and edges represent chemical bonds. BZR\_MD introduces controlled structural variations to simulate domain shifts while preserving consistent label semantics.

\end{itemize}

\end{itemize}

\subsection{Data Processing}

For datasets from the TUDataset~\footnote{https://chrsmrrs.github.io/datasets/}
 (e.g., PROTEINS and NCI1), we adopt the standard preprocessing and normalization procedures provided by PyTorch Geometric~\footnote{https://pyg.org/}
. For datasets from the Open Graph Benchmark (OGB)~\footnote{https://ogb.stanford.edu/}
, such as ogbg-molhiv, we follow the official OGB preprocessing and normalization protocols.

\section{Baselines}\label{sec:baselines}

In this section, we introduce the details of the compared baselines as follows:

\begin{itemize}
    \item \textbf{Graph kernel methods.} We compare \method{} with two graph kernel methods: 
    \begin{itemize}
        \item \textbf{WL subtree}: WL subtree \citep{shervashidze2011weisfeiler} is a graph kernel method that computes graph similarity via a kernel function, encoding local neighborhood structures into subtree patterns and efficiently capturing topological information within graphs.
        \item \textbf{PathNN}: PathNN~\citep{michel2023path} is an expressive graph neural network that models graphs by aggregating information along simple paths, enabling accurate capture of long-range dependencies and higher-order structural patterns beyond local neighborhoods.
    \end{itemize}
    \item \textbf{General Graph Neural Networks.} We compare \method{} with four general graph neural networks: 
    \begin{itemize}
        \item \textbf{GCN}: GCN \citep{kipf2017semi} is a message-passing graph neural network that updates node representations by aggregating and normalizing features from immediate neighbors, effectively capturing local structural and attribute information in graphs. 
        \item \textbf{GIN}: GIN \citep{xu2018how} is a graph neural network with an injective aggregation function that sums neighbor features, achieving strong expressive power equivalent to the Weisfeiler–Lehman test for distinguishing graph structures.
        \item \textbf{CIN}: CIN \citep{bodnar2021weisfeiler} is a higher-order graph neural network inspired by the Weisfeiler–Lehman framework, which operates on cellular complexes to capture rich topological structures beyond pairwise node interactions.
        \item \textbf{GMT}: GMT \citep{BaekKH21} is a graph neural network architecture that learns graph-level representations via hierarchical multiset pooling, enabling adaptive aggregation of node features and improving expressiveness for capturing complex graph structures.
    \end{itemize}
    \item \textbf{Manifold-based Graph Neural Networks}: We compare \method{} with four manifold-based general graph neural networks:
    \begin{itemize}
        \item \textbf{dDGM}: dDGM~\citep{de2023latent} is a latent graph inference framework that models graph structure in continuous product manifolds, enabling the learning of expressive graph representations by jointly inferring latent relational structure and node embeddings.
        \item \textbf{RieGrace}: RieGrace~\citep{sun2023self} is a self-supervised continual graph learning approach that adapts Riemannian representation spaces over time, enabling robust knowledge accumulation and transfer across evolving graph distributions.
        \item \textbf{ProGDM}: ProGDM~\citep{wang2024mixed} is a graph diffusion model that integrates mixed-curvature geometric spaces, enabling effective information propagation and representation learning across graphs with heterogeneous structural properties.
        \item \textbf{D-GCN}: D-GCN~\citep{sun2024motif} is a motif-aware Riemannian graph neural network that leverages generative–contrastive learning to capture higher-order structural patterns and geometric relationships in graph representations.
    \end{itemize}
    \item \textbf{Graph Domain Adaptation methods.} We compare \method{} with seven graph domain adaptation methods: 
    \begin{itemize}
        \item \textbf{DEAL}: DEAL \citep{yin2022deal} addresses unsupervised domain adaptation for graph-level classification by learning domain-invariant graph representations through adversarial training, where a shared graph encoder is optimized to confuse a domain discriminator while preserving discriminative power on source labels.
        \item \textbf{SGDA}: SGDA \citep{qiao2023semi} performs semi-supervised domain adaptation for graph transfer learning by jointly aligning feature representations and structural distributions across source and target domains, enabling effective knowledge transfer under limited target supervision. 
        \item \textbf{StruRW}: StruRW \citep{liu2023structural} improves graph domain adaptation by re-weighting structural components to mitigate distributional shifts and enhance cross-domain representation alignment.
        \item \textbf{A2GNN}: A2GNN \citep{liu2024rethinking} rethinks information propagation for unsupervised graph domain adaptation by redesigning message-passing mechanisms to better align cross-domain graph representations.
        \item \textbf{PA-BOTH}: PA-BOTH \citep{liu2024pairwise} enhances graph domain adaptation through pairwise alignment of node representations and structural patterns across domains, effectively reducing distribution shifts and facilitating robust knowledge transfer between source and target graphs. \item \textbf{GAA}: GAA~\citep{fang2025benefits} investigates attribute-driven graph domain adaptation by leveraging node attributes to guide representation alignment across domains, demonstrating that attribute information plays a critical role in mitigating domain shifts and improving transfer performance.
        \item \textbf{TDSS}: TDSS~\citep{chen2025smoothness} emphasizes graph smoothness as a key inductive bias for unsupervised graph domain adaptation, enforcing consistent representations among neighboring nodes to effectively reduce domain discrepancy and improve cross-domain generalization.
    \end{itemize}
    \item \textbf{Manifold-based Domain Adaptation methods.} We compare \method{} with three manifold-based domain adaptation methods:
    \begin{itemize}
        \item \textbf{GOTDA}: GOTDA~\citep{long2022domain} formulates domain adaptation as an optimal transport problem on Grassmann manifolds, aligning subspace representations across domains to reduce distributional discrepancy and enable effective knowledge transfer. 
        \item \textbf{MASH}: MASH~\citep{rustad2024graph} performs graph integration for diffusion-based manifold alignment by leveraging diffusion processes to align manifold structures across domains, facilitating coherent representation learning and cross-domain knowledge transfer.
        \item \textbf{GeoAdapt}: GeoAdapt~\cite{gharib2025geometric} performs domain adaptation by aligning geometric moments via Siegel embeddings, enabling principled matching of higher-order distributional statistics across domains in a structured geometric space.
    \end{itemize}
\end{itemize}

\subsection{Implementation Details}

\noindent\textbf{Implementation Details.} We implement \method{} and all baselines in PyTorch and conduct all experiments on NVIDIA A100 GPUs. For the encoder, \method{} adopts HGCN~\citep{liu2019hyperbolic} as the backbone, enabling adaptive representation learning across hyperbolic, Euclidean, and spherical spaces by modulating the curvature parameter $c$. All experiments use identical hyperparameter settings, with Adam, a learning rate of $1\times10^{-4}$, weight decay of $1\times10^{-12}$, hidden dimension 128, and three GNN layers. We set the manifold curvature to $c=-1.0$, the confidence threshold to $\zeta = 0.7$, and $\lambda_{1} = \lambda_{2} = \lambda_{3} = 0.1$. Following prior settings~\citep{wu2020unsupervised,yin2022deal}, \method{} is trained on labeled source data and evaluated on unlabeled target data. We report accuracy on TUDataset benchmarks (e.g., PROTEINS) and AUC on OGB datasets (e.g., ogbg-molhiv), averaged over five independent runs.

\section{Complexity Analysis}

Here we analyze the computational complexity of the proposed \method{}. The computational complexity is dominated by the Riemannian graph convolution and manifold flow matching. For a given batch of graphs, $|\mathcal{E}|$ denotes the total number of edges. $d$ is the embedding dimension. $L$ denotes the number of the HGCN. $|\mathcal{V}|$ is the total number of nodes. $B$ represents the batch size. The HGCN takes $\mathcal{O}\left(L \cdot\left(|\mathcal{E}| \cdot d+|\mathcal{V}| \cdot d^2\right)\right)$ time complexity, while the flow matching and geometric constraints introduce an additional term of $\mathcal{O}\left(B^2 \cdot d\right)$ for coupling and alignment. As a result, the overall computational complexity of \method{} is $\mathcal{O}\left(L \cdot\left(|\mathcal{E}| \cdot d+|\mathcal{V}| \cdot d^2\right)+B^2 \cdot d\right)$.

\section{More experimental results}

\subsection{More performance comparison}\label{sec:model performance}

In this section, we present additional experimental results comparing the proposed \method{} with all baseline models across various datasets, as reported in Tables~\ref{tab:proteins_node}–\ref{tab:hiv_idx}. We further observe that \method{} consistently outperforms baselines in most cases, further demonstrating its effectiveness and robustness.

Additionally, we conduct experiments to evaluate the flexibility of \method{}. Specifically, under the HGCN geometric framework~\citep{liu2019hyperbolic}, we replace the tangent space aggregation module with different GNN architectures (i.e., GCN, SAGE, GAT, and GIN), and report the results in Figure~\ref{fig:different_GNN}. The results show that GIN consistently outperforms other GNN architectures across most settings, indicating its stronger representational capacity. This observation further supports our choice of GIN as the default message passing mechanism to enhance the overall performance of \method{}. 

\subsection{More Ablation Study}\label{sec:ablation study}

To validate the effectiveness of each component in \method{}, we further conduct ablation studies on the PROTEINS, NCI1, and ogbg-molhiv datasets. Specifically, we evaluate four variants of \method{}, including \method{} w/o FM, \method{} w/o RA, \method{} w/o AA, and \method{} w/o PE. The experimental results are reported in Tables \ref{tab:proteins_ablation}, \ref{tab:nci1_ablation}, and \ref{tab:molhiv_ablation}. From these results, we observe trends consistent with those summarized in Section \ref{sec:ablation}.

\subsection{More Sensitivity Analysis}\label{sec:sensitive analysis}

\begin{wrapfigure}{r}{0.52\textwidth}
    \centering
    \vspace{-1cm} 
    \begin{subfigure}[b]{0.48\linewidth}
        \centering
        \includegraphics[width=\linewidth]{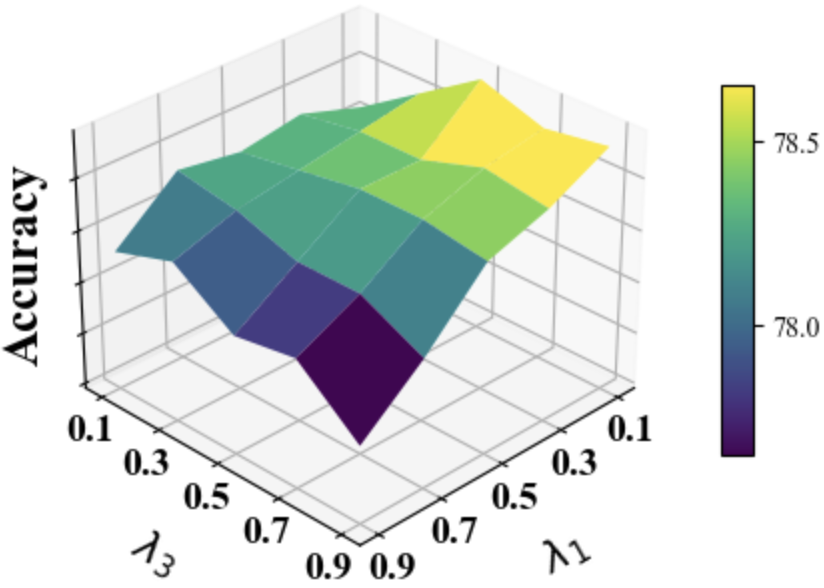}
        \label{fig:lambda1_lambda3}
    \end{subfigure}
    \begin{subfigure}[b]{0.48\linewidth}
        \centering
        \includegraphics[width=\linewidth]{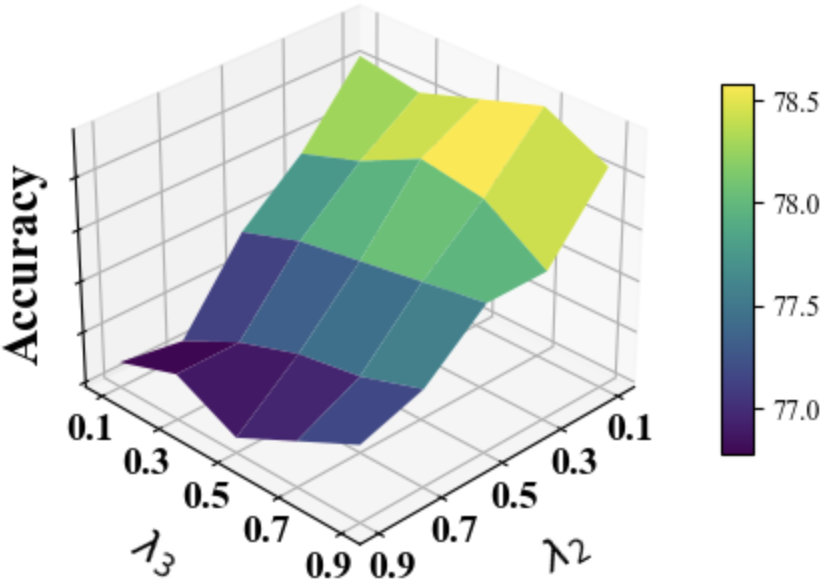}
        \label{fig:lambda2_lambda3}
    \end{subfigure}
    \caption{The model performance w.r.t. different combinations of $\lambda_1$, $\lambda_2$ and $\lambda_3$.}
    \label{fig:lambda_surface}
    \vspace{-0.4cm}
\end{wrapfigure}

In this section, we first investigate the sensitivity of the proposed \method{} to the angular confidence threshold $\zeta$ on the Mutagenicity, PROTEINS, NCI1, and ogbg-molhiv datasets. The threshold $\zeta$ governs the reliability of target pseudo-labels by excluding ambiguous samples during angular alignment. Figure~\ref{fig:hyper_threshold} reports the performance of \method{} as $\zeta$ varies over $\{0.5, 0.6, 0.7, 0.8, 0.9\}$. As shown in Figure~\ref{fig:hyper_threshold}, the performance consistently improves as $\zeta$ increases from 0.5 to 0.7, followed by a mild degradation at higher values. This behavior suggests that a low threshold allows noisy target samples to participate in angular alignment, leading to negative transfer, whereas an overly strict threshold excludes informative samples and weakens semantic guidance. Based on this trade-off, we set $\zeta = 0.7$ as the default value. We further analyze the sensitivity of \method{} to the curvature parameter $c$ on the PROTEINS, NCI1, and ogbg-molhiv datasets. The results, summarized in Figure~\ref{fig:hyper_curve}, exhibit trends consistent with the analysis presented in Section~\ref{sec:sensitivity}.

Additionally, we provide sensitivity analyses of the proposed \method{} with respect to
balance coefficients ($\lambda_1, \lambda_3$) and ($\lambda_2, \lambda_3$) on the Mutagenicity dataset. The results are illustrated in Figure~\ref{fig:lambda_surface}, from which we observe trends consistent with those discussed in Section \ref{sec:sensitivity}.

\subsection{Training Time and Memory Comparison}\label{sec:gpu_and_time}

Tables~\ref{tab:training_time} and \ref{tab:gpu_memory} present a comprehensive comparison of training time per epoch and GPU memory consumption among \method{} and existing graph domain adaptation methods under identical settings. The results indicate that \method{} incurs moderate computational cost and memory overhead across multiple datasets.

\begin{table}[t]
\centering
\caption{Training time per epoch of different graph domain adaptation methods (in seconds).}
\vspace{0.2cm}
\label{tab:training_time}
\resizebox{0.55\linewidth}{!}{
\begin{tabular}{lcccc}
\toprule
Methods & PROTEINS & NCI1 & Mutagenicity & ogbg-molhiv \\
\midrule
DEAL   & 0.176 & 0.528 & 0.598 & 1.347 \\
SGDA   & 0.072 & 0.247 & 0.260 & 1.073 \\
GAA    & 0.061 & 0.194 & 0.200 & 1.101 \\
TDSS   & 0.057 & 0.311 & 0.169 & 1.014 \\
\method{} & 0.091 & 0.477 & 0.511 & 1.737 \\
\bottomrule
\end{tabular}
}
\end{table}

\begin{table}[t]
\centering
\caption{GPU memory consumption of different graph domain adaptation methods during training (in GB).}
\label{tab:gpu_memory}
\vspace{0.2cm}
\resizebox{0.55\linewidth}{!}{
\begin{tabular}{lcccc}
\toprule
 & PROTEINS & NCI1 & Mutagenicity & ogbg-molhiv \\
\midrule
DEAL   & 1.6 & 1.3 & 1.4 & 2.3 \\
SGDA   & 1.5 & 1.5 & 1.3 & 1.9 \\
GAA    & 1.8 & 1.4 & 1.7 & 2.7 \\
TDSS   & 1.4 & 1.2 & 1.4 & 1.8 \\
\method{} & 1.6 & 1.3 & 1.5 & 2.0 \\
\bottomrule
\end{tabular}
}
\end{table}

\begin{table}[t]
\small
\centering
\caption{Correlation between the radial coordinate $r$ and graph complexity. Partial Corr indicates the Pearson correlation after controlling for the class label $y$.}
\vspace{0.2cm}
\begin{tabular}{l|c|c|c}
\toprule
\textbf{Structure $a(G)$} & \textbf{Spearman} & \textbf{Pearson} & \textbf{Partial Corr} \\
\midrule
Number of Nodes ($|V|$) & 0.5783 & 0.5366 & 0.5227 \\
Number of Edges ($|E|$) & 0.5773 & 0.6076 & 0.5949 \\
\bottomrule
\end{tabular}
\label{tab:structural_correlation}
\vspace{-0.2cm}
\end{table}

\begin{table}[t]
\small
\centering

\caption{Experimental results of structural prediction probe.}
\vspace{0.2cm}
\begin{tabular}{l|c|c}
\toprule
\textbf{Input Features} & \textbf{$R^2$ ($\uparrow$)} & \textbf{MAE ($\downarrow$)} \\
\midrule
$r$ only (1-dim) & 0.4454 & 8.7791 \\
$u$ only (128-dim) & 0.0812 & 12.3421 \\
$r + u$ (Combined) & \textbf{0.7072} & \textbf{6.8961} \\
Tangent Space & 0.5113 & 7.7992 \\
\bottomrule
\end{tabular}
\label{tab:structural_probe}
\vspace{-0.4cm}
\end{table}

\begin{table}[t]
\small
\centering
\caption{Experimental results of semantic probe analysis and feature permutation.}
\vspace{0.2cm}
\begin{tabular}{l|c|c}
\toprule
\textbf{Input Combination} & \textbf{Source Acc} & \textbf{Target Acc} \\
\midrule
$r$ only & 59.41\% & 63.19\% \\
$u$ only & 91.79\% & 78.41\% \\
$r + u$ & \textbf{92.39\%} & \textbf{78.71\%} \\
\midrule
Shuffled $r$ + Orig $u$ & 90.77\% & 77.77\% \\
Orig $r$ + Shuffled $u$ & 50.83\% & 53.04\% \\
\bottomrule
\end{tabular}
\label{tab:semantic_probe}
\end{table}

\begin{table}[t]
\small
\centering
\caption{Counterfactual radial intervention.}
\vspace{0.2cm}
\begin{tabular}{l|c|c|c|l}
\toprule
\textbf{Sample} & \textbf{Scale} & \textbf{New Radius ($r'$)} & \textbf{Pred Nodes (Structure)} & \textbf{Class Prob} \\
\midrule
Graph \#0 & 0.5 & 0.7019 & 30.77 & Class 1: 68.52\% \\
Graph \#0 & 1.0 & 1.4039 & 31.40 & Class 1: 68.52\% (Original)\\
Graph \#0 & 1.5 & 2.1058 & 32.04 & Class 1: 68.52\% \\
Graph \#0 & 2.0 & 2.8078 & 32.67 & Class 1: 68.52\% \\
\midrule
Graph \#1 & 0.5 & 0.7304 & 30.80 & Class 0: 93.34\% \\
Graph \#1 & 1.0 & 1.4609 & 31.45 & Class 0: 93.34\% (Original)\\
Graph \#1 & 1.5 & 2.1913 & 32.11 & Class 0: 93.34\% \\
Graph \#1 & 2.0 & 2.9217 & 32.77 & Class 0: 93.34\% \\
\bottomrule
\end{tabular}
\label{tab:radial_intervention}
\vspace{-0.3cm}
\end{table}

\begin{table}[t]
\small
\centering
\caption{Counterfactual angular intervention.}
\vspace{0.2cm}
\begin{tabular}{l|l|c|c|c}
\toprule
\textbf{Alpha} & \textbf{Angular Source} & \textbf{Pred Nodes (Struct)} & \textbf{Prob Class 0} & \textbf{Prob Class 1} \\
\midrule
0.00 & 100\% $u_A$ (Class 0) & 31.45 & 93.34\% & 6.66\% \\
0.25 & Mixed (0.25) & 31.45 & 87.80\% & 12.20\% \\
0.50 & Mixed (0.50) & 31.45 & 74.45\% & 25.55\% \\
0.75 & Mixed (0.75) & 31.45 & 52.12\% & 47.88\% \\
1.00 & 100\% $u_B$ (Class 1) & 31.45 & 31.48\% & 68.52\% \\
\bottomrule
\end{tabular}
\label{tab:angular_intervention}
\vspace{-0.2cm}
\end{table}

\begin{table}[t]
\small
\centering
\caption{Evaluation of source-target coupling precision.}
\vspace{0.2cm}
\begin{tabular}{l|c}
\toprule
\textbf{Coupling Strategy} & \textbf{Precision} \\
\midrule
Random coupling & 50.55\% \\
Global nearest neighbor & 77.31\% \\
Pseudo-label class-conditioned NN & 79.15\% \\
Confidence-gated class-cond NN (Ours) & \textbf{81.33\%} \\
\bottomrule
\end{tabular}
\label{tab:coupling_precision}
\vspace{-0.2cm}
\end{table}

\begin{table}[ht]
\small
\centering
\caption{Trajectory analysis of Riemannian Flow Matching.}
\vspace{0.2cm}
\begin{tabular}{l|l}
\toprule
\textbf{Trajectory Metric} & \textbf{Evaluation Result} \\
\midrule
Radial Discrepancy ($W_1$ Dist) & $0.4176$ $\rightarrow$ $\mathbf{0.0000}$ \\
Angular Drift ($AD_{\text{ang}}$) & $0.1533$ radians \\
Class Boundary Crossings & $\mathbf{0.00\%}$ \\
Path Energy ($E(\gamma)$) & \textbf{0.0033} (Hyperbolic FM) vs. 0.0246 (Euclidean) \\
\bottomrule
\end{tabular}
\label{tab:flow_trajectory}
\end{table}

\subsection{More Operational Disentanglement Analysis}\label{sec:disen_more}

A key empirical question in \method{} is whether the learned polar coordinates provide an operational separation between structure-related and label-predictive signals. Specifically, we examine whether the radial coordinate $r$ serves as a compact proxy for graph-size statistics, while the angular coordinate $u$ carries most class-discriminative information under the trained representation. We emphasize that this analysis does not assume a globally identifiable topology-semantics factorization. Instead, it evaluates whether the intended radial-angular bias is reflected in probe performance, permutation ablations, and latent-space interventions. We conduct three diagnostic analyses to assess this behavior.

\subsubsection{Operational Evidence for Graph-Structure Signals in Radial Coordinates}

\textbf{Structural Correlation Analysis.} We examine whether the radial coordinate $r$ is associated with simple graph-size statistics, including the number of nodes and edges. As shown in Table~\ref{tab:structural_correlation}, $r$ exhibits substantial positive correlations with both node count and edge count, suggesting that the learned radius carries information about graph complexity under our representation geometry. This pattern is consistent with the intended radial inductive bias, where radial distance provides a scale-sensitive coordinate for representing structural variation. To reduce the concern that this correlation is explained only by class labels, we further compute partial correlations after controlling for the ground-truth label $y$. The coefficients remain close to the unconditional correlations, suggesting that $r$ retains graph-size information beyond what is explained by the label under this linear control.

\textbf{Structural Prediction Probe.} To test whether structural information is more linearly accessible from $r$ than from $u$, we freeze the encoder and train a linear Ridge regression probe to predict the number of edges. As a baseline, we use tangent-space representations, where standard Euclidean vectors are obtained via the logarithmic map without explicit polar factorization. As shown in Table~\ref{tab:structural_probe}, using only the scalar $r$ achieves a moderate $R^2$, indicating that the radius captures a nontrivial amount of edge-count information. In contrast, the angular component $u$ alone yields a much lower $R^2$, suggesting that this particular structural statistic is less linearly accessible from the angular direction. Combining $r$ and $u$ further improves prediction performance and outperforms the tangent-space baseline in this probe. These results support an operational separation in which $r$ provides a compact structure-related signal, while $u$ contributes complementary information without concentrating the edge-count signal.

\subsubsection{Operational Evidence for Class Semantics in Angular Coordinates}

\textbf{Semantic Probe Analysis.} To assess whether label-predictive information is primarily accessible from the angular component, we evaluate the standalone classification performance of each decoupled factor. As shown in Table~\ref{tab:semantic_probe}, using only the radial coordinate ($r$) yields accuracy only modestly above random guessing for this binary task, indicating that $r$ contains limited label-predictive information compared with the angular component. In contrast, the angular component ($u$) alone achieves performance close to that of the full representation ($r+u$). This suggests that class-discriminative information is predominantly encoded in $u$, while incorporating $r$ provides only a small additional benefit for classification in this setting.

\textbf{Feature Permutation Ablation.} To further evaluate the role of each factor, we conduct a feature permutation experiment on the test set using a classifier trained on the full ($r+u$) representation, as shown in Table~\ref{tab:semantic_probe}. When the radial values are randomly shuffled across samples while preserving their original angular directions (Shuffled $r$ + Orig $u$), performance remains close to the original result, suggesting that the classifier is not strongly sensitive to sample-specific radial values. In contrast, when the angular directions are shuffled while keeping the true radial values (Orig $r$ + Shuffled $u$), accuracy drops sharply to near-random chance. This contrast indicates that semantic predictions in this experiment depend much more strongly on $u$ than on $r$, supporting the intended operational separation between radial and angular factors.

\subsubsection{Latent-Space Intervention Analysis}

\textbf{Radial Intervention.} We perform a latent-space intervention by scaling the radial coordinate of a target sample ($r' = \text{scale} \cdot r$) while fixing its angular direction ($u$). The intervened tangent representation $v'_t = r' \cdot u$ is then evaluated using both the structural probe and the semantic classifier. As shown in Table~\ref{tab:radial_intervention}, increasing the scaling factor from $0.5$ to $2.0$ induces a monotonic increase in the predicted structural score. In contrast, the semantic classification probabilities remain nearly unchanged within the reported precision. This result suggests that, for the selected samples and probes, radial changes mainly affect the structure-related predictor while having limited influence on the semantic classifier.

\textbf{Angular Intervention.} We perform the inverse intervention by fixing the radial coordinate $r$ while smoothly interpolating the angular component $u$ toward a sample from a different class. Specifically, we construct counterfactual tangent representations $v' = r_A \cdot u_{\text{interp}}$, where $u_{\text{interp}}$ transitions from a Class 0 sample to a Class 1 sample. As shown in Table~\ref{tab:angular_intervention}, interpolating $u$ induces a substantial semantic shift and eventually flips the classifier's prediction from Class 0 to Class 1. Meanwhile, the predicted structural score remains stable within the reported precision. Together with the radial intervention, this provides additional empirical evidence that the radial magnitude is associated with structure-related variation, while the angular direction carries most of the label-predictive information in this experiment.

\subsection{Other Analysis}

\subsubsection{Evaluation of Class-Consistent Flow Coupling}

A key challenge in \method{} is to construct semantically reliable source--target trajectories while reducing arbitrary cross-domain mixing. Since target labels are unavailable during training, we evaluate different coupling strategies using ground-truth labels only as an oracle metric. As shown in Table~\ref{tab:coupling_precision}, coupling precision improves monotonically across the evaluated strategies. Random coupling produces substantial semantic mismatch, while global nearest-neighbor coupling and standard pseudo-label coupling improve alignment but may still introduce mismatch noise and potentially contribute to negative transfer. In contrast, our proposed Confidence-Gated Class-Conditioned Coupling achieves the highest precision among the compared strategies. By filtering uncertain boundary samples, the method reduces likely semantic mismatches and provides cleaner source--target pairs for flow construction. Consequently, the learned vector field is supervised on trajectories that are more likely to connect semantically consistent source and target representations, supporting more semantic-aware transport.

\subsubsection{Trajectory Diagnostics of Riemannian Flow Matching}

A key empirical question is whether the learned Riemannian flow trajectories respect the intended radial--angular separation during transport. To examine this, we analyze integrated trajectory metrics along the ODE path. As shown in Table~\ref{tab:flow_trajectory}, the results suggest the following observations:

\begin{itemize}
\item \textbf{Radial Discrepancy Reduction.} The flow substantially reduces the radial Wasserstein-1 ($W_1$) discrepancy between source and target domains in the evaluated trajectories, suggesting improved alignment of structure-related radial statistics.

\item \textbf{Small Angular Drift ($AD_{\mathrm{ang}}$).} Despite radial alignment, the accumulated angular drift remains small. This indicates that the transport introduces limited perturbation to the angular direction, which is consistent with preserving most label-predictive information under the trained classifier.

\item \textbf{No Observed Boundary Crossings.} Along the evaluated trajectories, we observe no class-boundary crossings under the trained decision function. This suggests that the learned transport does not visibly change the predicted semantic region for these samples, although it should be interpreted as an empirical diagnostic rather than a formal invariance guarantee.

\item \textbf{Lower Path Energy ($E(\gamma)$).} The Riemannian trajectory has lower path energy than the Euclidean transport baseline in this diagnostic. This is consistent with the geometry-aware flow following more efficient manifold-adapted paths, but does not imply global optimality of the learned transport.
\end{itemize}

Overall, these trajectory-level metrics provide empirical evidence that the learned Riemannian flow mainly adjusts radial discrepancies while introducing limited angular perturbation in the reported setting. This supports the intended structure-aware and class-consistent behavior of \method{}, without assuming exact radial--angular independence throughout the entire latent space.

\newpage
\input{figure/different_gnn}
\input{figure/hyper_threshold}
\input{figure/hyper_curve}
\input{table/Mutag_ablation}
\input{table/PROTEINS_ablation}
\input{table/NCI1_ablation}
\input{table/Molhiv_ablation}

\input{table/PROTEINS_node}
\input{table/PROTEINS_idx}
\input{table/NCI1_node}
\input{table/NCI1_idx}
\input{table/Mutag_node}
\input{table/Mutag_idx}
\input{table/Molhiv_node}
\input{table/Molhiv_idx}

%% file: table/dataset.tex
\begin{table}[ht]
    \centering
    \caption{Statistics of the experimental datasets.}
    \begin{tabular}{lcccc}
        \toprule
        Datasets      & Graphs & Avg. Nodes & Avg. Edges & Classes \\
        \midrule
        NCI1  & 4,110   & 29.87     & 32.30     & 2       \\
        MUTAGENICITY  & 4,337   & 30.32      & 30.77      & 2       \\
        PROTEINS  & 1,113   & 39.1      & 72.8      & 2       \\
        ogbg-molhiv & 41,127 & 25.5 & 27.5 & 2 \\
        
        \midrule
        DD & 1,178 & 284.32 & 715.66 & 2 \\
        COX2 & 467 & 41.22 & 43.45& 2 \\
        COX2\_MD & 303 & 26.28 &	335.12 & 2\\
        BZR & 405 & 35.75 & 38.36 & 2 \\
        BZR\_MD & 306 & 21.30 &	225.06 & 2 \\
        \bottomrule
    \end{tabular}
    \label{tab:dataset}
\end{table}

%% file: figure/dataset_vis.tex
\begin{figure*}[t]
\includegraphics[width=1.0\linewidth]{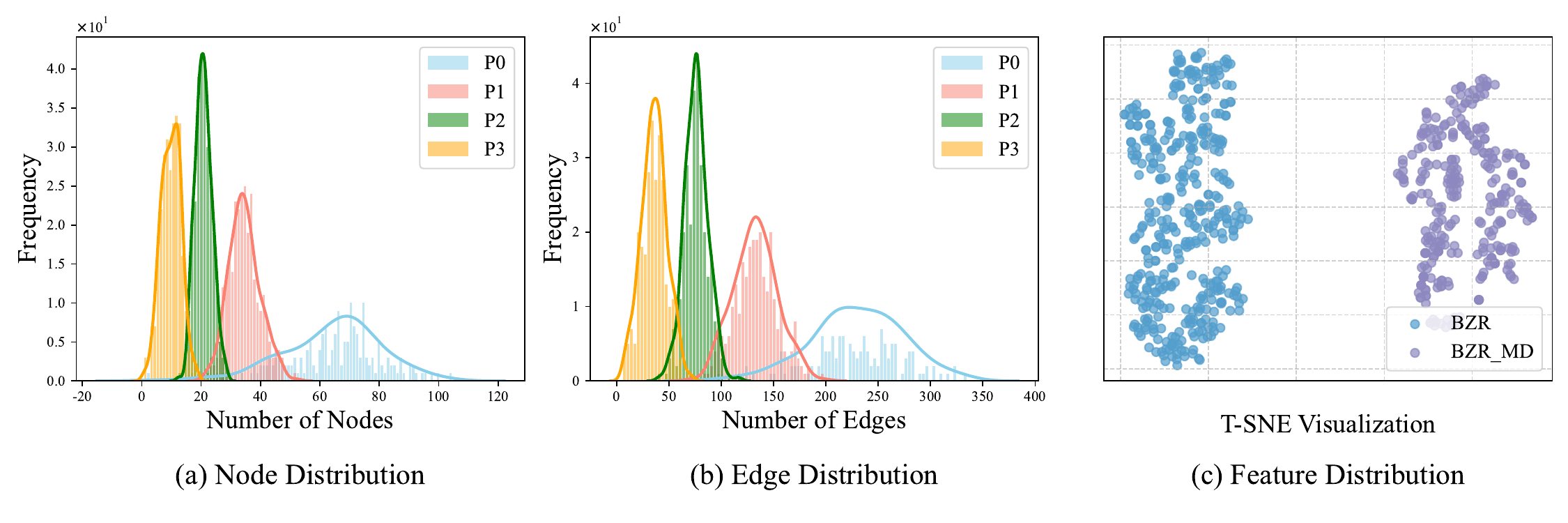}
    \vspace{-0.8cm}
    \caption{Visualization of domain shifts across different types. (a) Node distribution shift between sub-datasets of PROTEINS. (b) Edge distribution shift between sub-datasets of PROTEINS. (c) Feature distribution shift between BZR and BZR\_MD datasets.}
    \label{fig:shift}
    \vspace{-0.3cm}
\end{figure*}

%% file: figure/different_gnn.tex
\begin{figure*}[t]
    \centering
    \captionsetup[subfigure]{font=small}
    \begin{subfigure}[t]{0.24\textwidth}
        \centering
        \includegraphics[width=\textwidth]{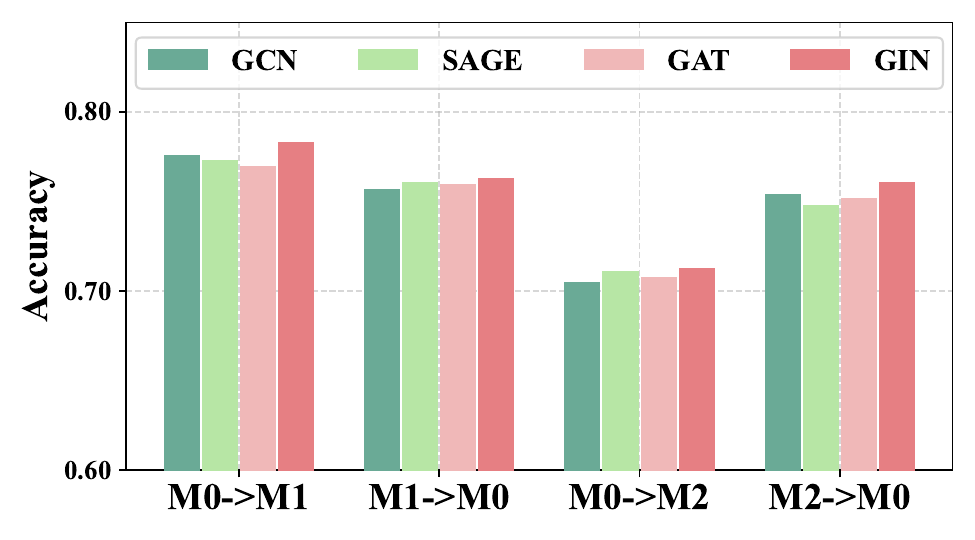}
        \caption{Mutagenicity}
    \end{subfigure}
    \hfill
    \begin{subfigure}[t]{0.24\textwidth}
        \centering
        \includegraphics[width=\textwidth]{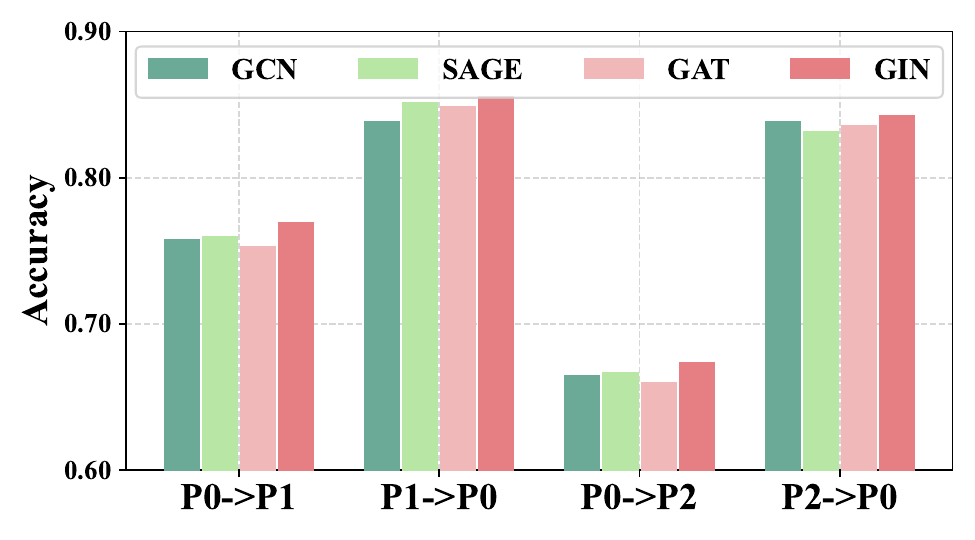}
        \caption{PROTEINS}
    \end{subfigure}
    \hfill
    \begin{subfigure}[t]{0.24\textwidth}
        \centering
        \includegraphics[width=\textwidth]{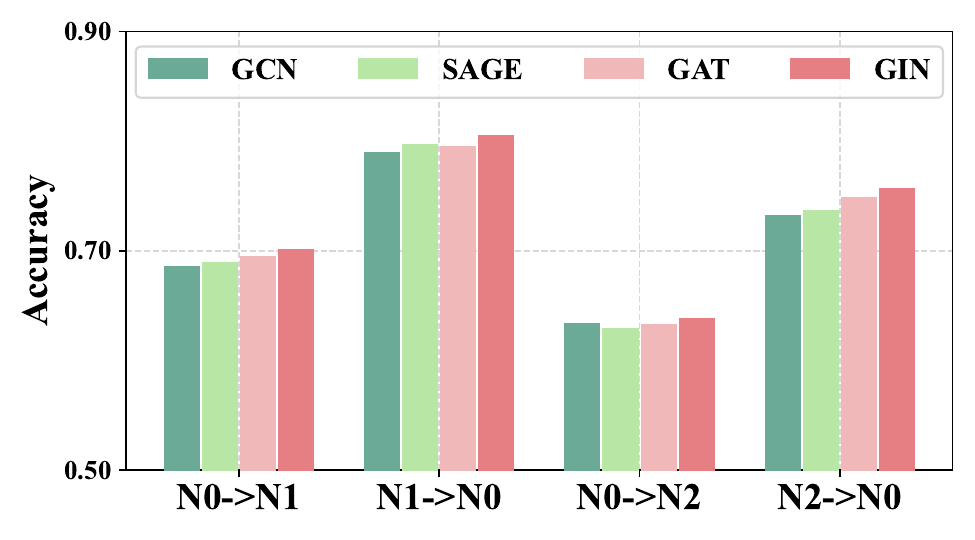}
        \caption{NCI1}
    \end{subfigure}
    \hfill
    \begin{subfigure}[t]{0.24\textwidth}
        \centering
        \includegraphics[width=\textwidth]{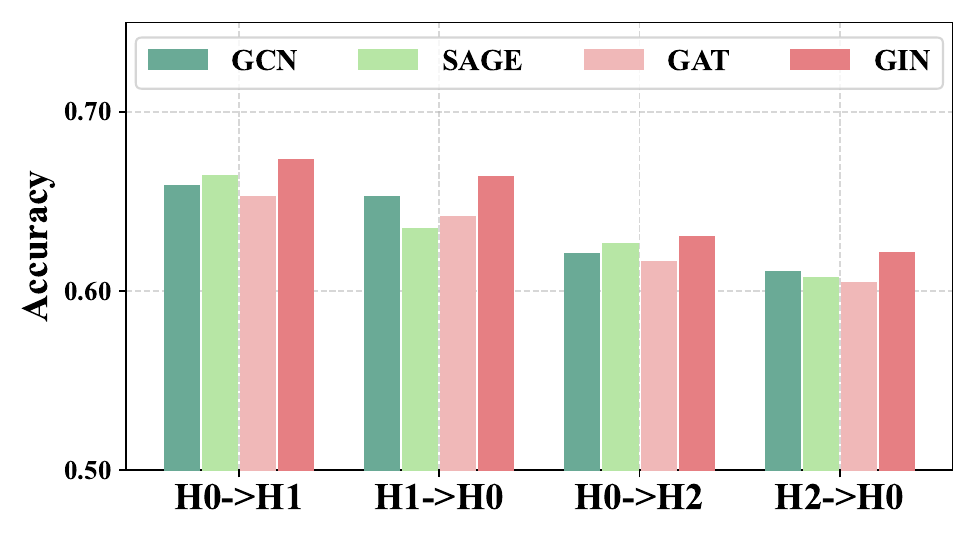}
        \caption{ogbg-molhiv}
    \end{subfigure}
    \caption{The performance with different tangent space aggregation modules on different datasets.}
    \label{fig:different_GNN}
\end{figure*}

%% file: figure/hyper_threshold.tex
\begin{figure*}[t]

    \centering
    \captionsetup[subfigure]{font=scriptsize} 
    \begin{subfigure}[t]{0.24\textwidth}
        \centering
        \includegraphics[width=\linewidth]{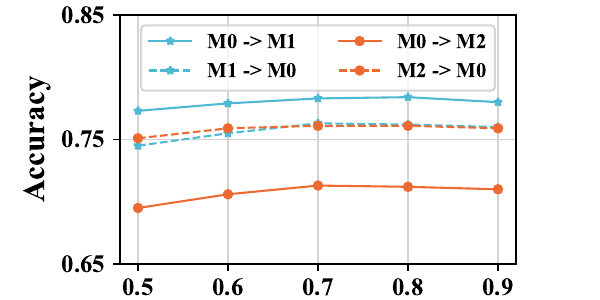}
        \caption{Mutagenicity}
    \end{subfigure}
    \hfill
    \begin{subfigure}[t]{0.24\textwidth}
        \centering
        \includegraphics[width=\linewidth]{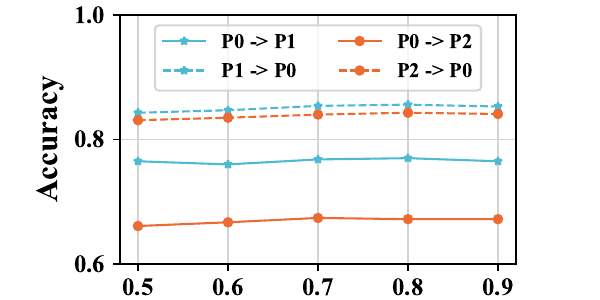}
        \caption{PROTEINS}
    \end{subfigure}
    \hfill
    \begin{subfigure}[t]{0.24\textwidth}
        \centering
        \includegraphics[width=\linewidth]{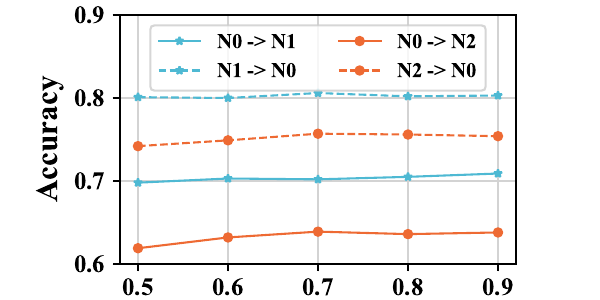}
        \caption{NCI1}
    \end{subfigure}
    \hfill
    \begin{subfigure}[t]{0.24\textwidth}
        \centering
        \includegraphics[width=\linewidth]{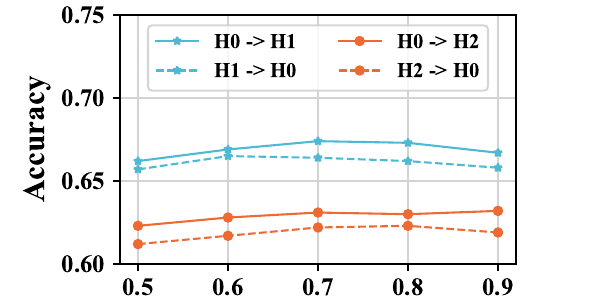}
        \caption{ogbg-molhiv}
    \end{subfigure}
    \caption{Hyperparameter sensitivity analysis of confidence threshold $\zeta$ on the Mutagenicity, PROTEINS, NCI1 and ogbg-molhiv datasets.}
    \label{fig:hyper_threshold}
    \vspace{-0.1cm}
\end{figure*}

%% file: figure/hyper_curve.tex
\begin{figure*}[t]

    \centering
    \captionsetup[subfigure]{font=scriptsize} 
    \begin{subfigure}[t]{0.32\textwidth}
        \centering
        \includegraphics[width=\linewidth]{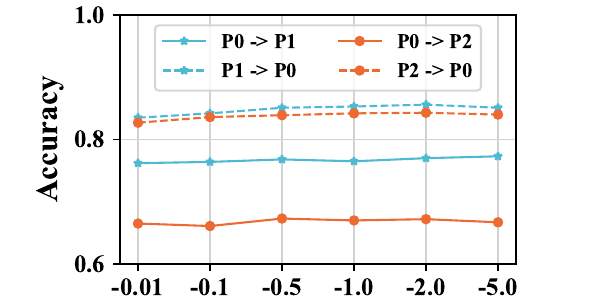}
        \caption{PROTEINS}
    \end{subfigure}
    \hfill
    \begin{subfigure}[t]{0.32\textwidth}
        \centering
        \includegraphics[width=\linewidth]{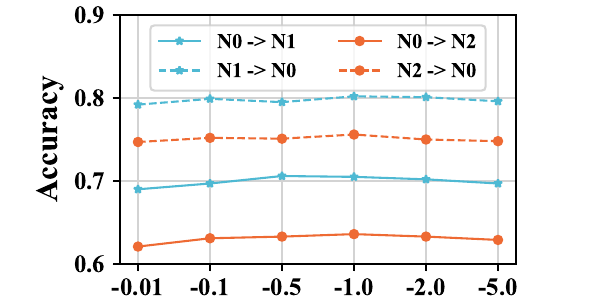}
        \caption{NCI1}
    \end{subfigure}
    \hfill
    \begin{subfigure}[t]{0.32\textwidth}
        \centering
        \includegraphics[width=\linewidth]{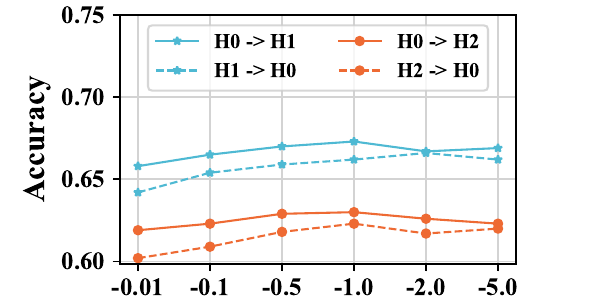}
        \caption{ogbg-molhiv}
    \end{subfigure}
    \caption{Hyperparameter sensitivity analysis of curvature parameter $c$ on the PROTEINS, NCI1 and ogbg-molhiv datasets.}
    \vspace{-0.1cm}
    \label{fig:hyper_curve}
\end{figure*}

%% file: table/Mutag_ablation.tex
\begin{table*}[t]
\small
\centering
\caption{The results of ablation studies on the Mutagenicity dataset (source → target). \textbf{Bold} results indicate the best performance.}
\resizebox{1.0\textwidth}{!}{
\begin{tabular}{l|c|c|c|c|c|c|c|c|c|c|c|c}
\toprule
{\bf Methods} &M0$\rightarrow$M1 &M1$\rightarrow$M0 &M0$\rightarrow$M2 &M2$\rightarrow$M0 &M0$\rightarrow$M3 &M3$\rightarrow$M0 &M1$\rightarrow$M2 &M2$\rightarrow$M1 &M1$\rightarrow$M3 &M3$\rightarrow$M1 &M2$\rightarrow$M3 &M3$\rightarrow$M2 \\
\midrule 

\method{} w/o FM & 76.2 & 73.6 & 69.6 & 74.0 & 65.3 & 66.2 & 77.2 & 81.1 & 66.3 & 69.3 & 71.2 & 73.6 \\

\method{} w/o RA & 76.9 & 74.4 & 70.7 & 73.9 & 66.1 & 67.1 & 77.8 & 82.0 & 66.8 & 70.6 & 71.9 & 74.7 \\

\method{} w/o AA & 76.4 & 74.0 & 70.2 & 73.5 & 65.7 & 66.0 & 77.0 & 81.3 & 65.9 & 69.5 & 71.3 & 73.8\\

\method{} w/o PE & 75.9 & 72.9 & 69.0 & 72.5 & 64.6 & 64.8 & 76.2 & 80.0 & 64.1 & 67.9 & 70.6 & 72.5\\


\midrule 
\method{} & 
\textbf{78.3} & \textbf{76.3} & \textbf{71.3} & \textbf{76.1} & \textbf{67.7} & \textbf{69.2} & \textbf{79.0} & \textbf{83.3} & \textbf{68.3} & \textbf{73.0} & \textbf{73.3} & \textbf{76.0} \\
\bottomrule
\end{tabular}
}
\label{tab:ablation_mutag}
\end{table*}

%% file: table/PROTEINS_ablation.tex
\begin{table*}[t]
\small
\centering
\caption{The results of ablation studies on the PROTEINS dataset (source → target). \textbf{Bold} results indicate the best performance.}
\vspace{-3pt}
\resizebox{1.0\textwidth}{!}{
\begin{tabular}{l|c|c|c|c|c|c|c|c|c|c|c|c}
\toprule
{\bf Methods} &P0$\rightarrow$P1 &P1$\rightarrow$P0 &P0$\rightarrow$P2 &P2$\rightarrow$P0 &P0$\rightarrow$P3 &P3$\rightarrow$P0 &P1$\rightarrow$P2 &P2$\rightarrow$P1 &P1$\rightarrow$P3 &P3$\rightarrow$P1 &P2$\rightarrow$P3 &P3$\rightarrow$P2 \\
\midrule 

\method{} w/o FM & 74.6 & 82.6 & 64.9 & 81.7 & 72.0 & 80.9 & 65.3 & 72.3 & 72.7 & 71.9 & 72.9 & 59.9\\

\method{} w/o RA & 75.3 & 82.9 & 66.0 & 82.1 & 73.4 & 81.7 & 67.2 & 73.3 & 73.0 & 73.4 & 72.1 & 60.9 \\

\method{} w/o AA & 75.1 & 83.0 & 65.7 & 81.5 & 72.9 & 81.5 & 66.1 & 72.7 & 71.6 & 72.7 & 73.0 & 60.3\\

\method{} w/o PE & 74.2 & 81.9 & 64.7 & 80.8 & 71.9 & 80.6 & 65.3 & 71.9 & 70.8 & 72.0 & 71.9 & 59.6 \\


\midrule 
\method{} & \textbf{77.0} & \textbf{85.6} & \textbf{67.4} & \textbf{84.3} & \textbf{75.9} & \textbf{83.8} & \textbf{67.9}& \textbf{74.9} & \textbf{75.9} & \textbf{74.4} & \textbf{74.5} & \textbf{62.9} \\
\bottomrule
\end{tabular}
}
\label{tab:proteins_ablation}
\end{table*}

%% file: table/NCI1_ablation.tex
\begin{table*}[t]
\small
\centering
\caption{The results of ablation studies on the NCI1 dataset (source → target). \textbf{Bold} results indicate the best performance.}
\resizebox{1.0\textwidth}{!}{
\begin{tabular}{l|c|c|c|c|c|c|c|c|c|c|c|c}
\toprule
\textbf{Methods}    &N0$\rightarrow$N1 &N1$\rightarrow$N0 &N0$\rightarrow$N2 &N2$\rightarrow$N0 &N0$\rightarrow$N3 &N3$\rightarrow$N0 &N1$\rightarrow$N2 &N2$\rightarrow$N1 &N1$\rightarrow$N3 &N3$\rightarrow$N1 &N2$\rightarrow$N3 &N3$\rightarrow$N2 \\
\midrule 

\method{} w/o FM & 68.4 & 78.4 & 62.3 & 73.4 & 62.2 & 70.0 & 68.6 & 70.1 & 66.2 & 56.4 & 72.1 & 64.0 \\

\method{} w/o RA & 68.9 & 78.2 & 62.6 & 72.9 & 63.2 & 71.3 & 69.2 & 70.7 & 66.4 & 57.2 & 72.6 & 64.5\\

\method{} w/o AA & 67.8 & 78.7 & 62.3 & 74.0 & 62.3 & 71.0 & 68.7 & 69.8 & 65.1 & 56.7 & 72.7 & 63.9\\

\method{} w/o PE & 67.0 & 77.6 & 62.0 & 71.7 & 62.0 & 70.2 & 67.9 & 69.3 & 64.0 & 56.4 & 71.9 & 63.2 \\


\midrule 
\method{} & \textbf{70.2} & \textbf{80.6} & \textbf{63.9} & \textbf{75.7} & \textbf{64.6} & \textbf{72.6} & \textbf{70.7} & \textbf{71.9} & \textbf{68.0} & \textbf{58.2} & \textbf{74.4} & \textbf{65.6} \\
\bottomrule
\end{tabular}
}
\label{tab:nci1_ablation}
\end{table*}

%% file: table/Molhiv_ablation.tex
\begin{table*}[t]
\small
\centering
\caption{The results of ablation studies on the ogbg-molhiv dataset (source → target). \textbf{Bold} results indicate the best performance.}
\vspace{-3pt}
\resizebox{1.0\textwidth}{!}{
\begin{tabular}{l|c|c|c|c|c|c|c|c|c|c|c|c}
\toprule
\textbf{Methods}    & H0$\rightarrow$H1 & H1$\rightarrow$H0 & H0$\rightarrow$H2 & H2$\rightarrow$H0 & H0$\rightarrow$H3 & H3$\rightarrow$H0 & H1$\rightarrow$H2 & H2$\rightarrow$H1 & H1$\rightarrow$H3 & H3$\rightarrow$H1 & H2$\rightarrow$H3 & H3$\rightarrow$H2 \\
\midrule 

\method{} w/o FM & 65.1 & 64.7 & 60.8 & 60.0 & 64.3 & 58.3 & 67.2 & 66.1 & 70.8 & 65.9 & 74.8 & 68.2 \\

\method{} w/o RA & 66.1 & 65.6 & 61.7 & 59.4 & 64.0 & 59.7 & 67.7 & 67.0 & 70.1 & 65.7 & 75.7 & 68.8 \\

\method{} w/o AA & 64.9 & 64.3 & 59.8 & 58.8 & 63.8 & 58.0 & 66.5 & 66.3 & 69.9 & 64.6 & 74.3 & 67.5 \\

\method{} w/o PE & 64.5 & 63.8 & 59.4 & 58.6 & 62.6 & 58.2 & 66.0 & 66.1 & 68.4 & 63.9 & 74.0 & 67.1 \\

\midrule 
\method{} & \textbf{67.4} & \textbf{66.4} & \textbf{63.1} & \textbf{62.2} & \textbf{66.2} & \textbf{60.9} & \textbf{68.7} & \textbf{68.4} & \textbf{72.1} & \textbf{67.3} & \textbf{77.2} & \textbf{70.5}\\
\bottomrule
\end{tabular}
}
\label{tab:molhiv_ablation}
\end{table*}

%% file: table/PROTEINS_node.tex
\begin{table*}[t]
\centering
\caption{The graph classification results (in \%) on PROTEINS under node density domain shift (source $\rightarrow$ target). P0, P1, P2, and P3 denote the sub-datasets partitioned with node density. \textbf{Bold} results indicate the best performance.}
\resizebox{1.0\textwidth}{!}{
\begin{tabular}{l|c|c|c|c|c|c|c|c|c|c|c|c}
\toprule
{\bf Methods} &P0$\rightarrow$P1 &P1$\rightarrow$P0 &P0$\rightarrow$P2 &P2$\rightarrow$P0 &P0$\rightarrow$P3 &P3$\rightarrow$P0 &P1$\rightarrow$P2 &P2$\rightarrow$P1 &P1$\rightarrow$P3 &P3$\rightarrow$P1 &P2$\rightarrow$P3 &P3$\rightarrow$P2 \\
\midrule
WL subtree  &69.1 &59.7 &61.2 &75.9 &41.6 &83.5 &61.5 &72.7 &24.7 &72.7 &63.1 &62.9 \\
GCN &73.7\scriptsize{$\pm$0.3} &82.7\scriptsize{$\pm$0.4} &57.6\scriptsize{$\pm$0.2} &84.0\scriptsize{$\pm$1.3} &24.4\scriptsize{$\pm$0.4} &17.3\scriptsize{$\pm$0.2} &57.6\scriptsize{$\pm$0.1} &70.9\scriptsize{$\pm$0.7} &24.4\scriptsize{$\pm$0.5} &26.3\scriptsize{$\pm$0.1} &37.5\scriptsize{$\pm$0.2} &42.5\scriptsize{$\pm$0.8} \\
GIN &71.8\scriptsize{$\pm$2.7} &70.2\scriptsize{$\pm$4.7} &58.5\scriptsize{$\pm$4.3} &56.9\scriptsize{$\pm$4.9} &74.2\scriptsize{$\pm$1.7} &78.2\scriptsize{$\pm$3.3} &63.3\scriptsize{$\pm$2.7} &67.1\scriptsize{$\pm$3.8} &35.9\scriptsize{$\pm$4.2} &61.0\scriptsize{$\pm$2.4} &71.9\scriptsize{$\pm$2.1} &65.1\scriptsize{$\pm$1.0} \\
GMT &73.7\scriptsize{$\pm$0.2} &82.7\scriptsize{$\pm$0.1} &57.6\scriptsize{$\pm$0.3} &83.1\scriptsize{$\pm$0.5} &75.6\scriptsize{$\pm$1.4} &17.3\scriptsize{$\pm$0.6} &57.6\scriptsize{$\pm$1.5} &73.7\scriptsize{$\pm$0.6} &75.6\scriptsize{$\pm$0.4} &26.3\scriptsize{$\pm$1.2} &75.6\scriptsize{$\pm$0.7} &42.4\scriptsize{$\pm$0.5} \\
CIN &74.1\scriptsize{$\pm$0.6} &83.8\scriptsize{$\pm$1.0} &60.1\scriptsize{$\pm$2.1} &78.6\scriptsize{$\pm$3.1} &75.6\scriptsize{$\pm$0.2} &74.8\scriptsize{$\pm$3.7} &63.9\scriptsize{$\pm$2.7} &74.1\scriptsize{$\pm$0.6} &57.0\scriptsize{$\pm$4.3} &58.9\scriptsize{$\pm$3.3} &75.6\scriptsize{$\pm$0.7} &63.6\scriptsize{$\pm$1.0} \\
\bottomrule
\end{tabular}
}
\label{tab:proteins_node}
\end{table*}

%% file: table/PROTEINS_idx.tex
\begin{table*}[t]
\centering
\caption{The graph classification results (in \%) on PROTEINS under edge density domain shift (source $\rightarrow$ target). P0, P1, P2, and P3 denote the sub-datasets partitioned with edge density. \textbf{Bold} results indicate the best performance. }
\resizebox{1.0\textwidth}{!}{
\begin{tabular}{l|c|c|c|c|c|c|c|c|c|c|c|c}
\toprule
{\bf Methods} &P0$\rightarrow$P1 &P1$\rightarrow$P0 &P0$\rightarrow$P2 &P2$\rightarrow$P0 &P0$\rightarrow$P3 &P3$\rightarrow$P0 &P1$\rightarrow$P2 &P2$\rightarrow$P1 &P1$\rightarrow$P3 &P3$\rightarrow$P1 &P2$\rightarrow$P3 &P3$\rightarrow$P2 \\
\midrule
WL subtree  &68.7 &82.3 &50.7 &82.3 &58.1 &83.8 &64.0 &74.1 &43.7 &70.5 &71.3 &60.1 \\
GCN &73.4\scriptsize{$\pm$0.2} &83.5\scriptsize{$\pm$0.3} &57.6\scriptsize{$\pm$0.2} &84.2\scriptsize{$\pm$1.8} &24.0\scriptsize{$\pm$0.1} &16.6\scriptsize{$\pm$0.4} &57.6\scriptsize{$\pm$0.2} &73.7\scriptsize{$\pm$0.4} &24.0\scriptsize{$\pm$0.1} &26.6\scriptsize{$\pm$0.2} &39.9\scriptsize{$\pm$0.9} &42.5\scriptsize{$\pm$0.1} \\
GIN &62.5\scriptsize{$\pm$4.7} &74.9\scriptsize{$\pm$3.7} &53.0\scriptsize{$\pm$4.6} &59.6\scriptsize{$\pm$4.2} &73.7\scriptsize{$\pm$0.8} &64.7\scriptsize{$\pm$3.4} &60.6\scriptsize{$\pm$2.7} &69.8\scriptsize{$\pm$0.6} &31.1\scriptsize{$\pm$2.8} &63.1\scriptsize{$\pm$3.4} &72.3\scriptsize{$\pm$2.7} &64.6\scriptsize{$\pm$1.4} \\
GMT &73.4\scriptsize{$\pm$0.3} &83.5\scriptsize{$\pm$0.2} &57.6\scriptsize{$\pm$0.1} &83.5\scriptsize{$\pm$0.3} &24.0\scriptsize{$\pm$0.1} &83.5\scriptsize{$\pm$0.1} &57.4\scriptsize{$\pm$0.2} &73.4\scriptsize{$\pm$0.2} &24.1\scriptsize{$\pm$0.1} &73.4\scriptsize{$\pm$0.3} &24.0\scriptsize{$\pm$0.1} &57.6\scriptsize{$\pm$0.2} \\
CIN &74.5\scriptsize{$\pm$0.2} &84.1\scriptsize{$\pm$0.5} &57.8\scriptsize{$\pm$0.2} &82.7\scriptsize{$\pm$0.9} &75.6\scriptsize{$\pm$0.6} &79.2\scriptsize{$\pm$2.2} &61.5\scriptsize{$\pm$2.7} &74.0\scriptsize{$\pm$1.0} &75.5\scriptsize{$\pm$0.8} &72.5\scriptsize{$\pm$2.1} &\textbf{76.0\scriptsize{$\pm$0.3}} &60.9\scriptsize{$\pm$1.2} \\
\bottomrule
\end{tabular}
}
\label{tab:proteins_idx}
\end{table*}

%% file: table/NCI1_node.tex
\begin{table*}[t]
\small
\caption{The classification results (in \%) on the NCI1 dataset under node density domain shift (source $\rightarrow$ target).  N0, N1, N2, and N3 denote the sub-datasets partitioned with node density. \textbf{Bold} results indicate the best performance.}
	\resizebox{1.0\textwidth}{!}{
	\begin{tabular}{l|c|c|c|c|c|c|c|c|c|c|c|c}
		\toprule                                                      
		\textbf{Methods}    &N0$\rightarrow$N1 &N1$\rightarrow$N0 &N0$\rightarrow$N2 &N2$\rightarrow$N0 &N0$\rightarrow$N3 &N3$\rightarrow$N0 &N1$\rightarrow$N2 &N2$\rightarrow$N1 &N1$\rightarrow$N3 &N3$\rightarrow$N1 &N2$\rightarrow$N3 &N3$\rightarrow$N2
 \\ \midrule
WL subtree  & 54.9 & 60.0 & 51.4 & 51.4 & 44.4 & 63.1 & 51.9 & 53.4 & 58.4 & 57.4 & 61.2 & 50.9\\
GCN & 54.0\scriptsize{$\pm$2.0} & 66.5\scriptsize{$\pm$1.9} & 47.3\scriptsize{$\pm$2.2} & 37.4\scriptsize{$\pm$1.8} & 38.4\scriptsize{$\pm$2.3} & 27.6\scriptsize{$\pm$1.5} & 59.0\scriptsize{$\pm$1.7} & 56.9\scriptsize{$\pm$2.0} & 56.5\scriptsize{$\pm$1.2} & 51.0\scriptsize{$\pm$1.7} & 60.5\scriptsize{$\pm$1.9} & 57.4\scriptsize{$\pm$1.6}\\
GIN & 57.0\scriptsize{$\pm$1.9} & 38.7\scriptsize{$\pm$2.3} & 45.7\scriptsize{$\pm$1.8} & 34.4\scriptsize{$\pm$1.9} & 43.6\scriptsize{$\pm$2.1} & 27.0\scriptsize{$\pm$2.0} & 54.7\scriptsize{$\pm$1.7} & 52.8\scriptsize{$\pm$2.1} & 59.0\scriptsize{$\pm$1.8} & 49.1\scriptsize{$\pm$2.5} & 59.4\scriptsize{$\pm$2.0} & 57.4\scriptsize{$\pm$1.3}\\
GMT & 53.0\scriptsize{$\pm$1.9} & 63.7\scriptsize{$\pm$1.6} & 48.9\scriptsize{$\pm$2.2} & 32.8\scriptsize{$\pm$2.0} & 41.5\scriptsize{$\pm$1.7} & 27.8\scriptsize{$\pm$1.6} & 56.4\scriptsize{$\pm$1.9} & 57.2\scriptsize{$\pm$1.7} & 54.1\scriptsize{$\pm$1.9} & 49.8\scriptsize{$\pm$1.7} & 61.6\scriptsize{$\pm$2.1} & 58.5\scriptsize{$\pm$1.1}\\
CIN & 54.5\scriptsize{$\pm$1.9} & 37.6\scriptsize{$\pm$1.9} & 53.2\scriptsize{$\pm$1.8} & 36.2\scriptsize{$\pm$1.4} & 42.1\scriptsize{$\pm$1.8} & 27.5\scriptsize{$\pm$2.0} & 57.2\scriptsize{$\pm$1.6} & 51.6\scriptsize{$\pm$1.9} & 59.3\scriptsize{$\pm$1.5} & 50.0\scriptsize{$\pm$1.4} & 62.0\scriptsize{$\pm$1.4} & 57.4\scriptsize{$\pm$2.3}\\
\bottomrule
	\end{tabular}}
    \label{tab:nci1_node}
\end{table*}

%% file: table/NCI1_idx.tex
\begin{table*}[t]
\small
\caption{The classification results (in \%) on the NCI1 dataset under edge density domain shift (source $\rightarrow$ target).  N0, N1, N2, and N3 denote the sub-datasets partitioned with edge density. \textbf{Bold} results indicate the best performance.}
	\resizebox{1.0\textwidth}{!}{
	\begin{tabular}{l|c|c|c|c|c|c|c|c|c|c|c|c}
		\toprule                                                      
		\textbf{Methods}    &N0$\rightarrow$N1 &N1$\rightarrow$N0 &N0$\rightarrow$N2 &N2$\rightarrow$N0 &N0$\rightarrow$N3 &N3$\rightarrow$N0 &N1$\rightarrow$N2 &N2$\rightarrow$N1 &N1$\rightarrow$N3 &N3$\rightarrow$N1 &N2$\rightarrow$N3 &N3$\rightarrow$N2
 \\ \midrule
WL subtree  & 50.1 & 59.9 & 50.9 & 51.0 & 51.8 & 54.5 & 56.7 & 57.0 & 52.5 & 57.6 & 54.4 & 46.9 \\
GCN & 53.7\scriptsize{$\pm$2.1} & 68.2\scriptsize{$\pm$1.9} & 50.6\scriptsize{$\pm$1.9} & 45.8\scriptsize{$\pm$1.7} & 35.5\scriptsize{$\pm$2.4} & 28.2\scriptsize{$\pm$2.2} & 55.9\scriptsize{$\pm$1.7} & 56.8\scriptsize{$\pm$1.9} & 55.7\scriptsize{$\pm$1.5} & 53.2\scriptsize{$\pm$2.1} & 60.4\scriptsize{$\pm$1.8} & 54.8\scriptsize{$\pm$0.9} \\
GIN & 51.4\scriptsize{$\pm$1.7} & 47.5\scriptsize{$\pm$2.1} & 50.4\scriptsize{$\pm$1.6} & 32.7\scriptsize{$\pm$1.6} & 47.0\scriptsize{$\pm$2.1} & 29.4\scriptsize{$\pm$2.3} & 52.4\scriptsize{$\pm$1.9} & 55.7\scriptsize{$\pm$2.0} & 64.3\scriptsize{$\pm$1.8} & 51.1\scriptsize{$\pm$1.2} & 60.3\scriptsize{$\pm$1.3} & 53.5\scriptsize{$\pm$1.6}\\
GMT & 53.0\scriptsize{$\pm$1.9} & 58.3\scriptsize{$\pm$1.7} & 49.0\scriptsize{$\pm$1.9} & 34.3\scriptsize{$\pm$1.6} & 39.7\scriptsize{$\pm$1.5} & 27.5\scriptsize{$\pm$2.4} & 55.8\scriptsize{$\pm$1.6} & 60.4\scriptsize{$\pm$2.1} & 51.3\scriptsize{$\pm$1.0} & 51.5\scriptsize{$\pm$1.6} & 60.0\scriptsize{$\pm$2.6} & 54.5\scriptsize{$\pm$1.4}\\
CIN & 54.5\scriptsize{$\pm$2.1} & 36.8\scriptsize{$\pm$1.9} & 53.0\scriptsize{$\pm$1.3} & 32.7\scriptsize{$\pm$2.1} & 44.9\scriptsize{$\pm$1.4} & 27.4\scriptsize{$\pm$1.1} & 58.4\scriptsize{$\pm$1.6} & 54.2\scriptsize{$\pm$2.3} & 63.9\scriptsize{$\pm$1.6} & 51.2\scriptsize{$\pm$1.4} & 64.7\scriptsize{$\pm$1.1} & 53.7\scriptsize{$\pm$2.1}\\
\bottomrule
	\end{tabular}}
    \label{tab:nci_idx}
\end{table*}

%% file: table/Mutag_node.tex
\begin{table*}[t]
\small
\caption{The classification results (in \%) on the Mutagenicity dataset under node density domain shift (source $\rightarrow$ target).  M0, M1, M2, and M3 denote the sub-datasets partitioned with node density. \textbf{Bold} results indicate the best performance.}
	\resizebox{1.0\textwidth}{!}{
	\begin{tabular}{l|c|c|c|c|c|c|c|c|c|c|c|c}
		\toprule                                                      
		\textbf{Methods}    &M0$\rightarrow$M1 &M1$\rightarrow$M0 &M0$\rightarrow$M2 &M2$\rightarrow$M0 &M0$\rightarrow$M3 &M3$\rightarrow$M0 &M1$\rightarrow$M2 &M2$\rightarrow$M1 &M1$\rightarrow$M3 &M3$\rightarrow$M1 &M2$\rightarrow$M3 &M3$\rightarrow$M2
 \\ \midrule
WL subtree  & 34.3 & 46.0 & 40.4 & 50.4 & 52.7 & 52.1 & 51.6 & 48.3 & 45.6 & 40.8 & 53.5 & 52.4\\
GCN & 64.1\scriptsize{$\pm$1.4} & 62.6\scriptsize{$\pm$2.3} & 65.5\scriptsize{$\pm$2.0} & 62.3\scriptsize{$\pm$1.1} & 56.9\scriptsize{$\pm$2.1} & 53.7\scriptsize{$\pm$1.8} & 66.1\scriptsize{$\pm$1.6} & 70.8\scriptsize{$\pm$2.1} & 47.7\scriptsize{$\pm$1.9} & 49.2\scriptsize{$\pm$2.1} & 52.1\scriptsize{$\pm$1.7} & 54.4\scriptsize{$\pm$2.1}\\
GIN & 66.5\scriptsize{$\pm$2.1} & 53.2\scriptsize{$\pm$2.3} & 52.0\scriptsize{$\pm$1.7} & 62.6\scriptsize{$\pm$2.3} & 53.7\scriptsize{$\pm$1.7} & 48.6\scriptsize{$\pm$2.3} & 71.5\scriptsize{$\pm$2.5} & 70.0\scriptsize{$\pm$2.1} & 52.2\scriptsize{$\pm$2.3} & 64.3\scriptsize{$\pm$2.1} & 50.9\scriptsize{$\pm$1.4} & 52.4\scriptsize{$\pm$1.5} \\
GMT & 65.7\scriptsize{$\pm$1.8} & 67.1\scriptsize{$\pm$1.9} & 62.1\scriptsize{$\pm$2.1} & 66.4\scriptsize{$\pm$1.2} & 59.0\scriptsize{$\pm$2.0} & 58.0\scriptsize{$\pm$2.7} & 69.9\scriptsize{$\pm$1.4} & 68.4\scriptsize{$\pm$2.3} & 60.4\scriptsize{$\pm$1.3} & 48.2\scriptsize{$\pm$2.0} & 59.1\scriptsize{$\pm$1.3} & 60.6\scriptsize{$\pm$1.7} \\
CIN & 65.1\scriptsize{$\pm$1.7} & 52.9\scriptsize{$\pm$2.3} & 66.0\scriptsize{$\pm$1.7} & 69.8\scriptsize{$\pm$2.0} & 55.2\scriptsize{$\pm$1.5} & 55.1\scriptsize{$\pm$1.3} & 66.0\scriptsize{$\pm$1.6} & 69.9\scriptsize{$\pm$2.1} & 54.9\scriptsize{$\pm$1.8} & 57.8\scriptsize{$\pm$1.5} & 46.7\scriptsize{$\pm$2.0} & 51.8\scriptsize{$\pm$1.1}\\
\bottomrule
	\end{tabular}}
    \label{tab:mutag_node}
\end{table*}

%% file: table/Mutag_idx.tex
\begin{table*}[t]
\small
\caption{The classification results (in \%) on the Mutagenicity dataset under edge density domain shift (source $\rightarrow$ target).  M0, M1, M2, and M3 denote the sub-datasets partitioned with edge density. \textbf{Bold} results indicate the best performance.}
	\resizebox{1.0\textwidth}{!}{
	\begin{tabular}{l|c|c|c|c|c|c|c|c|c|c|c|c}
		\toprule                                                      
		\textbf{Methods}    &M0$\rightarrow$M1 &M1$\rightarrow$M0 &M0$\rightarrow$M2 &M2$\rightarrow$M0 &M0$\rightarrow$M3 &M3$\rightarrow$M0 &M1$\rightarrow$M2 &M2$\rightarrow$M1 &M1$\rightarrow$M3 &M3$\rightarrow$M1 &M2$\rightarrow$M3 &M3$\rightarrow$M2
 \\ \midrule
WL subtree  & 34.4  & 39.6 & 47.6 & 43.6  & 52.7 & 46.4 & 59.8 & 40.9 & 53.0 & 65.1& 52.6  & 42.4 \\
GCN & 66.3\scriptsize{$\pm$1.7} & 65.1\scriptsize{$\pm$2.2} & 63.6\scriptsize{$\pm$1.4} & 67.6\scriptsize{$\pm$1.6} & 56.0\scriptsize{$\pm$1.4} & 49.6\scriptsize{$\pm$1.9} & 66.1\scriptsize{$\pm$2.0} & 71.8\scriptsize{$\pm$1.8} & 48.0\scriptsize{$\pm$2.3} & 47.8\scriptsize{$\pm$1.7} & 51.6\scriptsize{$\pm$1.8} & 56.6\scriptsize{$\pm$2.2} \\
GIN & 67.1\scriptsize{$\pm$1.7} & 56.1\scriptsize{$\pm$1.5} & 54.2\scriptsize{$\pm$2.6} & 65.6\scriptsize{$\pm$1.6} & 55.4\scriptsize{$\pm$1.9} & 50.8\scriptsize{$\pm$2.1} & 67.9\scriptsize{$\pm$2.1} & 69.4\scriptsize{$\pm$2.0} & 53.4\scriptsize{$\pm$1.6} & 66.7\scriptsize{$\pm$2.0} & 50.4\scriptsize{$\pm$1.1} & 55.4\scriptsize{$\pm$1.2}\\
GMT & 67.9\scriptsize{$\pm$1.3} & 63.9\scriptsize{$\pm$1.8} & 61.5\scriptsize{$\pm$1.8} & 68.5\scriptsize{$\pm$1.5} & 58.2\scriptsize{$\pm$2.4} & 54.3\scriptsize{$\pm$1.8} & 66.7\scriptsize{$\pm$1.9} & 76.1\scriptsize{$\pm$1.7} & 58.8\scriptsize{$\pm$1.5} & 44.5\scriptsize{$\pm$1.5} & 59.0\scriptsize{$\pm$1.4} & 63.7\scriptsize{$\pm$1.7}\\
CIN & 66.3\scriptsize{$\pm$1.8} & 56.7\scriptsize{$\pm$2.5} & 60.8\scriptsize{$\pm$1.7} & 67.4\scriptsize{$\pm$1.7} & 55.8\scriptsize{$\pm$2.4} & 52.3\scriptsize{$\pm$1.8} & 63.4\scriptsize{$\pm$2.3} & 69.3\scriptsize{$\pm$1.2} & 54.3\scriptsize{$\pm$1.2} & 60.7\scriptsize{$\pm$1.4} & 49.0\scriptsize{$\pm$1.3} & 54.4\scriptsize{$\pm$1.7}\\
\bottomrule
	\end{tabular}}
    \label{tab:mutag_idx}
\end{table*}

%% file: table/Molhiv_node.tex
\begin{table*}[t]
\small
\caption{The classification results (in \%) on the ogbg-molhiv dataset under node density domain shift (source $\rightarrow$ target).  H0, H1, H2, and H3 denote the sub-datasets partitioned with node density. \textbf{Bold} results indicate the best performance.}
	\resizebox{1.0\textwidth}{!}{
	\begin{tabular}{l|c|c|c|c|c|c|c|c|c|c|c|c}
		\toprule                                                      
		\textbf{Methods}    & H0$\rightarrow$H1 & H1$\rightarrow$H0 & H0$\rightarrow$H2 & H2$\rightarrow$H0 & H0$\rightarrow$H3 & H3$\rightarrow$H0 & H1$\rightarrow$H2 & H2$\rightarrow$H1 & H1$\rightarrow$H3 & H3$\rightarrow$H1 & H2$\rightarrow$H3 & H3$\rightarrow$H2 \\ \midrule
WL  & 49.6 & 51.8 & 54.0 & 54.7 & 48.2 & 48.9 & 50.6 & 51.9 & 38.1 & 51.0 & 51.8 & 50.7\\
GCN         & 57.8\scriptsize{$\pm$1.9} & 57.4\scriptsize{$\pm$2.1} & 56.5\scriptsize{$\pm$1.3} & 54.2\scriptsize{$\pm$1.2} & 54.5\scriptsize{$\pm$1.7} & 51.9\scriptsize{$\pm$1.7} & 61.8\scriptsize{$\pm$2.0} & 55.2\scriptsize{$\pm$1.3} & 49.8\scriptsize{$\pm$1.4} & 56.6\scriptsize{$\pm$1.3} & 61.5\scriptsize{$\pm$1.1} & 56.8\scriptsize{$\pm$1.7}\\
GIN         & 59.1\scriptsize{$\pm$1.6} & 53.3\scriptsize{$\pm$1.6} & 50.0\scriptsize{$\pm$1.7} & 54.8\scriptsize{$\pm$1.8} & 56.8\scriptsize{$\pm$1.6} & 50.7\scriptsize{$\pm$1.4} & 55.2\scriptsize{$\pm$1.3} & 58.7\scriptsize{$\pm$1.4} & 57.0\scriptsize{$\pm$1.8} & 49.6\scriptsize{$\pm$1.7} & 53.1\scriptsize{$\pm$1.3} & 55.3\scriptsize{$\pm$1.6}\\
GMT         & 51.1\scriptsize{$\pm$1.6} & 57.2\scriptsize{$\pm$1.7} & 50.6\scriptsize{$\pm$1.2} & 55.9\scriptsize{$\pm$1.9} & 58.4\scriptsize{$\pm$1.4} & 53.4\scriptsize{$\pm$1.7} & 56.1\scriptsize{$\pm$1.6} & 57.6\scriptsize{$\pm$1.3} & 60.9\scriptsize{$\pm$1.7} & 54.5\scriptsize{$\pm$2.0} & 55.4\scriptsize{$\pm$1.6} & 51.2\scriptsize{$\pm$1.8}\\
CIN         & 52.3\scriptsize{$\pm$1.6} & 53.4\scriptsize{$\pm$1.5} & 54.0\scriptsize{$\pm$1.9} & 57.5\scriptsize{$\pm$1.8} & 56.5\scriptsize{$\pm$1.4} & 51.1\scriptsize{$\pm$1.9} & 55.8\scriptsize{$\pm$1.4} & 63.6\scriptsize{$\pm$1.8} & 48.7\scriptsize{$\pm$1.6} & 56.8\scriptsize{$\pm$1.6} & 49.3\scriptsize{$\pm$1.2} & 52.8\scriptsize{$\pm$1.3}\\
\bottomrule
	\end{tabular}}
    \label{tab:hiv_node}
\end{table*}

%% file: table/Molhiv_idx.tex
\begin{table*}[t]
\small
\caption{The classification results (in \%) on the ogbg-molhiv dataset under edge density domain shift (source $\rightarrow$ target).  H0, H1, H2, and H3 denote the sub-datasets partitioned with edge density. \textbf{Bold} results indicate the best performance.}
	\resizebox{1.0\textwidth}{!}{
	\begin{tabular}{l|c|c|c|c|c|c|c|c|c|c|c|c}
		\toprule                                                      
		\textbf{Methods}    & H0$\rightarrow$H1 & H1$\rightarrow$H0 & H0$\rightarrow$H2 & H2$\rightarrow$H0 & H0$\rightarrow$H3 & H3$\rightarrow$H0 & H1$\rightarrow$H2 & H2$\rightarrow$H1 & H1$\rightarrow$H3 & H3$\rightarrow$H1 & H2$\rightarrow$H3 & H3$\rightarrow$H2
 \\ \midrule
 WL  & 45.4 & 54.9 & 44.9 & 46.7 & 54.4 & 52.2 & 53.7 & 46.0 & 56.1 & 45.3 & 51.1 & 50.8\\
GCN         & 57.5\scriptsize{$\pm$1.8} & 58.0\scriptsize{$\pm$1.5} & 56.1\scriptsize{$\pm$2.1} & 54.3\scriptsize{$\pm$1.2} & 54.1\scriptsize{$\pm$1.4} & 50.9\scriptsize{$\pm$1.4} & 61.2\scriptsize{$\pm$1.6} & 54.8\scriptsize{$\pm$1.9} & 54.9\scriptsize{$\pm$2.1} & 56.6\scriptsize{$\pm$1.4} & 65.4\scriptsize{$\pm$2.0} & 58.1\scriptsize{$\pm$1.6}\\
GIN         & 58.1\scriptsize{$\pm$1.1} & 55.3\scriptsize{$\pm$1.2} & 49.7\scriptsize{$\pm$1.5} & 55.2\scriptsize{$\pm$2.2} & 55.5\scriptsize{$\pm$1.9} & 50.5\scriptsize{$\pm$2.0} & 55.1\scriptsize{$\pm$1.3} & 58.8\scriptsize{$\pm$1.2} & 56.6\scriptsize{$\pm$1.6} & 50.0\scriptsize{$\pm$1.1} & 56.8\scriptsize{$\pm$1.8} & 56.2\scriptsize{$\pm$1.6} \\
GMT         & 50.7\scriptsize{$\pm$1.0} & 57.5\scriptsize{$\pm$1.7} & 51.1\scriptsize{$\pm$2.1} & 58.8\scriptsize{$\pm$1.5} & 59.8\scriptsize{$\pm$1.4} & 52.0\scriptsize{$\pm$1.4} & 57.0\scriptsize{$\pm$1.8} & 55.4\scriptsize{$\pm$1.0} & 54.0\scriptsize{$\pm$1.7} & 54.9\scriptsize{$\pm$1.0} & 58.1\scriptsize{$\pm$1.3} & 52.0\scriptsize{$\pm$1.3} \\
CIN         & 51.8\scriptsize{$\pm$1.4} & 54.1\scriptsize{$\pm$1.7} & 53.6\scriptsize{$\pm$1.4} & 57.3\scriptsize{$\pm$1.6} & 56.9\scriptsize{$\pm$1.7} & 53.9\scriptsize{$\pm$1.3} & 54.5\scriptsize{$\pm$1.7} & 63.6\scriptsize{$\pm$2.1} & 54.2\scriptsize{$\pm$1.2} & 56.9\scriptsize{$\pm$1.8} & 57.6\scriptsize{$\pm$1.3} & 58.7\scriptsize{$\pm$1.2}\\
\bottomrule
	\end{tabular}}
    \label{tab:hiv_idx}
\end{table*}